\documentclass{article}
\pdfoutput=1

\usepackage{xcolor}

\newcommand{\numericalConst}[1]{#1}

\usepackage[utf8]{inputenc}
\usepackage{amssymb, amsmath, amsthm}
\usepackage{algorithm, algpseudocode}
\usepackage{graphicx}
\usepackage{booktabs}
\usepackage{multirow,multicol}
\usepackage{hyperref}

\newtheorem{lemma}{Lemma}[section]
\newtheorem{definition}[lemma]{Definition}

\newtheorem{theorem}[lemma]{Theorem}
\newtheorem{corollary}[lemma]{Corollary}
\theoremstyle{remark}
\newtheorem{remark}[lemma]{Remark}

\newcommand{\dualp}[1]{\left\langle #1 \right\rangle} 

\newcommand{\E}[1]{\mathbb{E}\left[ #1 \right]}
\newcommand{\EE}[2]{\mathbb{E}_{#1}\left[ #2 \right]}
\newcommand{\pr}[1]{\operatorname{Pr}\left[ #1 \right]}
\DeclareMathOperator{\pdf}{pdf}

\newcommand{\loss}{\mathcal{L}}
\newcommand{\dom}{D}
\newcommand{\wdom}{\Theta}
\newcommand{\hs}{\mathcal{H}}

\newcommand{\ptheta}{{\bar{\theta}}}
\newcommand{\res}{\kappa}
\newcommand{\wdiff}{h}

\newcommand{\real}{\mathbb{R}}

\DeclareMathOperator{\relu}{ReLU}
\newcommand{\activation}{\sigma}
\newcommand{\activationp}[1]{\sigma\left(#1\right)}
\newcommand{\dactivation}{\dot{\sigma}}
\newcommand{\dactivationp}[1]{\dot{\sigma}\left(#1\right)}

\newcommand{\gaussian}[1]{\mathcal{N}(#1)}
\newcommand{\px}{\bar{x}}
\newcommand{\py}{\bar{y}}
\newcommand{\pg}{\bar{g}}
\newcommand{\ph}{\bar{h}}

\newcommand{\gp}{\Sigma}
\newcommand{\egp}{\hat{\Sigma}}
\newcommand{\dgp}{\dot{\Sigma}}
\newcommand{\edgp}{\hat{\dot{\Sigma}}}
\newcommand{\ntk}{\Gamma}
\newcommand{\entk}{\hat{\Gamma}}

\newcommand{\pegp}{\bar{\hat{\Sigma}}}

\newcommand{\pedgp}{\bar{\hat{\dot{\Sigma}}}}
\newcommand{\pntk}{\bar{\Gamma}}
\newcommand{\pentk}{\bar{\hat{\Gamma}}}
\newcommand{\egpr}{\hat{\Lambda}}
\newcommand{\eA}{\hat{A}}

\newcommand{\NTK}{\Theta}

\newcommand{\w}{\lambda}

\newcommand{\pW}{\bar{W}}
\newcommand{\pf}{\bar{f}}

\newcommand{\CN}[1]{{C^{#1}}}
\newcommand{\CHN}[1]{{C^{0;#1}}}
\newcommand{\CHHN}[2]{{C^{0;#1,#2}}}
\newcommand{\CND}[2]{{C^{#1}(#2)}}
\newcommand{\CHND}[2]{{C^{0;#1}(#2)}}
\newcommand{\CHHND}[3]{{C^{0;#1,#2}(#3)}}

\newcommand{\idelta}{\bar{\Delta}}

\newcommand{\Sd}{\mathbb{S}^{d-1}}

\newcommand{\wnorm}[1]{\left\| #1 \right\|_*}
\newcommand{\ntks}{\beta}
\newcommand{\holder}{\gamma}
\newcommand{\sm}{\alpha}
\newcommand{\Sm}{S}
\newcommand{\ntkop}{H}

\def\Xint#1{\mathchoice
   {\XXint\displaystyle\textstyle{#1}}%
   {\XXint\textstyle\scriptstyle{#1}}%
   {\XXint\scriptstyle\scriptscriptstyle{#1}}%
   {\XXint\scriptscriptstyle\scriptscriptstyle{#1}}%
   \!\int}
\def\XXint#1#2#3{{\setbox0=\hbox{$#1{#2#3}{\int}$ }
     \vcenter{\hbox{$#2#3$ }}\kern-.6\wd0}}

\def\dashint{\Xint-}

\begin{document}

\title{Approximation Results for Gradient Descent trained Neural Networks}
\author{G. Welper\footnote{Department of Mathematics, University of Central Florida, Orlando, FL 32816, USA, \texttt{gerrit.welper@ucf.edu}}}
\date{}
\maketitle

\begin{abstract}

  The paper contains approximation guarantees for neural networks that are trained with gradient flow, with error measured in the continuous  $L_2(\Sd)$-norm on the $d$-dimensional unit sphere and targets that are Sobolev smooth. The networks are fully connected of constant depth and increasing width. Although all layers are trained, the gradient flow convergence is based on a neural tangent kernel (NTK) argument for the non-convex second but last layer. Unlike standard NTK analysis, the continuous error norm implies an under-parametrized regime, possible by the natural smoothness assumption required for approximation. The typical over-parametrization re-enters the results in form of a loss in approximation rate relative to established approximation methods for Sobolev smooth functions.

\end{abstract}



\tableofcontents

\section{Introduction}

Direct approximation results for a large variety of methods, including neural networks, are typically of the form
\begin{equation} \label{eq:approx}
  \begin{aligned}
    \inf_\theta \|f_\theta - f\| & \le n(\theta)^{-r}, &
    f & \in K.
  \end{aligned}
\end{equation}
I.e., a target function $f$ is approximated by an approximation method $f_\theta$, parametrized by some degrees of freedom or weights $\theta$ up to a rate $n(\theta)^{-r}$ for some $n(\theta)$ that measures the richness of the approximation method as width, depth or number of weights for neural networks. Generally, the approximation rate can be arbitrarily slow unless the target $f$ is contained in some compact set $K$, which depends on the approximation method and application and is typically a unit ball in a Sobolev, Besov, Barron or other normed smoothness space. Such results are well established for a variety of neural network architectures and compact sets $K$, however, these results rarely address how to practically compute the infimum in the formula above and instead use hand-picked weights.

On the other hand, the neural network optimization literature, typically considers discrete error norms (or losses)
\[
  \|f_\theta - f\|_* := \left( \frac{1}{n} \sum_{i=1}^n |f_\theta(x_i) - f(x_i)|^2 \right)^{1/2},
\]
together with neural networks that are \emph{over-parametrized}, i.e. for which the number of weights is larger than the number of samples $n$ so that they can achieve zero training error
\[
  \inf_\theta \|f_\theta - f\|_* = 0,
\]
rendering the approximation question obsolete. In contrast, approximation theory measures the error in  continuous norms that emerge in the sample $n \to \infty$ limit, where the problem is necessarily \emph{under-parametrized}.

This paper contains approximation results of type \eqref{eq:approx} for fully connected networks that are trained with gradient flow and therefore avoids the question how to compute the infimum in \eqref{eq:approx}. The outline of the proof follows the typical neural tangent kernel (NTK) argument: We show that the empirical NTK is close to the infinite width NTK and that the NTK does not change too much during training. The main differences to the standard analysis are:
\begin{enumerate}
  \item Due to the under-parametrization, the eigenvalues of the NTK are not lower bounded away form zero. Instead we require that the NTK is coercive in a negative Sobolev norm.
  \item We show that the gradient flow networks are uniformly bounded in positive Sobolev norms.
  \item The coercivity in negative Sobolev smoothness and the uniform bounds of positive Sobolev smoothness allow us to derive $L_2$ error bounds by interpolation inequalities.
  \item All perturbation and concentration estimates are carried out in function space norms. In particular, the concentration results need some careful consideration and are proven by chaining arguments.
\end{enumerate}
The NTK is a sum of positive matrices from which we only use the contribution form the second but last layer to drive down the error, while all other layers are trained but estimated only by a perturbation analysis. The coercivity assumption on the NTK is not shown in this paper. It is known for $\relu$ activations, but we require smoother activations and only provide a preliminary numerical test while leaving a rigorous analysis of the resulting NTK for future work.

The proven approximation rates are lower than finite element, wavelet or spline rates under the same smoothness assumptions. This seems to be a variant of the over-parametrization in the usual NTK arguments: the networks need some redundancy in their degrees of freedom to aid the optimization.

\paragraph{Paper Organization} The paper is organized as follows. Section \ref{sec:setup} defines the neural networks and training procedures and Section \ref{sec:main} contains the main result. The coercivity of the NTK is discussed in Section \ref{sec:coercivity}. The proof is split into two parts. Section \ref{sec:proof-overview} provides an overview and all major lemmas. The proof the these lemmas and further details are provided in Section \ref{sec:proof-details}. Finally, to keep the paper self contained, Section \ref{sec:supplements} contains several facts from the literature.

\paragraph{Literature Review}

\begin{itemize}

  \item \emph{Approximation:} Some recent surveys are given in
  \cite{
    Pinkus1999,
    DeVoreHaninPetrova2020,
    WeinanChaoLeiWojtowytsch2020,
    BernerGrohsKutyniokPetersen2021%
  }. Most of the results prove direct approximation guarantees as in \eqref{eq:approx} for a variety of classes $K$ and network architectures. They show state of the art or even superior performance of neural networks, but typically do not provide training methods and rely on hand-picked weights, instead.
  \begin{itemize}
    \item Results for classical \emph{Sobolev} and \emph{Besov regularity} are in 
    \cite{
      GribonvalKutyniokNielsenEtAl2019,
      GuhringKutyniokPetersen2020,
      OpschoorPetersenSchwab2020,
      LiTangYu2019,
      Suzuki2019%
    }.
    \item 
    \cite{
      Yarotsky2017,
      Yarotsky2018,
      YarotskyZhevnerchuk2020,
      DaubechiesDeVoreFoucartEtAl2019,
      ShenYangZhang2019,
      LuShenYangZhang2021%
    }
    show better than classical approximation rates for Sobolev smoothness. Since classical methods are optimal (with regard to nonlinear width and entropy), this implies that the weight assignment $f \to \theta$ must be discontinuous.
    \item Function classes that are specifically tailored to neural networks are \emph{Barron spaces} for which approximation results are given in 
    \cite{
      Bach2017,
      KlusowskiBarron2018,
      WeinanMaWu2019,
      LiMaWu2020,
      SiegelXu2020,
      SiegelXu2020a,
      BreslerNagaraj2020%
    }.
    \item Many papers address specialized function classes
    \cite{
      ShahamCloningerCoifman2018,
      PoggioMhaskarRosascoEtAl2017%
    },
    often from applications like PDEs
    \cite{
      KutyniokPetersenRaslanSchneider2022,
      PetersenVoigtlaender2018,
      LaakmannPetersen2020,
      MarcatiOpschoorPetersenSchwab2022%
    }.
  \end{itemize}
  Besides approximation guarantees \eqref{eq:approx} many of the above papers also discuss limitations of neural networks, for more information see
  \cite{
    ElbraechterPerekrestenkoGrohsEtAl2019%
  }.

  \item \emph{Optimization:} We confine the literature overview to neural tangent kernel based approaches, which are most relevant to this paper. The NTK is introduced in 
  \cite{
    JacotGabrielHongler2018%
  }
  and similar arguments together with convergence and perturbation analysis appear simultaneously in
  \cite{
    LiLiang2018,
    Allen-ZhuLiSong2019,
    DuZhaiPoczosSingh2019,
    DuLeeLiEtAl2019%
  },
  Related optimization ideas are further developed in many papers, including
  \cite{
    ZouCaoZhouGu2020,
    AroraDuHuEtAl2019a,
    LeeXiaoSchoenholzBahriNovakSohlDicksteinPennington2019,
    SongYang2019,
    ZouGu2019,
    KawaguchiHuang2019,
    ChizatOyallonBach2019,
    OymakSoltanolkotabi2020,
    NguyenMondelli2020,
    BaiLee2020,
    SongRamezaniKebryaPethickEftekhariCevher2021,
    LeeChoiRyuNo2022%
  }.
  In particular, 
  \cite{
    AroraDuHuEtAl2019,
    SuYang2019,
    JiTelgarsky2020,
    ChenCaoZouGu2021%
  } 
  refine the analysis based on expansions of the target $f$ in the NTK eigenbasis and are closely related to the arguments in this paper, with the major difference that they rely on the typical over-parametrized regime, whereas we do solemnly rely on smoothness.

  The papers
  \cite{
    GeigerJacotSpiglerGabrielSagundAscoliBiroliHonglerWyart2019,
    HaninNica2020,
    FortDziugaitePaulKharaghaniRoyGanguli2020,
    LeeSchoenholzPenningtonAdlamXiaoNovakSohl-Dickstein2020,
    SeleznovaKutyniok2022a,
    VyasBansalNakkiran2022%
  }
  discuss to what extend the linearization approach of the NTK can describe real neural network training. Characterizations of the NTK are fundamental for this paper and given
  \cite{
    BiettiMairal2019,
    GeifmanYadavKastenGalunJacobsRonen2020,
    JiTelgarskyXian2020,
    ChenXu2021%
  }.
  Convergence analysis for optimizing NTK models directly are in
  \cite{
    VelikanovYarotsky2021,
    VelikanovYarotsky2022%
  }.

  \item \emph{Approximation and Optimization:} Since the approximation question is under-parametrized and the optimization literature largely relies on over-parametrization there is little work on optimization methods for approximation. The gap between approximation theory and practice is considered in 
  \cite{
    AdcockDexter2020,
    GrohsVoigtlaender2021%
  }. The previous paper 
  \cite{
    GentileWelper2022a%
  }
  contains comparable results for $1d$ shallow networks. Similar approximation results for gradient flow trained shallow $1d$ networks are in
  \cite{
    JentzenRiekert2022,
    IbragimovJentzenRiekert2022%
  }, 
  with slightly different assumptions on the target $f$, more general probability weighted $L_2$ loss and an alternative proof technique. Other approximation and optimization guarantees rely on alternative optimizers.
  \cite{
    SiegelXu2022,
    HaoJinSiegelXu2021%
  }
  use greedy methods and 
  \cite{
    HerrmannOpschoorSchwab2022%
  }
  uses a two step procedure involving a classical and subsequent neural network approximation.

  $L_2$ error bounds are also proven in generalization error bounds for statistical estimation. E.g. the papers 
  \cite{
    DrewsKohler2022,
    KohlerKrzyzak2022%
  }
  show generalization errors for parallel fully connected networks in over-parametrized regimes with Hölder continuity.

\end{itemize}

\section{Main Result}

\subsection{Notations}

\begin{itemize}
  \item $\lesssim$, $\gtrsim$, $\sim$ denote less, bigger and equivalence up to a constant that can change in every occurrence and is independent of smoothness and number of weights. It can depend on the number of layers $L$ and input dimension $d$. Likewise, $c$ is a generic constant that can be different in each occurrence.
  \item $[n] := \{1, \dots, n\}$
  \item $\w = ij;\ell$ is the index of the weight $W_\w := W_{ij}^\ell$ with $|\w| := \ell$. Likewise, we set $\partial_\w = \frac{\partial}{\partial W_\w}$.
  \item $\odot$: Element wise product
  \item $A_{i\cdot}$ and $A_{\cdot j}$ are $i$th row and $j$th column of matrix $A$, respectively.
\end{itemize}

\subsection{Setup}
\label{sec:setup}

\paragraph{Neural Networks}

We train fully connected deep neural networks without bias and a few modifications: The first and last layer remain untrained, we use gradient flow instead of (stochastic) gradient descent and the first layer remains unscaled. For $x$ in some bounded domain $\dom \subset \real^d$, the networks are defined by
\begin{equation} \label{eq:setup:network}
  \begin{aligned}
    f^1(x) & = W^0 V x, & & \\
    f^{\ell+1}(x) & = W^\ell n_\ell^{-1/2} \activationp{f^\ell(x)}, & \ell = 1, \dots, L \\ 
    f(x) & = f^{L+1}(x), & & 
  \end{aligned}
\end{equation}
which we abbreviate by $f^\ell = f^\ell(x)$ if $x$ is unimportant or understood from context. The weights are initialized as follows
\begin{align*}
  & W^{L+1} \in \{-1,+1\}^{1 \times n_{L+1}} & & \text{i.i.d. Rademacher} & & \text{not trained}, \\
  & W^{\ell} \in \real^{n_{\ell+1} \times n_\ell}, \, \ell \in [L] & & \text{i.i.d. }\gaussian{0, 1} & & \text{trained}, \\
  & V \in \real^{n_0 \times d} & & \text{orthogonal columns }V^T V = I & & \text{not trained},
\end{align*}
all trained by gradient flow, except for the last layer $W^{L+1}$ and the first matrix $V$, which is pre-chosen with orthonormal columns. All layers have conventional $1/\sqrt{n_\ell}$ scaling, except for the first, which ensures that the NTK is of unit size on the diagonal and common in the literature \cite{DuLeeLiEtAl2019,BiettiMairal2019,GeifmanYadavKastenGalunJacobsRonen2020,ChenXu2021}. We also require that the layers are of similar size, except for the last one which ensures scalar valued output of the network
\begin{align*}
  m & := n_{L-1}, &
  1 & = n_{L+1} \le n_L \sim \dots \sim n_0 \ge d.
\end{align*}

\paragraph{Activation Functions}

We require comparatively smooth activation functions that have no more that linear growth
\begin{equation} \label{eq:setup:activation-growth}
  |\activationp{x}| \lesssim |x|,
\end{equation}
uniformly bounded first derivatives
\begin{align} \label{eq:setup:dactivation-bounded}
  |\activation^{(i)}(x)| & \lesssim 1 &
  i & = 1,2, &
  x & \in \real
\end{align}
and continuous second and third derivative with at most polynomial growth
\begin{align} \label{eq:setup:dactivation-growth}
  |\activation^{(i)}(x)| & \le p(x), &
  i & = 0,1,2,3,4
\end{align}
for some polynomial $p$ and all $x \in \real$.

\paragraph{Training}

We wish to approximate a function $f \in L_2(\dom)$ by neural networks and therefore use the $L_2(\dom)$ norm for the loss function
\begin{equation*}
  \loss(\theta) := \frac{1}{2} \|f_\theta - f\|_{L_2(\dom)}^2.
\end{equation*}
In the usual split up into approximation and estimation error in the machine learning literature, this corresponds to the former. It can also be understood as an infinite sample limit of the mean squared loss. This implies that we perform convergence analysis in an under-parametrized regime, different from the bulk of the neural network optimization literature, which typically relies on over-parametrization.

For simplicity, we optimize the loss by gradient flow
\begin{equation} \label{eq:setup:gradient-flow}
  \frac{d}{dt} \theta = -\nabla \loss(\theta)
\end{equation}
and not gradient descent or stochastic gradient descent.

\paragraph{Smoothness}

Since we are in an under-parametrized regime, we require smoothness of $f$ to guarantee meaningful convergence bounds. In this paper, we use Sobolev spaces $H^\sm(\Sd)$ on the sphere $\dom = \Sd$, with norms and scalar products denoted by $\|\cdot\|_{H^\sm(\Sd)}$ and $\dualp{\cdot, \cdot}_{H^\alpha(\Sd)}$. We drop the explicit reference to the domain $\Sd$ when convenient. Definitions and required properties are summarized in Section \ref{sec:sobolev:norms}.

\paragraph{Neural Tangent Kernel}

The analysis is based on the neural tangent kernel, which for the time being, we informally define as
\begin{equation} \label{eq:setup:ntk-limit}
  \ntk(x,y) 
  = \lim_{\text{width}\to\infty} \sum_{|\w| = L-1} \partial_\w f_r^{L+1}(x) \partial_\w f_r^{L+1}(y).
\end{equation}
The rigorous definition is in \eqref{eq:def-ntk}, based on an recursive formula as in \cite{JacotGabrielHongler2018}. Our definition differs slightly form the standard version because we only include weights from layer $|\w| = L-1$. We require that it is coercive in Sobolev norms
\begin{equation} \label{eq:setup:coercivity}
  \dualp{f, \int_\dom \ntk(\cdot,y) f(y) \, dy}_{H^{\Sm(\Sd)}}
  \gtrsim \|f\|_{H^{\Sm - \ntks}}
\end{equation}
for some $0 \le \sm \le \frac{\ntks}{2}$, $\Sm \in \{-\sm,\sm\}$ and all $f \in H^{\sm}(\Sd)$. For $\relu$ activations and regular NTK, including all layers, this property easily follows from \cite{BiettiMairal2019,GeifmanYadavKastenGalunJacobsRonen2020,ChenXu2021} as shown in Lemma \ref{lemma:supplements:sobolev:ntk-coercive}. However, our convergence theory requires smoother activations and therefore Section \ref{sec:coercivity} provides some numerical evidence, while a rigorous analysis is left for future research.

The paper \cite{JacotGabrielHongler2018} provides a recursive formula for the NTK, which in our simplified case reduces to
\begin{equation*}
  \ntk(x,y) = \dgp^L(x,y) \gp^{L-1}(x,y),
\end{equation*}
where $\dgp^L(x,y)$ and $\gp^{L-1}(x,y)$ are the covariances of two Gaussian processes that characterize the forward evaluation of the networks $W^L n_L^{1/2} \dactivationp{f^L}$ and $f^{L-1}$ in the infinite width limit, see Section \ref{sec:ntk} for their rigorous definition. We require that\begin{align} \label{eq:setup:gp-bounds}
  c_\gp \le \gp^k(x,x) & \le C_\gp > 0, &
\end{align}
for all $x,y \in \dom$, $k = 1, \dots, L$ and constants $c_\gp, C_\gp \ge 0$. As we see in Section \ref{sec:coercivity}, the kernels are zonal, i.e. they only depend on $x^T y$. Hence, with a slight abuse of notation \eqref{eq:setup:gp-bounds} simplifies to $\gp^k(x,x) = \gp^k(x^T x) = \gp(1) \ne 0$. In fact, for $\relu$ activation (which is not sufficiently differentiable for our results) the paper \cite{ChenXu2021} shows $\gp^k(x,x) = 1$.

\subsection{Result}
\label{sec:main}

We are now ready to state the main result of the paper.

\begin{theorem} \label{th:convergence}

  Assume that the neural network \eqref{eq:setup:network} - \eqref{eq:setup:dactivation-growth} is trained by gradient flow \eqref{eq:setup:gradient-flow}. Let $\res(t) := f_{\theta(t)} - f$ be the residual and assume:
  \begin{enumerate}
    \item The NTK satisfies coercivity \eqref{eq:setup:coercivity} for some $0 \le \sm \le \frac{\ntks}{2}$ and the forward process satisfies \eqref{eq:setup:gp-bounds}.
    \item All hidden layers are of similar size: $n_0 \sim \dots \sim n_{L-1} =: m$.
    \item Smoothness is bounded by $0 < \sm < 1/2$.
    \item $0 < \holder < 1 - \sm$ is an arbitrary number (used for Hölder continuity of the NTK in the proof).
    \item For $\tau$ specified below, $m$ is sufficiently large so that
    \begin{align*}
      \|\res(0)\|_{-\sm}^{\frac{1}{2}} \|\res(0)\|_\sm^{\frac{1}{2}} m^{-\frac{1}{2}} & \lesssim 1, &
      \frac{cd}{m} & \le 1, &
      \frac{\tau}{m} & \le 1.
    \end{align*}
  \end{enumerate}
  Then with probability at least $1 - c L (e^{-m} + e^{-\tau})$ we have
  \begin{align} \label{eq:main:convergence}
    \|\res(t)\|_{L_2(\Sd)}^2
    & \lesssim \left[ \wdiff^{\frac{\ntks \holder}{\ntks-\sm}} \|\res(0)\|_{H^\sm(\Sd)}^{\frac{\ntks}{\sm}} + \|\res(0)\|_{H^{-\sm}(\Sd)}^{\frac{\ntks}{\sm}} e^{-c\wdiff^{\frac{\ntks \holder}{\ntks-\sm}} \frac{\ntks}{2\sm} t} \right]^{\frac{\sm}{\ntks}}
      \|\res(0)\|_{H^\sm(\Sd)}
  \end{align}
  for some $\wdiff$ with
  \begin{align*}
    \wdiff 
    & \lesssim \max \left\{ \left[ \frac{\|\res(0)\|_{H^{-\sm}(\Sd)}^{\frac{1}{2}}\|\res(0)\|_{H^\sm(\Sd)}^{\frac{1}{2}}}{\sqrt{m}} \right]^{\frac{\ntks - \sm}{\ntks(1+\holder) - \sm}},\, c\sqrt{\frac{d}{m}} \right\}, &
    \tau & = \wdiff^{2\holder} m
  \end{align*}
  and generic constant $c \ge 0$, dependent on smoothness $\sm$, depth $L$ and dimension $d$, independent of width $m$ and residual $\res$.
  
\end{theorem}

All assumptions are easy to verify, except for the coercivity of the NTK \eqref{eq:setup:coercivity} and the bounds \eqref{eq:setup:gp-bounds} of the forward kernel, which we discuss in the next section. The error bound \eqref{eq:main:convergence} consists of two summands, only one of which depends on the gradient flow time $t$. For large $t$, it converges to zero and we are left with the first error term. This results in the following corollary, which provides a direct approximation result of type \eqref{eq:approx} for the outcome of gradient flow training.

\begin{corollary} \label{cor:approximation}

  Let all assumptions of Theorem \ref{th:convergence} be satisfied. Then for $m$ sufficiently large, with high probability (both as in Theorem \ref{th:convergence}), we have
  \begin{align*}
    \|\res\|_{L_2(\Sd)}
    & \lesssim 
    \max \left\{ \left[ \frac{C(\res(0))}{m} \right]^{\frac{1}{4} \frac{\sm \holder}{\ntks(1+\holder) - \sm}},\, \left[\frac{d}{m}\right]^{\frac{1}{4} \frac{\sm \holder}{\ntks-\sm}} \right\} \|\res(0)\|_{H^\sm(\Sd)},
    \\
    C(\res(0)) & = \|\res(0)\|_{H^{-\sm}(\Sd)} \|\res(0)\|_{H^\sm(\Sd)}
  \end{align*}
  where $\res := f_{\theta(t)} - f$ is the gradient flow residual for sufficiently large time $t$.
\end{corollary}

For traditional approximation methods, one would expect convergence rate $m^{-\sm/d}$ for functions in the Sobolev space $H^{\sm}$. Our rates are lower, which seems to be a variation of over-parametrization is disguise: In the over-parametrized as well as in our approximation regime the optimizer analysis seems to require some redundancy and thus more weights than necessary for the approximation alone. Of course, we only provide upper bounds and practical neural networks may perform better. Some preliminary experiments in \cite{GentileWelper2022a} show that shallow networks in one dimension outperform the theoretical bounds but are still worse than classical approximation theory would suggest. In addition, the linearization argument of the NTK results in smoothness measures in Hilbert spaces $H^{\sm}$ and not in larger $L_p$ based smoothness spaces with $p<2$ or even Barron spaces, as is common for nonlinear approximation.

\begin{remark}
  Although Theorem \ref{th:convergence} and Corollary \ref{cor:approximation} seem to show dimension independent convergence rates, they are not. Indeed, $\ntks$ depends on the dimension and smoothness of the activation function as we see in Section \ref{sec:coercivity} and Lemma \ref{lemma:supplements:sobolev:ntk-coercive}.
\end{remark}

\section{Coercivity of the NTK}
\label{sec:coercivity}

While most assumptions of Theorem \ref{th:convergence} are easy to verify, the coercivity \eqref{eq:setup:coercivity} is less clear. This section contains some results for the NTK $\ntk(x,y)$ in this paper, which only considers the second but last layer, as well as the regular NTK defined by the infinite width limit
\begin{equation*}
  \NTK(x,y) 
  = \lim_{\text{width}\to\infty} \sum_\w \partial f^{L+1}(x) \partial_p f^{L+1}(y)
\end{equation*}
of all layers. Coercivity easily follows once we understand the NTK's spectral decomposition. To this end, first note that $\ntk(x,y)$ and $\NTK(x,y)$ are both zonal kernels, i.e. they only depend on $x^Ty$, and as consequence their eigenfunctions are spherical harmonics.

\begin{lemma}[{\cite[Lemma 1]{GeifmanYadavKastenGalunJacobsRonen2020}}]
  The eigenfunctions of the kernels $\ntk(x,y)$ and $\NTK(x,y)$ on the sphere with uniform measure are spherical harmonics.
\end{lemma}

\begin{proof}
See \cite[Lemma 1]{GeifmanYadavKastenGalunJacobsRonen2020} and the discussion thereafter.
\end{proof}

Hence, it is sufficient to show lower bounds for the eigenvalues. These are provided in \cite{BiettiMairal2019,GeifmanYadavKastenGalunJacobsRonen2020,ChenXu2021} under slightly different assumptions than required in this paper:
\begin{enumerate}
  \item They use all layers $\NTK(x,y)$ instead of only the second but last one in $\ntk(x,y)$. (The reference \cite{DuLeeLiEtAl2019} does consider $\ntk(x,y)$ and shows that the eigenvalues are strictly positive in the over-parametrized regime with discrete loss and non-degenerate data.)
  \item They use bias, whereas we don't. We can however easily introduce bias into the first layer by the usual technique to incorporate one fixed input component $x_0 = 1$.
  \item The cited papers use $\relu$ activations, which do not satisfy the third derivative smoothness requirements \eqref{eq:setup:dactivation-bounded}.
\end{enumerate}
Anyways, with these modified assumptions, it is easy to derive coercivity from the NTK's RKHS in \cite{BiettiMairal2019,GeifmanYadavKastenGalunJacobsRonen2020,ChenXu2021}.

\newcommand{\lemmaCoercivity}{

  Let $\NTK(x,y)$ be the neural tangent kernel for a fully connected neural network with bias on the sphere $\Sd$ with $\relu$ activation. Then for any $\alpha \in \real$
  \begin{equation*}
    \dualp{f, L_\NTK f}_{H^\alpha(\Sd)}
    \gtrsim \|f\|_{H^{\alpha - d/2}(\Sd)}^2,
  \end{equation*}
  where $L_\NTK$ is the integral operator with kernel $\NTK(x,y)$.

}

\begin{lemma} \label{lemma:supplements:sobolev:ntk-coercive}
  \lemmaCoercivity
\end{lemma}

The proof is given at the end of Section \ref{sec:ntk-sphere}. Note that this implies $\ntks = d/2$ and thus Theorem \ref{th:convergence} cannot be expected to be dimension independent. In fact, due to smoother activations, the kernel $\ntk(x,y)$ is expected to be more smoothing than $\NTK(x,y)$ resulting in a faster decay of the eigenvalues and larger $\ntks$. This leads to Sobolev coercivity (Lemmas \ref{lemma:supplements:sobolev:ntk-eigenvalues} and \ref{lemma:supplements:sobolev:ntk-coercive}) as long as the decay is polynomial, which we only verify numerically in this paper, as shown in Figure \ref{fig:ntk-ev} for \numericalConst{$n=100$} uniform samples on the \numericalConst{$d=2$} dimensional sphere and \numericalConst{$L-1 = 1$} hidden layers of width \numericalConst{$m=1000$}. The plot uses log-log axes so that straight lines represent polynomial decay. As expected, $\relu$ and ELU activations show polynomials decay with higher order for the latter, which are smoother. For comparison the $C^\infty$ activation $GELU$ seems to show super polynomial decay. However, the results are preliminary and have to be considered carefully:
\begin{enumerate}
  \item The oscillations at the end, are for eigenvalues of size $\sim 10^{-7}$, which is machine accuracy for floating point numbers.
  \item Most eigenvalues are smaller than the difference between the empirical NTK and the actual NTK. For comparison, the difference between two randomly sampled empirical NTKs (in matrix norm) is: %
ReLU: $0.280$, ELU: $0.524$, GELU: $0.262$ %
.
  \item According to \cite{BiettiMairal2019}, for shallow networks without bias, every other eigenvalue of the NTK should be zero. This is not clear from the experiments (which do not use bias, but have one more layer), likely because of the large errors in the previous item.
  \item The errors should be better for wider hidden layers, but since the networks involve dense matrices, their size quickly becomes substantial.
\end{enumerate}
In conclusion, the experiments show the expected polynomial decay of NTK eigenvalues and activations with singularities in higher derivatives, but the results have to be regraded with care.

\begin{figure}
  \centering\includegraphics[width=5cm]{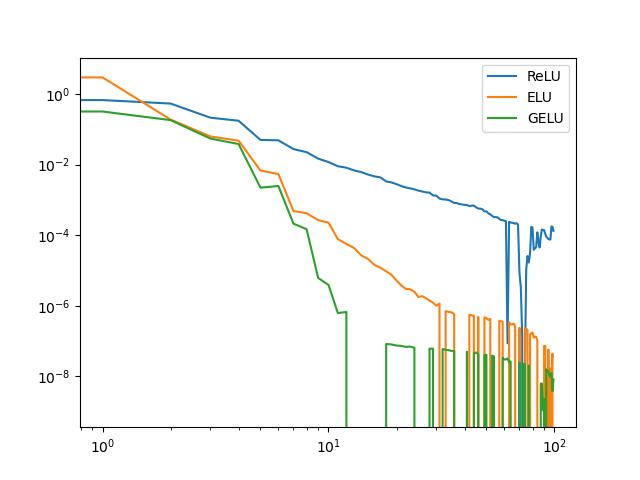}
  \caption{Eigenvalues of the NTK $\ntk(x,y)$ for different activation functions.}
  \label{fig:ntk-ev}
\end{figure}

\section{Proof Overview}
\label{sec:proof-overview}

\subsection{Preliminaries}

\subsubsection{Neural Tangent Kernel}
\label{sec:ntk}

In this section, we recall the definition of the neural tangent kernel (NTK) and setup notations for its empirical variants. Our definition differs slightly from the literature because we only use the last hidden layer (weights $W^{L-1}$) to reduce the loss, whereas all other layers are trained but only estimated by a perturbation analysis. Throughout the paper, we only need the definitions as stated, not that they are the infinite width limit of the network derivatives as stated in \eqref{eq:setup:ntk-limit}, although we sometimes refer to this for motivation.

As usual, we start with the recursive definition of the covariances
\begin{align*}
  \gp^{\ell+1}(x,y) & := \EE{u,v \sim \gaussian{0, A}}{\activationp{u}, \activationp{v}}, &
  A & = \begin{bmatrix}
    \gp^\ell(x,x) & \gp^\ell(x,y) \\
    \gp^\ell(y,x) & \gp^\ell(y,y)
  \end{bmatrix}, & 
  \gp^0(x,y) & = x^T y,
\end{align*}
which define a Gaussian process that is the infinite width limit of the forward evaluation of the hidden layer $f^\ell(x)$, see \cite{JacotGabrielHongler2018}. Likewise, we define
\begin{align*}
  \dgp^{\ell+1}(x,y) & := \EE{u,v \sim \gaussian{0, A}}{\dactivationp{u}, \dactivationp{v}}, &
  A & = \begin{bmatrix}
    \gp^\ell(x,x) & \gp^\ell(x,y) \\
    \gp^\ell(y,x) & \gp^\ell(y,y)
  \end{bmatrix},
\end{align*}
with activation function of the last layer is exchanged with its derivative. Then the neural tangent kernel (NTK) is defined by
\begin{equation} \label{eq:def-ntk}
  \ntk(x,y) := \dgp^L(x,y) \gp^{L-1}(x,y).
\end{equation}
The paper \cite{JacotGabrielHongler2018} shows that all three definitions above are infinite width limits of the corresponding empirical processes (denoted with an extra hat $\hat{\cdot}$)
\begin{equation} \label{eq:def-gp}
  \begin{aligned}
    \egp^\ell(x,y) 
    & := \frac{1}{n_\ell}\sum_{r=1}^{n_\ell} \activationp{f_r^\ell(x)} \activationp{f_r^\ell(y)}
    = \frac{1}{n_\ell} \activationp{f^\ell(x)}^T \activationp{f^\ell(y)}, \\ 
    \edgp^\ell(x,y) 
    & := \frac{1}{n_\ell}\sum_{r=1}^{n_\ell} \dactivationp{f_r^\ell(x)} \dactivationp{f_r^\ell(y)}
    = \frac{1}{n_\ell} \dactivationp{f^\ell(x)}^T \dactivationp{f^\ell(y)}
  \end{aligned}
\end{equation}
and
\begin{equation*}
  \entk(x,y) := \sum_{|\w| = L-1} \partial_\w f_r^{L+1}(x) \partial_\w f_r^{L+1}(y).
\end{equation*}
Note that unlike the usual definition of the NTK, we only include weights from the second but last layer. Formally, we do not show that $\gp^\ell$, $\dgp^\ell$ and $\ntk$ arise as infinite width limits of the empirical versions $\egp^\ell$, $\edgp^\ell$ and $\entk$, but rather concentration inequalities between them.

The next lemma shows that the empirical kernels satisfy the same identity \eqref{eq:def-ntk} as their limits.

\begin{lemma} \label{lemma:ntk:entk-as-gp}
  Assume that $W_{ij}^L \in \{-1, +1\}$. Then  
  \[
    \entk(x,y) = \edgp^L(x,y) \egp^{L-1}(x,y).
  \]
\end{lemma}

\begin{proof}

By definitions of $f^L$ and $f^{L-1}$, we have
\begin{align*}
  \partial_{W_{ij}^{L-1}} f_r^{L+1}
  & = \sum_{1=r}^{n_L} W_{\cdot r}^L n_L^{-1/2} \partial_{W_{ij}^{L-1}} \activationp{f_r^L}
  \\
  & = \sum_{1=r}^{n_L} W_{\cdot r}^L n_L^{-1/2} \dactivationp{f_r^L} \partial_{W_{ij}^{L-1}} f_r^L
  \\
  & = \sum_{1=r}^{n_L} W_{\cdot r}^L n_L^{-1/2} \dactivationp{f_r^L} \delta_{ir} n_{L-1}^{-1/2}  \activationp{f_j^{L-1}}
  \\
  & = W_{\cdot i}^L n_L^{-1/2} n_{L-1}^{-1/2} \dactivationp{f_i^L} \activationp{f_j^{L-1}}.
\end{align*}
It follows that
\begin{align*}
  \entk(x,y)
  & = \sum_{i=1}^{n_L} \sum_{j=1}^{n_{L-1}} 
      \partial_{W_{ij}^{L-1}} f_r^{L+1}(x) 
      \partial_{W_{ij}^{L-1}} f_r^{L+1}(y)
  \\
  & = \frac{1}{n_L} \sum_{i=1}^{n_L} \frac{1}{n_{L-1}} \sum_{j=1}^{n_{L-1}} 
      \left| W_{\cdot i}^L \right|^2
      \dactivationp{f_i^L(x)} \dactivationp{f_i^L(y)} 
      \activationp{f_j^{L-1}(x)} \activationp{f_j^{L-1}(y)}
  \\
  & = \edgp^L(x,y) \egp^{L-1}(x,y),
\end{align*}
where in the last step we have used that $\left| W_{\cdot i}^L \right|^2 = 1$ by assumption and the definitions of $\edgp^L$ and $\egp^{L-1}$.

\end{proof}

The NTK and empirical NTK induce integral operators, which we denote by
\begin{align*}
  \ntkop f & := \int_\dom \ntk(\cdot,y) f(y) \, dy, &
  \ntkop_\theta f & := \int_\dom \entk(\cdot,y) f(y) \, dy
\end{align*}
The last definition makes the dependence on the weights explicit, which is hidden in $\entk$.

\subsubsection{Norms}

We use several norms for our analysis.
\begin{enumerate}

  \item \emph{$\ell_2$ and matrix norms:} $\|\cdot\|$ denotes the $\ell_2$ norm when applied to a vector and the matrix norm when applied to a matrix.

  \item \emph{Hölder norms} $\|\cdot\|_\CHND{\alpha}{\dom; V}$ for functions $f \colon \dom \subset \real^d \to V$ into some normed vector space $V$, with Hölder continuity measured in the $V$ norm
  \begin{align*}
    \|f\|_\CND{0}{\dom;V}
    := \sup_{x \in D} \|f(x)\|_V
      + \sup_{x \ne \px \in D} \frac{\|f(x) - f(\px)\|_V}{\|x - \px\|_U^\alpha}.
  \end{align*}
  We drop $V$ in $\|\cdot\|_\CHND{\alpha}{\dom}$ when $V = \ell_2$ and $\dom$ in $\|\cdot\|_\CHN{\alpha}$ when it is understood from context. We also use alternate definitions as the supremum over the finite difference operator
  \begin{align*}
    \Delta^0_h f(x) & = f(x), &
    \Delta_h^\alpha f(x) & = \|h\|_U^{-\alpha} [f(x+h) - f(x)], &
    \alpha & > 0,
  \end{align*}
  See Section \ref{sec:supplements:holder-spaces} for the full definitions and basic properties.

  \item \emph{Mixed Hölder norms} $\|\cdot\|_\CHHND{\alpha}{\beta}{\dom; V}$ for functions $f \colon \dom \times \dom \subset \real^d \to V$ of two variables. They measure the supremum of all mixed finite difference operators $\Delta_{x,h_x}^s \Delta_{y,h_y}^t$ for any $s \in \{0, \alpha\}$ and $t \in \{0, \beta\}$, similar to Sobolev spaces with mixed smoothness. As for Hölder norms for one variable, we use two different definitions, which are provided in Section \ref{sec:supplements:holder-spaces}.

  \item \emph{Sobolev Norms on the Sphere} denoted by $\|\cdot\|_{H^\alpha(\Sd)}$. Definitions and properties are provided in Section \ref{sec:sobolev:norms}. The bulk of the analysis is carried out in Hölder norms, which control Sobolev norms by
  \begin{equation*}
    \|\cdot\|_{H^\alpha(\Sd)} \lesssim \|\cdot\|_\CHND{\alpha+\epsilon}{\Sd}.
  \end{equation*}
  for $\epsilon > 0$, see Lemma \ref{lemma:supplements:sobolev-holder}.

  \item \emph{Generic Smoothness norms} $\|\cdot\|_\sm$, $\sm \in \real$ for associated Hilbert spaces $\hs^\sm$. These are used in abstract convergence results and later replaced by Sobolev norms.

  \item \emph{Orlicz norms} $\|\cdot\|_{\psi_i}$ for $i=1,2$ measure sub-gaussian and sub-exponential concentration. Some required results are summarized in Section \ref{sec:supplements:concentration}.

  \item \emph{Gaussian weighted $L_2$ norms} defined by
  \begin{align*}
    \|f\|_N^2 & = \dualp{f,f}_N, &
    \dualp{f,g} & = \int_\real f(x)^2 d\gaussian{0,1}(x)
  \end{align*}

\end{enumerate}

\subsubsection{Neural Networks}

Many results use a generic activation function denoted by $\activation$ with derivative $\dactivation$, which is allowed to change in each layer, although we always use the same symbol for notational simplicity. They satisfy the linear growth condition
\begin{equation} \label{eq:assumption:activation-growth}
  |\activationp{x}| \lesssim |x|,
\end{equation}
are Lipschitz
\begin{equation} \label{eq:assumption:activation-lipschitz}
  |\activationp{x} - \activationp{\px}| \lesssim |x - \px|
\end{equation}
and have uniformly bounded derivatives
\begin{equation} \label{eq:assumption:dactivation-bounded}
  |\dactivationp{x}| \lesssim 1.
\end{equation}

\subsection{Abstract Convergence result}

We first show convergence in a slightly generalized setting. To this end, we consider neural networks as maps from the parameter space to the square integrable functions $f_\cdot \colon \wdom \subset \ell_2(\real^m) \to L_2(\dom)$ defined by $\theta \to f_\theta(\cdot)$. More generally, for the time being, we replace $L_2(\dom)$ by an arbitrary Hilbert space $\hs$ and the network by an arbitrary Fr\'{e}chet differentiable function 
\begin{align*}
  & f: \wdom = \ell_2(\real^m) \to \hs, & \theta & \to f_\theta.
\end{align*}
For a target function $f \in \hs$, we define the loss
\[
  L(\theta) = \frac{1}{2} \|f_\theta - f\|_\hs^2
\]
and the corresponding gradient flow for $\theta(t)$
\begin{equation} \label{eq:gradient-flow}
  \frac{d}{dt} \theta(t) = - \nabla L(\theta),
\end{equation}
initialized with random $\theta(0)$. The convergence analysis relies on a regime where the evolution of the gradient flow is governed by its linearization
\begin{equation*}
  \ntkop_\theta := D f_\theta (Df_\theta)^*,
\end{equation*}
where $*$ denotes the adjoint and $H_\theta$ is the empirical NTK if $f_\theta$ is a neural network. To describe the smoothness of the target and spectral properties of $H_\theta$, we use a series of Hilbert spaces $\hs^\sm$ for some smoothness index $\sm \in \real$ so that $\hs^0 = \hs$. As stated in the lemma below, they satisfy interpolation inequalities and coercivity conditions. In this abstract framework, we show convergence as follows.

\begin{lemma} \label{lemma:convergence:perturbed-gradient-flow}
  Let $\theta(t)$ be defined by the gradient flow \eqref{eq:gradient-flow}, $\res = f_\theta - f$ be the residual and $m$ be a number that satisfies all assumptions below, which is typically related to the degrees of freedom. For constants $c_\infty, c_0, \ntks, \holder > 0$ and $0 \le \sm \le \frac{\ntks}{2}$, functions $p_0(m), p_\infty(\tau)$, $p_L(m,\wdiff)$ and weight norm $\wnorm{\cdot}$ assume that:
  \begin{enumerate}

    \item \label{item:perturbed-gradient-flow:weight-distance} With probability at least $1-p_0(m)$, the distance of the weights from their initial value is controlled by
    \begin{align} \label{eq:perturbed-gradient-flow:weight-diff}
      \wnorm{\theta(t) - \theta(0)} & \le 1 &
      & \Rightarrow & 
      \wnorm{\theta(t) - \theta(0)}
      & \lesssim \sqrt{\frac{2}{m}} \int_0^t \|\res(\tau)\|_0 \, d\tau.
    \end{align}

    \item \label{item:perturbed-gradient-flow:norms} The norms and scalar product satisfy interpolation and continuity
    \begin{align} \label{eq:interpolation-inequality}
      \|\cdot\|_b & \lesssim \|\cdot\|_a^{\frac{c-b}{c-a}} \|\cdot\|_c^{\frac{b-a}{c-a}}, &
      \dualp{\cdot, \cdot}_{-\sm} & \lesssim \|\cdot\|_{-3\sm} \|\cdot\|_\sm,
    \end{align}
    for all $-\sm-\ntks \le a \le b \le c \le \sm$.

    \item \label{item:perturbed-gradient-flow:initial} Let $\ntkop \colon \hs^\sm \to \hs^{-\sm}$ be an operator that satisfies the concentration inequality
    \begin{equation} \label{eq:perturbed-gradient-flow:initial}
      \pr{\|\ntkop - \ntkop_{\theta(0)}\|_{\sm \leftarrow -\sm} \ge c\sqrt{\frac{d}{m}} + \sqrt{\frac{c_\infty \tau}{m}}} \le p_\infty(\tau)
    \end{equation}
    for all $\tau$ with $\sqrt{\frac{c_\infty \tau}{m}} \le 1$. (In our application $\ntkop$ is the NTK and $\ntkop_{\theta(0)}$ the empirical NTK.)

    \item \label{item:perturbed-gradient-flow:perturbation} Hölder continuity with high probability:
    \begin{multline} \label{eq:perturbed-gradient-flow:perturbation}
      \pr{\exists \, \ptheta \in \wdom\text{ with } \wnorm{\ptheta - \theta(0)} \le \wdiff\text{ and }\|\ntkop_\ptheta - \ntkop_{\theta(0)}\|_{\sm \leftarrow -\sm} \ge c_0 \wdiff^\holder}
      \\
      \le p_L(m,\wdiff)
    \end{multline}
    for all $0 < \wdiff \le 1$.

    \item \label{item:perturbed-gradient-flow:coercive} $\ntkop$ is coercive for $\Sm \in \{-\sm, \sm\}$
    \begin{align} \label{eq:perturbed-gradient-flow:coercive}
      \|v\|_{\Sm-\ntks}^2 & \lesssim \dualp{v, \ntkop v}_\Sm, &
      v & \in \hs^{\Sm-\ntks}
    \end{align}

    \item For $\tau$ specified below, $m$ is sufficiently large so that
    \begin{align*}
      \|\res(0)\|_{-\sm}^{\frac{1}{2}} \|\res(0)\|_\sm^{\frac{1}{2}} m^{-\frac{1}{2}} & \lesssim 1, &
      \frac{cd}{m} & \le 1, &
      \frac{\tau}{m} & \le 1.
    \end{align*}

  \end{enumerate}
  Then with probability at least $1 - p_0(m) - p_\infty(\tau) - p_L(m,\wdiff)$ we have
  \begin{align*}
    \|\res\|_{-\sm}^2
    & \lesssim \left[ \wdiff^{\frac{\ntks \holder}{\ntks-\sm}} \|\res(0)\|_\sm^{\frac{\ntks}{\sm}} + \|\res(0)\|_{-\sm}^{\frac{\ntks}{\sm}} e^{-c\wdiff^{\frac{\ntks \holder}{\ntks-\sm}} \frac{\ntks}{2\sm} t} \right]^{\frac{2\sm}{\ntks}}
    \\
    \|\res\|_\sm^2
    & \lesssim \|\res(0)\|_\sm^2
  \end{align*}
  for some $\wdiff$ with
  \begin{align*}
    \wdiff 
    & \lesssim \max \left\{ \left[ \frac{\|\res(0)\|_{-\sm}^{\frac{1}{2}} \|\res(0)\|_\sm^{\frac{1}{2}}}{\sqrt{m}} \right]^{\frac{\ntks - \sm}{\ntks(1+\holder) - \sm}},\, c\sqrt{\frac{d}{m}} \right\}, &
    \tau & = \wdiff^{2\holder} m
  \end{align*}
  and generic constants $c \ge 0$ dependent of $\sm$ and independent of $\res$ and $m$.

\end{lemma}

We defer the proof to Section \ref{sec:convergnece} and only consider a sketch here. As for standard NTK arguments, the proof is based on the following observation
\begin{equation}
  \frac{1}{2} \frac{d}{dt} \|\res\|^2
  = - \dualp{\res, \ntkop_{\theta(t)} \, \res}
  \approx - \dualp{\res, \ntkop \, \res}
\end{equation}
which can be shown by a short computation. The last step relies on the observation that empirical NTK stays close to its initial $\ntkop_{\theta(t)} \approx \ntkop_{\theta(0)}$ and that the initial is close to the infinite width limit $\ntkop_{\theta(0)} \approx \ntkop$. However, since we are not in an over-parametrized regime, the NTK's eigenvalues can be arbitrarily close to zero and we only have coercivity in the weaker norm $\dualp{\res, \, \ntkop \, \res} \gtrsim \|\res\|_{-\sm}$, which is not sufficient to show convergence by e.g. Grönwall's inequality. To avoid this problem, we derive a closely related system of coupled ODEs
\begin{align*}
  \frac{1}{2} \frac{d}{dt} \|\res\|_{-\sm}^2
  & \lesssim - c \|\res\|_{-\sm}^{2\frac{2\sm+\ntks}{2\sm}} \|\res\|_\sm^{- 2\frac{\ntks}{2\sm}}+ \wdiff^{\frac{\ntks \holder}{\ntks-\sm}} \|\res\|_{-\sm}^2
  \\
  \frac{1}{2} \frac{d}{dt} \|\res\|_\sm^2
  & \lesssim - c \|\res\|_{-\sm}^{2\frac{\ntks}{2\sm}} \|\res\|_\sm^{2\frac{2\sm-\ntks}{2\sm}} + \wdiff^{\holder} \|\res\|_\sm \|\res\|_{-\sm}.
\end{align*}
The first one is used to bound the error in the $\hs^{-\sm}$ norm and the second ensures that the smoothness of the residual $\res(t)$ is uniformly bounded during gradient flow. Together with the interpolation inequality \eqref{eq:interpolation-inequality}, this shows convergence in the $\hs = \hs^0$ norm.

It remains to verify all assumption of Lemma \ref{lemma:convergence:perturbed-gradient-flow}, which we do in the following subsections. Details are provided in Section \ref{sec:proof:th:convergence}.

\subsection{Assumption (\ref{eq:perturbed-gradient-flow:perturbation}): Hölder continuity}

We use a bar $\bar{\cdot}$ to denote perturbation, in particular $\pW^\ell$ is a perturbed weight, and $\pentk$ is the corresponding empirical neural tangent kernel. In order to obtain continuity results, we require that the weight matrices and domain are bounded
\begin{align} \label{eq:assumption:weights-domain-bounded}
   \left\|W^\ell\right\| n_\ell^{-1/2} & \lesssim 1, & 
   \left\|\pW^\ell\right\| n_\ell^{-1/2} & \lesssim 1, &
   \|x\| & \lesssim 1 \, \forall x \in \dom.
\end{align}
For the initial weights $W^\ell$, this holds with high probability because its entries are i.i.d. standard Gaussian. For perturbed weights we only need continuity bounds under the condition that $\wnorm{\theta - \ptheta} \le 1$ or equivalently that $\|W^\ell - \pW^\ell\| n_\ell^{-1/2} \le 1$ so that the weight bound of the perturbation $\pW^\ell$ follow from the bounds for $W^\ell$. With this setup, we show the following lemma.

\newcommand{\LemmaEntkHolder}{
  Assume that $\activation$ and $\dactivation$ satisfy the growth and Lipschitz conditions \eqref{eq:assumption:activation-growth}, \eqref{eq:assumption:activation-lipschitz} and may be different in each layer. Assume the weights, perturbed weights and domain are bounded \eqref{eq:assumption:weights-domain-bounded} and $n_L \sim n_{L-1} \sim \dots \sim n_0$. Then for $0 < \alpha < 1$
  \begin{align*}
    \left\|\entk\right\|_\CHHN{\alpha}{\alpha} & \lesssim 1
    \\
    \left\|\pentk\right\|_\CHHN{\alpha}{\alpha} & \lesssim 1
    \\
    \left\|\entk - \pentk\right\|_\CHHN{\alpha}{\alpha}
    & \lesssim \frac{n_0}{n_L} \left[ \sum_{k=0}^{L-1} \left\|W^k - \pW^k \right\| n_k^{-1/2} \right]^{1-\alpha}.
  \end{align*}
}

\begin{lemma} \label{lemma:continuity:entk-holder}
  \LemmaEntkHolder
\end{lemma}

The proof is at the end of Section \ref{sec:ntk-holder}. The lemma shows that the kernels $\left\|\entk^\ell - \pentk^\ell\right\|_\CHHN{\alpha}{\alpha}$ are Hölder continuous (w.r.t. weights) in a Hölder norm (w.r.t. $x$ and $y$). This directly implies that the induced integral operators $\|H_\theta - H_\ptheta\|_{\sm \leftarrow -\sm}$ are bounded in operator norms induced by Sobolev norms (up to $\epsilon$ less smoothness), which implies Assumption \eqref{eq:perturbed-gradient-flow:perturbation}, see Section \ref{sec:proof:th:convergence} for details.

\subsection{Assumption (\ref{eq:perturbed-gradient-flow:initial}): Concentration}

For concentration, we need to show that the empirical NTK is close to the NTK, i.e. that $\|\ntkop - \ntkop_{\theta(0)}\|_{\sm \leftarrow -\sm}$ is small in the operator norm. To this end, it suffices to bound the corresponding integral kernels $\|\ntk - \entk\|_\CHHN{\sm+\epsilon}{\sm+\epsilon}$ in Hölder norms with slightly higher smoothness, see Lemma \ref{lemma:supplements:kernel-bound}. Concentration is then provided by the following Lemma. See the end of Section \ref{sec:concentration} for a proof and Section \ref{sec:proof:th:convergence} for its application in the proof of the main result.

\newcommand{\LemmaConcentration}[1]{
  Let $\alpha = \beta = 1/2$ and $k = 0, \dots, L-1$.
  \begin{enumerate}
    \item Assume that $W^L \in \{-1, +1\}$ with probability $1/2$ each.
    \item Assume that all $W^k$ are are i.i.d. standard normal. 
    \item Assume that $\activation$ and $\dactivation$ satisfy the growth condition \eqref{eq:assumption:activation-growth}, have uniformly bounded derivatives \eqref{eq:assumption:dactivation-bounded}, derivatives $\activation^{(i)}$, $i=0, \dots, 3$ are continuous and have at most polynomial growth for $x \to \pm \infty$ and the scaled activations satisfy
    \begin{align*}
      \left\|\partial^i (\activation_a) \right\|_N & \lesssim 1, & 
      \left\|\partial^i (\dactivation_a) \right\|_N & \lesssim 1, & 
      a & \in \{\gp^k(x,x): x \in \dom\}, &
      i & = 1, \dots, 3,
    \end{align*}
    with $\activation_a(x) := \activation(ax)$. The activation functions may be different in each layer.
    \item For all $x \in \dom$ assume
    \begin{equation*}
      \gp^k(x,x) \ge c_\gp > 0.
    \end{equation*}
    \item The widths satisfy $n_\ell \gtrsim n_0$ for all $\ell=0, \dots, L$.
  \end{enumerate}
  Then, with probability at least
  \begin{equation} #1
    1 - c \sum_{k=1}^{L-1} e^{-n_k} + e^{-u_k}
  \end{equation}
  we have
  \begin{equation*}
    \left\| \entk - \ntk \right\|_\CHHN{\alpha}{\beta} 
    \lesssim \sum_{k=0}^{L-1} \frac{n_0}{n_k}\left[ \frac{\sqrt{d}+\sqrt{u_k}}{\sqrt{n_k}} + \frac{d+u_k}{n_k} \right]
    \le \frac{1}{2} c_\gp
  \end{equation*}
  for all $u_1, \dots, u_{L-1} \ge 0$ sufficiently small so that the rightmost inequality holds.
}

\begin{lemma} \label{lemma:concentration:concentration-ntk}
  \LemmaConcentration{}
\end{lemma}

\subsection{Assumption (\ref{eq:perturbed-gradient-flow:weight-diff}): Weights stay Close to Initial}

Assumption \eqref{eq:perturbed-gradient-flow:weight-diff} follows from the following lemma, which shows that the weights stay close to their random initialization. Again, the estimates are proven in Hölder norms, which control the relevant Sobolev norms, see Section \ref{sec:proof:th:convergence} for details.

\newcommand{\LemmaCloseToInitial}{
  Assume that $\activation$ satisfies the growth and derivative bounds \eqref{eq:assumption:activation-growth}, \eqref{eq:assumption:dactivation-bounded} and may be different in each layer. Assume the weights are defined by the gradient flow \eqref{eq:setup:gradient-flow} and satisfy
  \begin{align*}
    \|W^\ell(0)\| n_\ell^{-1/2} & \lesssim 1, &
    \ell & =1, \dots, L, &
    \\
    \|W^\ell(0) - W^\ell(\tau)\| n_\ell^{-1/2} & \lesssim 1, &
    0 & \le \tau < t.
  \end{align*}
  Then
  \begin{equation*}
    \left\|W^\ell(t) - W^\ell(0)\right\| n_\ell^{-1/2}
    \lesssim \frac{n_0^{1/2}}{n_\ell} \int_0^t \|\res\|_{\CND{0}{\dom}'} \, dx \, d\tau,
  \end{equation*}
  where $\CND{0}{\dom}'$ is the dual space of $\CND{0}{\dom}$.
}

\begin{lemma} \label{lemma:close-to-initial:close-to-initial}
  \LemmaCloseToInitial
\end{lemma}

\section{Proof of the Main Result}
\label{sec:proof-details}

\subsection{Proof of Lemma \ref{lemma:convergence:perturbed-gradient-flow}: Generalized Convergence}
\label{sec:convergnece}

\paragraph{NTK Evolution}

In this section, we prove the convergence result in Lemma \ref{lemma:convergence:perturbed-gradient-flow}. Let us first recall the evolution of the loss in NTK theory. The Fr\'{e}chet derivative of the loss is
\begin{align*}
  D L(\theta) v
  & = \dualp{\res, (D f_\theta) v}
  = \dualp{ (D f_\theta)^* \res, v}, &
  & \text{for all }v \in \wdom
\end{align*}
and the gradient of the loss is the Riesz lift of the derivative
\begin{equation} \label{eq:gradient}
  \nabla L(\theta)
  = (D f_\theta)^* \res.
\end{equation}
Using the chain rule, we obtain the evolution of the residual
\begin{equation} \label{eq:gradient-flow-residual}
  \frac{d\res}{dt}
  = (D f_\theta) \frac{d\theta}{dt}
  = - (D f_\theta) \nabla L(\theta)
  = - (D f_\theta) (D f_\theta)^* \res
  =: \ntkop_\theta \res
\end{equation}
and the loss in any $\hs^\Sm$ norm
\begin{equation} \label{eq:gradient-flow-loss}
  \frac{1}{2} \frac{d}{dt} \|\res\|_\Sm^2
  = \dualp{\res, \frac{d\res}{dt} }_S
  = - \dualp{\res, (D f_\theta) (D f_\theta)^* \res}_S
  = - \dualp{\res, \ntkop_\theta \, \res}_\Sm,
\end{equation}
with
\begin{equation*}
   \ntkop_\theta := (D f_\theta) (D f_\theta)^*.
\end{equation*}

\paragraph{Proof of Lemma \ref{lemma:convergence:perturbed-gradient-flow}}

\begin{proof}[Proof of Lemma \ref{lemma:convergence:perturbed-gradient-flow}]

For the time being, we assume that the weights remain within a finite distance
\begin{equation} \label{eq:perturbed-gradient-flow:close-to-initial}
  \wdiff 
  := \max \left\{\sup_{t \le T} \wnorm{\theta(t) - \theta(0)}, c \sqrt{\frac{d}{m}} \right\}
  \le 1
\end{equation}
to their initial up to a time $T$ to be determined below, but sufficiently small so that the last inequality holds. With this condition, we can bound the time derivatives of the loss $\|\res\|_{-\sm}$ and the smoothness $\|\res\|_\sm$. For $\Sm \in \{-\sm,\sm\}$ and respective $\bar{S} \in \{-3\sm, \sm\}$, we have already calculated the exact evolution in \eqref{eq:gradient-flow-loss}, which we estimate by
\begin{align*}
  \frac{1}{2} \frac{d}{dt} \|\res\|_\Sm^2
  & = - \dualp{ \res, \ntkop_{\theta(t)} \res }_S
  \\
  & = - \dualp{ \res, \ntkop \res }_S
      + \dualp{ \res, (\ntkop - \ntkop_{\theta(0)}) \res }_S
      + \dualp{ \res, (\ntkop_{\theta(0)} - \ntkop_{\theta(t)}) \res }_\Sm.
\end{align*}
We estimate the last two summands as
\[
  \dualp{\res, [\dots] \res}_S
  \le \|\res\|_{\bar{S}} \|[\dots] \res\|_s
  \le \|\res\|_{\bar{S}} \|[\dots]\|_{\sm \leftarrow -\sm} \|\res\|_{-\sm},
\]
where ${\bar{S}}=\sm$ for $\Sm=\sm$ and ${\bar{S}}=-3\sm$ for $\Sm=-\sm$ by Assumption \ref{item:perturbed-gradient-flow:norms}. Then, we obtain
\begin{align*}
  \frac{1}{2} \frac{d}{dt} \|\res\|_\Sm^2
  & \le - \dualp{ \res, \ntkop \res }_S
        + \|\ntkop - \ntkop_{\theta(0)}\|_{\sm \leftarrow -\sm} \|\res\|_{\bar{S}} \|\res\|_{-\sm}
        + \|\ntkop_{\theta(0)} - \ntkop_{\theta(t)}\|_{\sm \leftarrow -\sm} \|\res\|_{\bar{S}} \|\res\|_{-\sm}
  \\
  & \le - \dualp{ \res, \ntkop \res }_S
  + \left[ c \sqrt{\frac{d}{m}} + \sqrt{\frac{c_\infty \tau}{m}} + c_0 \wdiff^{\holder} \right] \|\res\|_{\bar{S}} \|\res\|_{-\sm},
  \\
  & \lesssim - c \|\res\|_{\Sm-\ntks}^2 + \wdiff^{\holder} \|\res\|_{\bar{S}} \|\res\|_{-\sm},
\end{align*}
with probability at least $1-p_\infty(\tau) - p_L(m,\wdiff)$, where the second but last inequality follows from assumptions \eqref{eq:perturbed-gradient-flow:initial}, \eqref{eq:perturbed-gradient-flow:perturbation} and in the last inequality we have used the coercivity, \eqref{eq:perturbed-gradient-flow:close-to-initial} and chosen $\tau = \wdiff^{2\holder} m$ so that $\sqrt{\frac{c_\infty \tau}{m}} \lesssim \wdiff^\holder$. The left hand side contains one negative term $-\|\res\|_{\Sm-\ntks}^2$, which decreases the residual $\frac{d}{dt}\|\res\|_\Sm^2$, and one positive term which enlarges it. In the following, we ensure that these terms are properly balanced.

We eliminate all norms that are not $\|\res\|_{-\sm}$ or $\|\res\|_\sm$ so that we obtain a closed system of ODEs in these two variables. We begin with $\|\res\|_{\bar{S}}$, which is already of the right type if ${\bar{S}}=\sm$ but $\|\res\|_{-3\sm}$ for ${\bar{S}}=-\sm$. Since $0 < \sm < \frac{\ntks}{2}$, we have $-\sm-\ntks \le -3\sm \le \sm$ so that we can invoke the interpolation inequality from Assumption \ref{item:perturbed-gradient-flow:norms}
\[
  \|v\|_{-3\sm} \le \|v\|_{-\sm-\ntks}^{\frac{2\sm}{\ntks}} \|v\|_{-\sm}^{\frac{\ntks-2\sm}{\ntks}}.
\]
Together with Young's inequality, this implies
\begin{align*}
  \wdiff^{\holder} \|\res\|_{\bar{S}} \|\res\|_{-\sm}
  & \le \wdiff^{\holder} \|\res\|_{-\sm-\ntks}^{\frac{2\sm}{\ntks}} \|\res\|_{-\sm}^{\frac{2\ntks-2\sm}{\ntks}}
  \\
  & \le \frac{\sm}{\ntks} \left[c \|\res\|_{-\sm-\ntks}^{\frac{2\sm}{\ntks}} \right]^{\frac{\ntks}{\sm}}
    + \frac{\ntks-\sm}{\ntks} \left[ c^{-1} \wdiff^{\holder} \|\res\|_{-\sm}^{\frac{2\ntks-2\sm}{\ntks}} \right]^{\frac{\ntks}{\ntks-\sm}}
  \\
  & = \frac{\sm}{\ntks} c^{\frac{\ntks}{\sm}} \|\res\|_{-\sm-\ntks}^2
      + c^{\frac{\ntks}{(\ntks-\sm)}} \wdiff^{\frac{\holder \ntks}{\ntks-\sm}} \|\res\|_{-\sm}^2
\end{align*}
for any generic constant $c>0$. Choosing this constant sufficiently small and plugging into the evolution equation for $\|\res\|_{-\sm}$, we obtain
\begin{align*}
  \frac{1}{2} \frac{d}{dt} \|\res\|_{-\sm}^2
  & \lesssim - c \|\res\|_{-\sm-\ntks}^2 + \wdiff^{\frac{\holder \ntks}{\ntks-\sm}} \|\res\|_{-\sm}^2,
\end{align*}
with a different generic constant $c$. Hence, together with the choice $S = \sm$, we arrive at the system of ODEs
\begin{align*}
  \frac{1}{2} \frac{d}{dt} \|\res\|_{-\sm}^2
  & \lesssim - c \|\res\|_{-\sm-\ntks}^2 + \wdiff^{\frac{\holder \ntks}{\ntks-\sm}} \|\res\|_{-\sm}^2,
  \\
  \frac{1}{2} \frac{d}{dt} \|\res\|_\sm^2
  & \lesssim - c \|\res\|_{\sm-\ntks}^2 + \wdiff^{\holder} \|\res\|_\sm \|\res\|_{-\sm}.
\end{align*}
Next, we eliminate the $\|\res\|_{-\sm-\ntks}^2$ and $\|\res\|_{\sm-\ntks}^2$ norms. Since $0 < \sm < \frac{\ntks}{2}$ implies $-\sm-\ntks < \sm-\ntks < -\sm < \sm$ the interpolation inequalities in Assumption \ref{item:perturbed-gradient-flow:norms} yield
\begin{align*}
  \|\res\|_{-\sm} & \le \|\res\|_{-\sm-\ntks}^{\frac{2\sm}{2\sm+\ntks}} \|\res\|_\sm^{\frac{\ntks}{2\sm+\ntks}} &
  & \Rightarrow &
  \|\res\|_{-\sm-\ntks} & \ge \|\res\|_{-\sm}^{\frac{2\sm+\ntks}{2\sm}} \|\res\|_\sm^{- \frac{\ntks}{2\sm}}
  \\
  \|\res\|_{-\sm} & \le \|\res\|_{\sm-\ntks}^{\frac{2\sm}{\ntks}} \|\res\|_\sm^{\frac{\ntks-2\sm}{\ntks}} &
  & \Rightarrow &
  \|\res\|_{\sm-\ntks} & \ge \|\res\|_{-\sm}^{\frac{\ntks}{2\sm}} \|\res\|_\sm^{\frac{2\sm-\ntks}{2\sm}},
\end{align*}
so that we obtain the differential inequalities
\begin{align*}
  \frac{1}{2} \frac{d}{dt} \|\res\|_{-\sm}^2
  & \lesssim - c \|\res\|_{-\sm}^{2\frac{2\sm+\ntks}{2\sm}} \|\res\|_\sm^{- 2\frac{\ntks}{2\sm}}+ \wdiff^{\frac{\ntks \holder}{\ntks-\sm}} \|\res\|_{-\sm}^2
  \\
  \frac{1}{2} \frac{d}{dt} \|\res\|_\sm^2
  & \lesssim - c \|\res\|_{-\sm}^{2\frac{\ntks}{2\sm}} \|\res\|_\sm^{2\frac{2\sm-\ntks}{2\sm}} + \wdiff^{\holder} \|\res\|_\sm \|\res\|_{-\sm}.
\end{align*}
Bounds for the solutions are provided by Lemma \ref{lemma:ode-system-bounds} with $x = \|\res\|_{-\sm}^2$, $y = \|\res\|_\sm^2$ and $\rho = \frac{\ntks}{2\sm} \ge 1 \ge \frac{1}{2}$: Given that
\begin{equation} \label{eq:perturbed-gradient-flow:time-limit}
  \|\res\|_{-\sm}^2 \gtrsim \wdiff^{2\frac{\holder \sm}{\ntks-\sm}} \|\res(0)\|_\sm^2,
\end{equation}
i.e. the error $\|\res\|_{-\sm}$ is still larger than the right hand side, which will be our final error bound, we have
\begin{align}
   \label{eq:perturbed-gradient-flow:bound}
  \|\res\|_{-\sm}^2
  & \lesssim \left[ \wdiff^{\frac{\ntks \holder}{\ntks-\sm}} \|\res(0)\|_\sm^{\frac{\ntks}{\sm}} + \|\res(0)\|_{-\sm}^{\frac{\ntks}{\sm}} e^{-c\wdiff^{\frac{\ntks \holder}{\ntks-\sm}} \frac{\ntks}{2\sm} t} \right]^{\frac{2\sm}{\ntks}}
  \\
  \|\res\|_\sm^2
  & \lesssim \|\res(0)\|_\sm^2.
\end{align}
The second condition $B(t) \ge 0$ in Lemma \ref{lemma:ode-system-bounds} is equivalent to $a x_0^{\rho} \ge b y_0^\rho$ (notation of the lemma), which in our case is identical to \eqref{eq:perturbed-gradient-flow:time-limit} at $t=0$. Notice that the right hand side of \eqref{eq:perturbed-gradient-flow:time-limit} corresponds to the first summand in the $\|\res\|_{-\sm}^2$ bound so that the second summand must dominate and we obtain the simpler expression
\begin{equation} \label{eq:error-bound-simplified}
  \begin{aligned}
    \|\res\|_{-\sm}^2
    & \lesssim \|\res(0)\|_{-\sm}^2 e^{-c\wdiff^{\frac{\ntks \holder}{\ntks-\sm}} t},
    \\
    \|\res\|_\sm^2
    & \lesssim \|\res(0)\|_\sm^2.
  \end{aligned}
\end{equation}

Finally, we compute $\wdiff$, first for the case $h = \sup_{t \le T} \wnorm{\theta(t) - \theta(0)}$. For $T$ we use the smallest time for which \eqref{eq:perturbed-gradient-flow:time-limit} fails and temporarily also $h \le 1$. Then by Assumption \eqref{eq:perturbed-gradient-flow:weight-diff}, interpolation inequality \eqref{eq:interpolation-inequality} and the $\|\res\|_{-\sm}^2$, $\|\res\|_\sm^2$ bounds, with probability at least $1-p_0(m)$, we have
\begin{align*}
  h
  = \sup_{t \le T} \wnorm{\theta(t) - \theta(0)}
  & \lesssim \sqrt{\frac{2}{m}} \int_0^T \|\res(\tau)\|_0 \, d\tau
  \\
  & \lesssim \sqrt{\frac{2}{m}} \int_0^T \|\res(\tau)\|_{-\sm}^{\frac{1}{2}} \|\res(\tau)\|_\sm^{\frac{1}{2}} \, d\tau
  \\
  & \lesssim \sqrt{\frac{2}{m}} \|\res(0)\|_{-\sm}^{\frac{1}{2}} \|\res(0)\|_\sm^{\frac{1}{2}} \int_0^T e^{- c\wdiff^{\frac{\ntks \holder}{\ntks-\sm}} \frac{\tau}{4}} \, d\tau
  \\
  & \le c \sqrt{\frac{1}{m}} \frac{\|\res(0)\|_{-\sm}^{\frac{1}{2}} \|\res(0)\|_\sm^{\frac{1}{2}}}{\wdiff^{\frac{\ntks \holder}{\ntks-\sm}}},
\end{align*}
for some generic constant $c>0$. Solving for $\wdiff$, we obtain
\begin{align*}
  & \wdiff^{1+\frac{\ntks \holder}{\ntks-\sm}} \lesssim \|\res(0)\|_{-\sm}^{\frac{1}{2}} \|\res(0)\|_\sm^{\frac{1}{2}} m^{-\frac{1}{2}} & 
  & \Leftrightarrow &
  & \wdiff \lesssim \left[ \|\res(0)\|_{-\sm}^{\frac{1}{2}} \|\res(0)\|_\sm^{\frac{1}{2}} m^{-\frac{1}{2}} \right]^{\frac{\ntks - \sm}{\ntks(1+\holder) - \sm}}.
\end{align*}
Notice that by assumption $m$ is sufficiently large so that the right hand side is strictly smaller than one and thus $T$ is only constrained by \eqref{eq:perturbed-gradient-flow:time-limit}. In case $h = c \sqrt{d/m}$ there is nothing to show and we obtain
\begin{equation*}
  \wdiff 
  \lesssim \max \left\{ \left[ \|\res(0)\|_{-\sm}^{\frac{1}{2}} \|\res(0)\|_\sm^{\frac{1}{2}} m^{-\frac{1}{2}} \right]^{\frac{\ntks - \sm}{\ntks(1+\holder) - \sm}},\, c\sqrt{\frac{d}{m}} \right\}.
\end{equation*}
Finally, we extend the result beyond the largest time $T$ for which \eqref{eq:perturbed-gradient-flow:time-limit} is satisfied and hence \eqref{eq:perturbed-gradient-flow:time-limit} holds with equality. Since $\|\res\|_0^2$ is defined by a gradient flow, it is monotonically decreasing and thus for any time $t > T$, we have
\begin{multline*}
  \|\res(t)\|_{-\sm}^2 \le \|\res(T)\|_{-\sm}^2
  = c \wdiff^{2\frac{\holder \sm}{\ntks-\sm}} \|\res(0)\|_\sm^2
  = c \left[ \wdiff^{\frac{\holder \ntks}{\ntks-\sm}} \|\res(0)\|_\sm^{\frac{\ntks}{\sm}}\right]^{\frac{2\sm}{\ntks}}
  \\
  \lesssim \left[ \wdiff^{\frac{\ntks \holder}{\ntks-\sm}} \|\res(0)\|_\sm^{\frac{\ntks}{\sm}} + \|\res(0)\|_{-\sm}^{\frac{\ntks}{\sm}} e^{-c\wdiff^{\frac{\ntks \holder}{\ntks-\sm}} \frac{\ntks}{2\sm} t} \right]^{\frac{2\sm}{\ntks}}
\end{multline*}
so that the error bound \eqref{eq:perturbed-gradient-flow:bound} holds for all times up to an adjustment of the constants. This implies the statement of the lemma with our choice of $\wdiff$ and $\tau$.

\end{proof}

\paragraph{Technical Supplements}

\begin{lemma} \label{lemma:ode-system-bounds}
  Assume $a, b, c, d > 0$, $\rho \ge \frac{1}{2}$ and that $x$, $y$ satisfy the differential inequality
  \begin{align}
    \label{eq:shallow:ode-x}
    x' & \le - a x^{1+\rho} y^{-\rho} + b x, & x(0) & = x_0 \\
    \label{eq:shallow:ode-y}
    y' & \le - c x^\rho y^{1-\rho} + d \sqrt{x y}, & y(0) & = y_0.
  \end{align}
  Then within any time interval $[0, T]$ for which
  \begin{equation} \label{eq:shallow:cond}
    x(t)
    \ge \left(\frac{d}{c}\right)^{\frac{2}{2\rho - 1}} y_0,
  \end{equation}
  with
  \begin{align*}
    A & := \frac{b}{a} y_0^\rho, &
    B(t) & := \left[1 -  \frac{b}{a} \left(\frac{x_0}{y_0}\right)^{-\rho}\right] e^{-b \rho t}
  \end{align*}
  we have
  \begin{align*}
    x(t) & \le A \left( 1 - B(t) \right)^{-1}, &
    y(t) & \le y_0.
  \end{align*}
  If $B(t) \ge 0$, this can be further estimated by
  \begin{align*}
    x(t) & \le \left(A + x_0^\rho e^{-b \rho t} \right)^{\frac{1}{\rho}}, &
    y(t) & \le y_0.
  \end{align*}

\end{lemma}

\begin{proof}

First, we show that $y(t) \le y_0$ for all $t \in T$. To this end, note that condition \eqref{eq:shallow:cond} states that we are above a critical point for the second ODE \eqref{eq:shallow:ode-y}. Indeed, setting $y'(t) = 0$ and thus $y(t) = y_0$ and solving the second ODE (with $=$ instead of $\le$) for $x(t)$, we have
\begin{equation*}
  x(t)
  = \left(\frac{d}{c}\right)^{\frac{2}{2\rho - 1}} y_0.
\end{equation*}
To show that $y(t) \ge y_0$, let $\epsilon \ge 0$ and define
\begin{align*}
  T_\epsilon & = \sup \left\{ t \le T \middle| x(t) \ge \left(\frac{d}{c}\right)^{\frac{2}{2\rho - 1}} y_0 (1+\epsilon) \right\}, &
  \\
  \tau_\epsilon & = \inf \left\{t \le T_\epsilon \middle| y(t) \ge y_0(1+\epsilon) \right\},
\end{align*}
where the definition of $T_\epsilon$ resembles the definition of $T$ up to a safety factor of $1+\epsilon$ and $\tau_\epsilon$ is the smallest time when our hypothesis $y(t) \le y_0$ fails up to a small margin.
Assume that $\tau_\epsilon < T_\epsilon$. Since $2 \rho - 1 \ge 0$, for all $t < \tau_\epsilon$, we have
\[
  x(t)^{2\rho - 1}
  \ge \left(\frac{d}{c}\right)^2 \left[ y_0 (1+\epsilon) \right]^{2 \rho - 1}
  \ge \left(\frac{d}{c}\right)^2 y(t)^{2 \rho - 1},
\]
which upon rearrangement is equivalent to
\[
    - c x^\rho y^{1-\rho} + d \sqrt{x y} \le 0,
\]
so that the differential equation \eqref{eq:shallow:ode-y} yields $y'(t) \le 0$ and hence $y(t) \le y_0$ for all $t < \tau_\epsilon$. On the other hand, for all $t > \tau_\epsilon$ we have $y(t) > y_0 (1+\epsilon)$, which contradicts the continuity of $y$. It follows that $\tau_\epsilon \ge T_\epsilon$ and with $\lim_{\epsilon \to 0} T_\epsilon = T$, we obtain
\begin{align*}
  y(t) & \le y_0, &
  t & < T.
\end{align*}
Next, we show the bounds for $x(t)$. For any fixed function $y$, the function $x$ is bounded by the solution $z$ of the equality case
\begin{align*}
  z' & = - a z^{1+\rho} y^{-\rho} + b z, &
  z(0) & = x_0
\end{align*}
of the first equation \eqref{eq:shallow:ode-x}. This is a Bernoulli differential equation, with solution
\[
  x(t)
  \le z(t)
  = \left[ e^{-b\rho t} \left( a \rho \int_0^t e^{b \rho \tau} y(\tau)^{-\rho} \, d\tau + x_0^{-\rho} \right) \right]^{-\frac{1}{\rho}}.
\]
Since $y(t) \le y_0$, in the relevant time interval this simplifies to
\begin{align*}
  z(t)^{\rho}
  & \le e^{b\rho t} \left( a \rho \int_0^t e^{b \rho \tau} y_0^{-\rho} \, d\tau + x_0^{-\rho} \right)^{-1}
  \\
  & = e^{b\rho t} \left( \frac{a}{b} \left( e^{b \rho t} - 1 \right) y_0^{-\rho} + x_0^{-\rho} \right)^{-1}
  \\
  & = \left( \frac{a}{b} y_0^{-\rho} - \left(\frac{a}{b} y_0^{-\rho} -  x_0^{-\rho}\right) e^{-b \rho t} \right)^{-1}
  \\
  & = \underbrace{\frac{b}{a} y_0^\rho}_{=:A} \left( 1 - \underbrace{\left(1 -  \frac{b}{a} \left(\frac{x_0}{y_0}\right)^{-\rho}\right) e^{-b \rho t}}_{=: B(t)} \right)^{-1},
\end{align*}
which shows the first bound for $x(t)$. We can estimate this further by
\begin{align*}
  z(t)^\rho
  \le \frac{A}{1 - B(t)}
  = \frac{A[1 - B(t)]}{1 - B(t)} + \frac{A B(t)}{1 - B(t)}
  = A + \frac{A}{1 - B(t)} B(t).
\end{align*}
In case $B(t) \ge 0$, the function $A / (1-B(t))$ is monotonically decreasing and thus with $A / (1-B(0)) = x_0^\rho$, we have
\[
  z(t)^\rho
  \le A + \frac{A}{1 - B(0)} B(t)
  = A + x_0^\rho B(t)
  \le A + x_0^\rho e^{- b \rho t},
\]
which shows the second bound for $x(t)$ in the lemma.

\end{proof}

\subsection{Proof of Lemma \ref{lemma:continuity:entk-holder}: NTK Hölder continuity}
\label{sec:ntk-holder}

The proof is technical but elementary. We start with upper bounds and Hölder continuity for simple objects, like hidden layers, and then compose these for derived objects with results for the NTK at the end of the section.

Throughout this section, we use a bar $\bar{\cdot}$ to denote a perturbation. In particular $\pW^\ell$ is a perturbed weight,  
\begin{align*}
  \pf^{\ell+1}(x) & = \pW^\ell n_\ell^{-1/2} \activationp{\pf^\ell(x)}, &
  \pf^1(x) & = \pW^0 V x
\end{align*}
is the neural network with perturbed weights and $\pegp$, $\pedgp$, $\pntk$ and $\pentk$ are the kernels of the perturbed network. The bounds in this section depend on the operator norm of the weight matrices. At initialization, they are bounded $\left\|W^\ell\right\| n_\ell^{-1/2} \lesssim 1$, with high probability. All perturbations of the weights that we need are close $\left\|W^\ell - \pW^\ell\right\| n_\ell^{-1/2} \lesssim 1$ so that we may assume
\begin{align}
   \label{eq:continuity:weights-bounded}
   \left\|W^\ell\right\| n_\ell^{-1/2} & \lesssim 1 \\
   \label{eq:continuity:perturbed-weights-bounded}
   \left\|\pW^\ell\right\| n_\ell^{-1/2} & \lesssim 1
\end{align}
In addition, we consider bounded domains
\begin{equation} \label{eq:continuity:domain-bounded}
  \begin{aligned}
    \|x\| & \lesssim 1 & & \text{for all} & x & \in \dom.
  \end{aligned}
\end{equation}

\begin{lemma} \label{lemma:continuity:f-bounded-lipschitz}
\begin{enumerate}

  Assume that $\|x\| \lesssim  1$.

  \item Assume that $\activation$ satisfies the growth condition \eqref{eq:assumption:activation-growth} and may be different in each layer. Assume the weights are bounded \eqref{eq:continuity:weights-bounded}. Then
  \begin{equation*}
    \left\| f^{\ell}(x) \right\| 
    \lesssim n_0^{1/2} \prod_{k=0}^{\ell-1} \left\|W^k\right\| n_k^{-1/2}.
  \end{equation*}

  \item Assume that $\activation$ satisfies the growth and Lipschitz conditions \eqref{eq:assumption:activation-growth} and \eqref{eq:assumption:activation-lipschitz} and may be different in each layer. Assume the weights and perturbed weights are bounded \eqref{eq:continuity:weights-bounded}, \eqref{eq:continuity:perturbed-weights-bounded}. Then
  \begin{equation*}
    \left\| f^{\ell}(x) - \pf^{\ell}(x) \right\| 
    \lesssim n_0^{1/2} \sum_{k=0}^{\ell-1} \left\|W^k - \pW^k \right\| n_k^{-1/2} \prod_{\substack{j=0 \\ j \ne k}}^{\ell-1} \max \left\{\left\|W^j\right\|, \, \left\|\pW^j\right\|\right\} n_j^{-1/2}.
  \end{equation*}

  \item Assume that $\activation$ has bounded derivative \eqref{eq:assumption:dactivation-bounded} and may be different in each layer. Assume the weights are bounded \eqref{eq:continuity:weights-bounded}. Then
  \begin{equation*}
    \left\| f^{\ell}(x) - f^\ell(\px) \right\| 
    \lesssim n_0^{1/2} \left[\prod_{k=0}^{\ell-1} \left\|W^k\right\| n_k^{-1/2}\right] \|x - \px\|.
  \end{equation*}

\end{enumerate}
  
\end{lemma}

\begin{proof}

\begin{enumerate}

  \item For $\ell=0$, we have
  \[
    \left\|f^1(x)\right\| 
    = \left\|W^0 V x\right\| 
    \le n_0^{1/2} \left\|W^0\right\| n_0^{-1/2},
  \]
  where in the last step we have used that $V$ has orthonormal columns and $\|x\| \lesssim 1$. For $\ell > 0$, we have
  \begin{multline*}
    \left\|f^{\ell+1}\right\|
    = \left\|W^{\ell} n_\ell^{-1/2} \activationp{f^\ell}\right\|
    \le \left\|W^{\ell}\right\| n_\ell^{-1/2} \left\|\activationp{f^\ell}\right\|
    \stackrel{\eqref{eq:assumption:activation-growth}}{\lesssim} \left\|W^{\ell}\right\| n_\ell^{-1/2} \left\|f^\ell\right\|
    \\
    \stackrel{\text{\tiny induction}}{\lesssim} \left\|W^{\ell}\right\| n_\ell^{-1/2} n_0^{1/2} \prod_{k=0}^{\ell-1} \left\|W^k\right\| n_k^{-1/2}
    = n_0^{1/2} \prod_{k=0}^\ell \left\|W^k\right\| n_k^{-1/2},
  \end{multline*}
  where in the first step we have used the definition of $f^{\ell+1}$, in the third the growth condition and in the fourth the induction hypothesis.

  \item For $\ell=0$ we have
  \[
    \left\|f^1-\pf^1\right\|
    = \left\| [W^0 - \pW^0] V x \right\|
    = n_0^{1/2} \left\| W^0 - \pW^0 \right\| n_0^{-1/2},
  \]
  where in the last step we have used that $V$ has orthonormal columns and $\|x\| \lesssim 1$. For $\ell > 0$, we have
  \begin{align*}
    \left\|f^{\ell+1}-\pf^{\ell+1}\right\|
    & = \left\|W^{\ell} n_\ell^{-1/2} \activationp{f^\ell} - \pW^\ell n_\ell^{-1/2} \activationp{\pf^\ell}\right\| 
    \\
    & \le \left\|W^{\ell} - \pW^\ell\right\| n_\ell^{-1/2} \left\|\activationp{f^\ell}\right\| \\
    & \quad + \left\|\pW^{\ell}\right\| n_\ell^{-1/2} \left\|\activationp{f^\ell} - \activationp{\pf^\ell}\right\| \\
    & =: I + II
  \end{align*}
  For the first term, the growth condition \eqref{eq:assumption:activation-growth} implies $\left\|\activationp{f^\ell}\right\| \lesssim \left\|f^\ell\right\|$ and thus the first part of the Lemma yields
  \[
    I
    \lesssim \left\|W^{\ell} - \pW^\ell\right\| n_\ell^{-1/2} n_0^{1/2} \prod_{k=0}^{\ell-1} \left\|W^k\right\| n_k^{-1/2}.
  \]
  For the second term, we have by Lipschitz continuity \eqref{eq:assumption:activation-lipschitz} and induction
  \begin{multline*}
    II
    = \left\|\pW^{\ell}\right\| n_\ell^{-1/2} \left\|\activationp{f^\ell} - \activationp{\pf^\ell}\right\|
    \lesssim \left\|\pW^{\ell}\right\| n_\ell^{-1/2} \left\|f^\ell - \pf^\ell\right\|
    \\
    \lesssim n_0^{1/2} \sum_{k=0}^{\ell-1} \left\|W^k - \pW^k \right\| n_k^{-1/2} \prod_{\substack{j=0 \\ j \ne k}}^\ell \max \left\{\left\|W^j\right\|, \, \left\|W^j\right\|\right\} n_j^{-1/2}.
  \end{multline*}
  By $I$ and $II$ we obtain
  \begin{equation*}
    \left\| f^{\ell+1} - \pf^{\ell+1}\right\| 
    \lesssim n_0^{1/2} \sum_{k=0}^\ell \left\|W^k - \pW^k \right\| n_k^{-1/2} \prod_{\substack{j=0 \\ j \ne k}}^\ell \max \left\{\left\|W^j\right\|, \, \left\|W^j\right\|\right\} n_j^{-1/2},
  \end{equation*}
  which shows the lemma.

  \item Follows from the mean value theorem because by Lemma \ref{lemma:continuity:df-bounded} below the first derivatives are uniformly bounded.

\end{enumerate}
  
\end{proof}

\begin{lemma} \label{lemma:continuity:df-bounded}

  Assume that $\activation$ has bounded derivative \eqref{eq:assumption:dactivation-bounded} and may be different in each layer. Assume the weights are bounded \eqref{eq:continuity:weights-bounded}. Then
  \begin{equation*}
    \left\| D f^{\ell}(x) \right\| 
    \lesssim n_0^{1/2} \prod_{k=0}^{\ell-1} \left\|W^k\right\| n_k^{-1/2}.
  \end{equation*}

\end{lemma}

\begin{proof}

For $\ell=0$, we have
\[
  \left\|D f^1(x)\right\| 
  = \left\|W^0 V D x\right\| 
  \le n_0^{1/2} \left\|W^0\right\| n_0^{-1/2},
\]
where in the last step we have used that $V$ has orthonormal columns and $\|D x\| = \|I\| = 1$. For $\ell > 0$, we have
\begin{align*}
  \left\|D f^{\ell+1}\right\|
  & = \left\|W^{\ell} n_\ell^{-1/2} D \activationp{f^\ell}\right\|
  \\
  & = \left\|W^{\ell} n_\ell^{-1/2}\right\| \left\|D \activationp{f^\ell}\right\|
  \le \left\|W^{\ell}\right\| n_\ell^{-1/2} \left\|\dactivationp{f^\ell} \odot D f^\ell \right\|
  \\
  & \stackrel{\eqref{eq:assumption:dactivation-bounded}}{\lesssim} \left\|W^{\ell}\right\| n_\ell^{-1/2} \left\|D f^\ell\right\|
  \stackrel{\text{\tiny induction}}{\lesssim} \left\|W^{\ell}\right\| n_\ell^{-1/2} n_0^{1/2} \prod_{k=0}^{\ell-1} \left\|W^k\right\| n_k^{-1/2}
  \\
  & = n_0^{1/2} \prod_{k=0}^\ell \left\|W^k\right\| n_k^{-1/2},
\end{align*}
where in the first step we have used the definition of $f^{\ell+1}$, in the fourth the boundedness of $\dactivation$ and in the fifth the induction hypothesis.

\end{proof}

\begin{remark}

  An argument analogous to Lemma \ref{lemma:continuity:df-bounded} does not show that the derivative is Lipschitz or similarly second derivatives $\left\|\partial_{x_i} \partial{x_j} f^\ell \right\|$ are bounded. Indeed, the argument uses that 
  \[
    \left\| \partial_{x_i} \activationp{f^\ell}\right\| 
    = \left\|\dactivationp{f^\ell} \odot \partial_{x_i} f^\ell \right\|
    \le \left\|\dactivationp{f^\ell} \right\|_\infty \left\|\partial_{x_i} f^\ell \right\|,
  \]
  where we bound the first factor by the upper bound of $\dactivation$ and the second by induction. However, higher derivatives produce products
  \begin{align*}
    \left\| \partial_{x_i}  \partial_{x_j}\activationp{f^\ell}\right\| 
    & = \left\|\dactivationp{f^\ell} \odot \partial_{x_i} \partial_{x_i} f^\ell + \activation^{(2)}\left(f^\ell\right) \odot \partial_{x_i} f^\ell \odot \partial_{x_j} f^\ell \right\|
    \\
    & \le \left\|\dactivationp{f^\ell} \right\|_\infty \left\|\partial_{x_i} \partial_{x_j}  f^\ell \right\| + \left\|\activation^{(2)}\left(f^\ell\right)\right\|_\infty \left\|\partial_{x_i} f^\ell \odot \partial_{x_j} f^\ell \right\|
  \end{align*}
  With bounded weights \eqref{eq:continuity:weights-bounded} the hidden layers are of size $\left\|\partial_{x_i} f^\ell \right\| \lesssim n_0^{1/2}$ but a naive estimate of their product by Cauchy Schwarz and embedding $\left\|\partial_{x_i} f^\ell \odot \partial_{x_j} f^\ell \right\| \le \|\partial_{x_i} f^\ell\|_{\ell_4} \|\partial_{x_i} f^\ell\|_{\ell_4} \le \|\partial_{x_i} f^\ell\| \|\partial_{x_i} f^\ell\| \lesssim n_0$ is much larger. 

\end{remark}

Given the difficulties in the last remark, we can still show that $f^\ell$ is Hölder continuous with respect to the weights in a Hölder norm with respect to $x$.

\begin{lemma} \label{lemma:continuity:nn-holder}
  Assume that $\activation$ satisfies the growth and Lipschitz conditions \eqref{eq:assumption:activation-growth}, \eqref{eq:assumption:activation-lipschitz} and may be different in each layer. Assume the weights, perturbed weights and domain are bounded \eqref{eq:continuity:weights-bounded}, \eqref{eq:continuity:perturbed-weights-bounded}, \eqref{eq:continuity:domain-bounded}. Then for $0 < \alpha < 1$
  \begin{align*}
    \left\|\activationp{f^\ell}\right\|_\CHN{\alpha}
    & \lesssim n_0^{1/2}.
    \\
    \left\|\activationp{\pf^\ell}\right\|_\CHN{\alpha}
    & \lesssim n_0^{1/2}.
    \\
    \left\|\activationp{f^\ell} - \activationp{\pf^\ell}\right\|_\CHN{\alpha}
    & \lesssim n_0^{1/2} \left[ \sum_{k=0}^{\ell-1} \left\|W^k - \pW^k \right\| n_k^{-1/2} \right]^{1-\alpha}.
  \end{align*}
\end{lemma}

\begin{proof}

By the growth condition \eqref{eq:assumption:activation-growth} and the Lipschitz continuity \eqref{eq:assumption:activation-lipschitz} of the activation function, we have
\begin{align*}
  \left\|\activationp{f^\ell}\right\|_\CN{0} & \lesssim \left\|f^\ell\right\|_\CN{0} , &
  \left\|\activationp{f^\ell}\right\|_\CHN{1} & \lesssim \left\|f^\ell\right\|_\CHN{1}.
\end{align*}
Thus the interpolation inequality in Lemma \ref{lemma:supplements:holder-properties} implies
\begin{align*}
  \left\|\activationp{f^\ell}\right\|_\CHN{\alpha}
  & \lesssim \left\|\activationp{f^\ell}\right\|_\CN{0}^{1-\alpha} \left\|\activationp{f^\ell}\right\|_\CHN{1}^\alpha
  \lesssim \left\|f^\ell\right\|_\CN{0}^{1-\alpha} \left\|f^\ell\right\|_\CHN{1}^\alpha
  \lesssim n_0^{1/2},
\end{align*}
where in the last step we have used the bounds form Lemma \ref{lemma:continuity:f-bounded-lipschitz} together with $\left\|W^\ell\right\| n_\ell^{-1/2} \lesssim 1$ and $\left\|\pW^\ell\right\| n_\ell^{-1/2} \lesssim 1$ from Assumptions \eqref{eq:continuity:weights-bounded}, \eqref{eq:continuity:perturbed-weights-bounded}. Likewise, by the interpolation inequality in Lemma \ref{lemma:supplements:holder-properties} we have
\begin{align*}
  \left\|\activationp{f^\ell} - \activationp{\pf^\ell}\right\|_\CHN{\alpha}
  & \lesssim \left\|\activationp{f^\ell} - \activationp{\pf^\ell}\right\|_\CN{0}^{1-\alpha} \left\|\activationp{f^\ell} - \activationp{\pf^\ell}\right\|_\CHN{1}^\alpha
  \\
  & \lesssim \left\|\activationp{f^\ell} - \activationp{\pf^\ell}\right\|_\CN{0}^{1-\alpha} \max \left\{ \left\|\activationp{f^\ell}\right\|_\CHN{1}^\alpha \left\|\activationp{\pf^\ell}\right\|_\CHN{1}^\alpha \right\}.
  \\
  & \lesssim \left\|f^\ell - \pf^\ell\right\|_\CN{0}^{1-\alpha} \max \left\{ \left\|f^\ell\right\|_\CHN{1}^\alpha \left\|\pf^\ell\right\|_\CHN{1}^\alpha \right\}.
  \\
  & \lesssim n_0^{1/2} \left[ \sum_{k=0}^{\ell-1} \left\|W^k - \pW^k \right\| n_k^{-1/2} \right]^{1-\alpha},
\end{align*}
where in the third step we have used that $\activation$ is Lipschitz and in the last step the bounds from Lemma \ref{lemma:continuity:f-bounded-lipschitz} together with the bounds $\left\|W^\ell\right\| n_\ell^{-1/2} \lesssim 1$ and $\left\|\pW^\ell\right\| n_\ell^{-1/2} \lesssim 1$ from Assumptions \eqref{eq:continuity:weights-bounded}, \eqref{eq:continuity:perturbed-weights-bounded}.

\end{proof}

\begin{lemma} \label{lemma:continuity:egp-holder}
  Assume that $\activation$ satisfies the growth and Lipschitz conditions \eqref{eq:assumption:activation-growth}, \eqref{eq:assumption:activation-lipschitz} and may be different in each layer. Assume the weights, perturbed weights and domain are bounded \eqref{eq:continuity:weights-bounded}, \eqref{eq:continuity:perturbed-weights-bounded}, \eqref{eq:continuity:domain-bounded}. Then for $0 < \alpha, \beta < 1$
  \begin{align*}
    \left\|\egp^\ell\right\|_\CHHN{\alpha}{\beta} & \lesssim \frac{n_0}{n_\ell},
    \\
    \left\|\pegp^\ell\right\|_\CHHN{\alpha}{\beta} & \lesssim \frac{n_0}{n_\ell},
    \\
    \left\|\egp^\ell - \pegp^\ell\right\|_\CHHN{\alpha}{\alpha}
    & \lesssim \frac{n_0}{n_\ell} \left[ \sum_{k=0}^{\ell-1} \left\|W^k - \pW^k \right\| n_k^{-1/2} \right]^{1-\alpha}.
  \end{align*}
\end{lemma}

\begin{proof}

\newcommand{\fx}{f}
\newcommand{\pfx}{\pf}
\newcommand{\fy}{\tilde{f}}
\newcommand{\pfy}{\tilde{\pf}}
Throughout the proof, we abbreviate
\begin{align*}
  \fx^\ell & = f^\ell(x), &
  \pfx^\ell & = \pf^\ell(x), &
  \fy^\ell & = f^\ell(y), &
  \pfy^\ell & = \pf^\ell(x), &
\end{align*}
for two independent variables $x$ and $y$. Then by definition \eqref{eq:def-gp} of $\egp^\ell$
\begin{align*}
  \left\|\egp^\ell\right\|_\CHHN{\alpha}{\beta}
  & = \frac{1}{n_\ell} \left\|\activationp{\fx^\ell}^T \activationp{\fy^\ell}\right\|_\CHHN{\alpha}{\beta}
  & \le \frac{1}{n_\ell} \left\|\activationp{\fx^\ell}\right\|_\CHN{\alpha} \left\|\activationp{\fy^\ell}\right\|_\CHN{\beta}
  & \lesssim \frac{n_0}{n_\ell},
\end{align*}
where in the second step we have used the product identity Item \ref{item:dot-product:lemma:supplements:holder-properties} in Lemma \ref{lemma:supplements:holder-properties} and in the last step Lemma \ref{lemma:continuity:nn-holder}. The bound for $\left\|\pegp^\ell\right\|_\CHHN{\alpha}{\beta}$ follows analogously. Likewise for $\alpha = \beta$
\begin{align*}
  \left\|\egp^\ell - \pegp^\ell\right\|_\CHHN{\alpha}{\alpha}
  & = \frac{1}{n_\ell} \left\|\activationp{\fx^\ell}^T \activationp{\fy^\ell} - \activationp{\pfx^\ell}^T \activationp{\pfy^\ell}\right\|_\CHHN{\alpha}{\alpha}
  \\
  & = \frac{1}{n_\ell} \left\|\left[ \activationp{\fx^\ell} - \activationp{\pfx^\ell} \right]^T \activationp{\fy^\ell} - \activationp{\pfx^\ell}^T \left[ \activationp{\fy^\ell} - \activationp{\pfy^\ell} \right]\right\|_\CHHN{\alpha}{\alpha}
  \\
  & \le \frac{1}{n_\ell} \left\|\left[ \activationp{\fx^\ell} - \activationp{\pfx^\ell} \right]^T \activationp{\fy^\ell}\right\|_\CHHN{\alpha}{\alpha} + \left\| \activationp{\pfx^\ell}^T \left[ \activationp{\fy^\ell} - \activationp{\pfy^\ell} \right]\right\|_\CHHN{\alpha}{\alpha}
  \\
  & = \frac{2}{n_\ell} \left\|\left[ \activationp{\fx^\ell} - \activationp{\pfx^\ell} \right]^T \activationp{\fy^\ell}\right\|_\CHHN{\alpha}{\alpha},
\end{align*}
where in the last step we have used symmetry in $x$ and $y$. Thus, by the product identity Item \ref{item:dot-product:lemma:supplements:holder-properties} in Lemma \ref{lemma:supplements:holder-properties}, we obtain
\begin{align*}
  \left\|\egp^\ell - \pegp^\ell\right\|_\CHHN{\alpha}{\alpha}
  & \le \frac{2}{n_\ell} \left\|\activationp{\fx^\ell} - \activationp{\pfx^\ell}\right\|_\CHN{\alpha} \left\|\activationp{\fy^\ell}\right\|_\CHN{\alpha}
  \\
  & \lesssim \frac{n_0}{n_\ell} \left[ \sum_{k=0}^{\ell-1} \left\|W^k - \pW^k \right\| n_k^{-1/2} \right]^{1-\alpha},
\end{align*}
where in the last step we have used Lemma \ref{lemma:continuity:nn-holder}.

\end{proof}

\begin{lemma}[Lemma \ref{lemma:continuity:entk-holder} restated form overview]
  \LemmaEntkHolder
\end{lemma}

\begin{proof}

By Lemma \ref{lemma:continuity:egp-holder} and $n_\ell \sim n_0$, we have
\begin{align*}
  \left\|\egp^\ell\right\|_\CHHN{\alpha}{\alpha}, \left\|\pegp^\ell\right\|_\CHHN{\alpha}{\alpha} & \lesssim 1, &
  \left\|\egp^\ell - \pegp^\ell\right\|_\CHHN{\alpha}{\alpha}
  & \lesssim \frac{n_0}{n_\ell} \left[ \sum_{k=0}^{\ell-1} \left\|W^k - \pW^k \right\| n_k^{-1/2} \right]^{1-\alpha}.
\end{align*}
Since $\dactivation$ satisfies the same assumptions as $\activation$, the same lemma provides
\begin{align*}
  \left\|\edgp^\ell\right\|_\CHHN{\alpha}{\alpha}, \left\|\pedgp^\ell\right\|_\CHHN{\alpha}{\alpha} & \lesssim 1, &
  \left\|\edgp^\ell - \pedgp^\ell\right\|_\CHHN{\alpha}{\alpha}
  & \lesssim \frac{n_0}{n_\ell} \left[ \sum_{k=0}^{\ell-1} \left\|W^k - \pW^k \right\| n_k^{-1/2} \right]^{1-\alpha}.
\end{align*}
Furthermore, by Lemma \ref{lemma:ntk:entk-as-gp}, we have
\[
  \entk^(x,y) = \edgp^L(x,y) \egp^{L-1}(x,y).
\]
Thus, since Hölder spaces are closed under products, Lemma \ref{lemma:supplements:holder-properties} Item \ref{item:product:lemma:supplements:holder-properties}, it follows that
\begin{align*}
  \left\|\entk - \pentk\right\|_\CHHN{\alpha}{\alpha}
  & = \left\|\edgp^L(x,y) \egp^{L-1}(x,y) - \pedgp^L(x,y) \pegp^{L-1}(x,y)\right\|_\CHHN{\alpha}{\alpha}
  \\
  & \le \left\| \left[\edgp^L(x,y) - \pedgp^L(x,y) \right] \egp^{L-1}(x,y)\right\|_\CHHN{\alpha}{\alpha}
  \\
  & \quad + \left\|\pedgp^L(x,y) \left[\egp^{L-1}(x,y) - \pegp^{L-1}(x,y)\right]\right\|_\CHHN{\alpha}{\alpha}
  \\
  & \le \left\|\edgp^L(x,y) - \pedgp^L(x,y)\right\|_\CHHN{\alpha}{\alpha} \left\|\egp^{L-1}(x,y)\right\|_\CHHN{\alpha}{\alpha}
  \\
  & \quad + \left\|\pedgp^L(x,y)\right\|_\CHHN{\alpha}{\alpha} \left\|\egp^{L-1}(x,y) - \pegp^{L-1}(x,y)\right\|_\CHHN{\alpha}{\alpha}
  \\
  & \lesssim \frac{n_0}{n_\ell} \left[ \sum_{k=0}^{\ell-1} \left\|W^k - \pW^k \right\| n_k^{-1/2} \right]^{1-\alpha},
\end{align*}
where in the last step we have used Lemma \ref{lemma:continuity:egp-holder} and $n_L \sim n_{L-1}$.
  
\end{proof}

\subsection{Proof of Lemma \ref{lemma:concentration:concentration-ntk}: Concentration}
\label{sec:concentration}

Concentration for the NTK
\begin{equation*}
  \ntk(x,y) := \dgp^L(x,y) \gp^{L-1}(x,y)
\end{equation*}
is derived from concentration for the forward kernels $\dgp^L$ and $\gp^{L-1}$. They are shown inductively by splitting off the expectation $\EE{\ell}{\cdot}$ with respect to the last layer $W^\ell$ in 
\begin{equation*}
  \left\| \egp^{\ell+1} - \gp^{\ell+1} \right\|_\CHHN{\alpha}{\beta}
  \le \left\| \egp^{\ell+1} - \EE{\ell}{\egp^{\ell+1}} \right\|_\CHHN{\alpha}{\beta}
    + \left\| \EE{\ell}{\egp^{\ell+1}} - \gp^{\ell+1} \right\|_\CHHN{\alpha}{\beta}.
\end{equation*}
Concentration for the first term is shown in Section \ref{sec:concentration:last-layer} by a chaining argument and bounds for the second term in Section \ref{sec:concentration:covariance} with an argument similar to \cite{DuLeeLiEtAl2019}. The results are combined into concentration for the NTK in Section \ref{sec:concentration:ntk}.

\subsubsection{Concentration of the Last Layer}
\label{sec:concentration:last-layer}

We define
\begin{equation*}
  \egpr_r^\ell(x,y) 
  := \activationp{f_r^\ell(x)} \activationp{f_r^\ell(y)}
\end{equation*}
as the random variables that constitute the kernel
\begin{equation*}
  \egp^\ell(x,y) 
  = \frac{1}{n_\ell}\sum_{r=1}^{n_\ell} \egpr_r^\ell(x,y)
  = \frac{1}{n_\ell}\sum_{r=1}^{n_\ell} \activationp{f_r^\ell(x)} \activationp{f_r^\ell(y)}.
\end{equation*}
For fixed weights $W^0, \dots, W^{\ell-2}$ and random $W^{\ell-1}$, all $\egpr_r^\ell$, $r \in [n_\ell]$ are random variables dependent only on the random vector $W_{r\cdot}^{\ell-1}$ and thus independent. Hence, we can show concentration uniform in $x$ and $y$ by chaining. For Dudley's inequality, one would bound the increments
\begin{equation*}
  \left\|\egpr_r^\ell(x,y) - \egpr_r^\ell(\px,\py)\right\|_{\psi_2} \lesssim \|x-\px\|^\alpha + \|y-\py\|^\alpha,
\end{equation*}
where the right hand side is a metric for $\alpha \le 1$. However, this is not sufficient in our case. First, due to the product in the definition of $\egpr_r^\ell$, we can only bound the $\psi_1$ norm and second this leads to a concentration of the supremum norm $\|\egpr_r^\ell\|_\CN{0}$, whereas we need a Hölder norm. Therefore, we bound the finite difference operators
\begin{multline*}
  \left\| \Delta_{x,h_x}^\alpha \Delta_{y,h_y}^\beta \egpr_r^\ell(x,y)
  - \Delta_{x,\ph_x}^\alpha \Delta_{y,\ph_y}^\beta \egpr_r^\ell(\px,\py) \right\|_{\psi_1}
  \\
  \lesssim \|x-\px\|^\alpha + \|h_x-\ph_x\|^\alpha + \|y-\py\|^\beta + \|h_y-\ph_y\|^\beta,
\end{multline*}
which can be conveniently expressed by the Orlicz space valued Hölder norm
\begin{equation*}
  \left\|\Delta_x^\alpha \Delta_y^\beta \egpr_r^\ell \right\|_\CHHND{\alpha}{\beta}{\Delta\dom \times \Delta\dom;\psi_1}
  \lesssim 1,
\end{equation*}
with the following notations:
\begin{enumerate}
  \item Finite difference operators $\Delta^\alpha \colon (x,h) \to h^{-\alpha} [ f(x+h) - f(x) ]$, depending both on $x$ and $h$, with partial application two variables $x$ and $y$ denoted by $\Delta_x^\alpha$ and $\Delta_y^\alpha$, respectively. See Section \ref{sec:supplements:holder-spaces}.
  \item Domain $\Delta\dom$ consisting of all pairs $(x,h)$ for which $x, x+h \in \dom$, see \eqref{eq:supplements:delta-dom}. Likewise the domain $\Delta\dom \times \Delta\dom$ consists of all feasible $x$, $h_x$, $y$ and $h_y$.
  \item Following the definitions in Section \ref{sec:supplements:holder-spaces}, we use the Hölder space $\CHHND{\alpha}{\beta}{\Delta\dom \times \Delta\dom;L_{\psi_i}}$, $i=1,2$ with values in the Orlicz spaces $L_{\psi_i}$ of random variables for which the $\|\cdot\|_{\psi_i}$ norms are finite. For convenience, we abbreviate this by $\CHHND{\alpha}{\beta}{\Delta\dom \times \Delta\dom;\psi_i}$.
\end{enumerate}
Given the above inequalities, we derive concentration by chaining for for mixed tail random variables in \cite{Dirksen2015} summarized in Corollary \ref{corollary:supplements:chaining}.

\begin{lemma} \label{lemma:concentration:continuity:bounded-lipschitz}
Assume for $k=0, \dots, \ell-2$ the weights $W_k$ are fixed and bounded $\|W^k\| n_k^{-1/2} \lesssim 1$. Assume that $W^{\ell-1}$ is i.i.d. sub-gaussian with $\|W_{ij}^{\ell-1}\|_{\psi_2} \lesssim 1$. Let $r \in [n_\ell]$.
\begin{enumerate}

  \item Assume that $\activation$ satisfies the growth condition \eqref{eq:assumption:activation-growth} and may be different in each layer. Then
  \begin{equation*}
    \left\|\activationp{f_r^{\ell}(x)}\right\|_{\psi_2}
    \lesssim \left(\frac{n_0}{n_{\ell-1}}\right)^{1/2}.
  \end{equation*}

  \item Assume that $\activation$ has bounded derivative \eqref{eq:assumption:dactivation-bounded} and may be different in each layer. Then
  \begin{equation*}
    \left\| \activationp{f_r^{\ell}(x)} - \activationp{f_r^{\ell}(\px)} \right\|_{\psi_2}
    \lesssim \left(\frac{n_0}{n_{\ell-1}}\right)^{1/2} \|x-\px\|.
  \end{equation*}

\end{enumerate}
  
\end{lemma}

\begin{proof}
\begin{enumerate}
  \item Since for frozen $W^0, \dots, W^{\ell-2}$
  \begin{equation*}
    W_{r\cdot}^{\ell-1} n_{\ell-1}^{-1/2} \activationp{f^{\ell-1}}
    = \sum_{s=1}^{n_{\ell-1}} W_{rs}^{\ell-1} n_{\ell-1}^{-1/2} \activationp{f_s^{\ell-1}}
  \end{equation*}
  is a sum of independent random variables $W_{rs}^{\ell-1} n_{\ell-1}^{-1/2} \activationp{f_s^{\ell-1}}$, $s \in [n_{\ell-1}]$, by Hoeffding's inequality (general version for sub-gaussian norms, see e.g. \cite[Proposition 2.6.1]{Vershynin2018}) we have
  \begin{equation*}
    \left\|W_{r\cdot}^{\ell-1} n_{\ell-1}^{-1/2} \activationp{f^{\ell-1}}\right\|_{\psi_2}
    \\
    \lesssim n_{\ell-1}^{-1/2} \left\|\activationp{f^{\ell-1}}\right\|.
  \end{equation*}
  Thus
  \begin{multline*}
    \left\|\activationp{f_r^\ell}\right\|_{\psi_2}
    \lesssim \left\|f_r^\ell\right\|_{\psi_2}
    = \left\|W_{r\cdot}^{\ell-1} n_{\ell-1}^{-1/2} \activationp{f^{\ell-1}}\right\|_{\psi_2}
    \\
    \le n_{\ell-1}^{-1/2} \left\|\activationp{f^{\ell-1}}\right\|
    \le n_{\ell-1}^{-1/2} \left\|f^{\ell-1}\right\|
    \lesssim \left(\frac{n_0}{n_{\ell-1}}\right)^{1/2},
  \end{multline*}
  where in the first step we have used the growth condition and Lemma \ref{lemma:supplements:psi-norm-lipschitz}, in the fourth step the growth condition and in the last step the upper bounds from Lemma \ref{lemma:continuity:f-bounded-lipschitz}.

  \item Using Hoeffding's inequality analogous to the previous item, we have
  \begin{multline*}
    \left\|W_{r\cdot}^{\ell-1} n_{\ell-1}^{-1/2} \left[\activationp{f^{\ell-1}(x)} - \activationp{f^{\ell-1}(\px)}\right]\right\|_{\psi_2}
    \\
    \lesssim n_{\ell-1}^{-1/2} \left\|\activationp{f^{\ell-1}(x)} - \activationp{f^{\ell-1}(\px)}\right\|
  \end{multline*}
  and
  \begin{align*}
    \left\|\activationp{f_r^\ell(x)} - \activationp{f_r^\ell(\px)}\right\|_{\psi_2}
    & \lesssim \left\|f_r^\ell(x) - f_r^\ell(\px)\right\|_{\psi_2}
    \\
    & = \left\|W_{r\cdot}^{\ell-1} n_{\ell-1}^{-1/2} \left[\activationp{f^{\ell-1}(x)} - \activationp{f^{\ell-1}(\px)}\right]\right\|_{\psi_2}
    \\
    & \lesssim n_{\ell-1}^{-1/2} \left\|\activationp{f^{\ell-1}(x)} - \activationp{f^{\ell-1}(\px)}\right\|
    \\
    & \lesssim n_{\ell-1}^{-1/2} \left\|f^{\ell-1}(x) - f^{\ell-1}(\px)\right\|
    \\
    & \lesssim \left(\frac{n_0}{n_{\ell-1}}\right)^{1/2} \|x-\px\|,
  \end{align*}
  where in the first step we have used the Lipschitz condition and Lemma \ref{lemma:supplements:psi-norm-lipschitz}, in the fourth step the Lipschitz condition and in the last step the Lipschitz bounds from Lemma \ref{lemma:continuity:f-bounded-lipschitz}.
\end{enumerate}
\end{proof}

\begin{lemma} \label{lemma:concentration:fd-to-holder}
  Let $U$ and $V$ be two normed spaces and $\dom \subset U$. For all $0 \le \alpha \le \frac{1}{2}$, we have
  \begin{equation*}
    \left\|\Delta^\alpha f\right\|_\CHND{\alpha}{\Delta\dom;V}
    \le 4 \left\|f\right\|_\CHND{2\alpha}{\dom;V},
  \end{equation*}
  with $\Delta\dom$ defined in \eqref{eq:supplements:delta-dom}.
\end{lemma}

\begin{proof}

Throughout the proof, let $\CHN{2\alpha} = \CHND{2\alpha}{\dom;V}$ and $|\cdot| = \|\cdot\|_U$ or $|\cdot| = \|\cdot\|_V$ depending on context. Unraveling the definitions, for every $(x,h), (\px, \ph) \in \Delta\dom$, we have to show
\begin{equation*}
  \left| \Delta_h^\alpha f(x) - \Delta_{\ph}^\alpha f(\px) \right| 
  \le 4 \|f\|_\CHN{2\alpha} \max\{|x-\px|,|h-\ph|\}^\alpha.
\end{equation*}
We consider two cases. First, assume that $|h| \le \max \{|x-\px|, |h-\ph|\}$ and $\ph$ is arbitrary. Then $|\ph| \le |\ph - h| + |h| \le 2\max \{|x-\px|, |h-\ph|\}$ and thus
\begin{multline*}
  \left| \Delta_h^\alpha f(x) - \Delta_{\ph}^\alpha f(\px) \right|
  \le \left|\Delta_h^\alpha f(x)\right| + \left|\Delta_{\ph}^\alpha f(\px)\right|
  \\
  \le \|f\|_\CHN{2\alpha} |h|^\alpha + \|f\|_\CHN{2\alpha} |\ph|^\alpha
  \le 3 \|f\|_\CHN{2\alpha} \max \{|x-\px|, |h-\ph|\}^\alpha.
\end{multline*}
In the second case, assume that $\max\{|x-\px|, |h-\ph|\} \le |h|$ and without loss of generality that $|h| \le |\ph|$. Then
\begin{align*}
  \left| \Delta_h^\alpha f(x) - \Delta_{\ph}^\alpha f(\px) \right|
  & \le \left| [f(x+h) - f(x)]|h|^{-\alpha} - [f(\px+\ph) - f(\px)]|\ph|^{-\alpha} \right|
  \\
  & \le \left| f(x+h) - f(x) - f(\px+\ph) + f(\px) \right| |h|^{-\alpha}
  \\
  & \quad + |f(\px+\ph) - f(\px)| \left| |h|^{-\alpha} - |\ph|^{-\alpha} \right|
  \\
  & =: I + II.
\end{align*}
For the first term, we have
\begin{align*}
  I
  & \le \left| f(x+h) - f(x) - f(\px+\ph) + f(\px) \right| |h|^{-\alpha}
  \\
  & \le \|f\|_\CHN{2\alpha} \left[ |x+h -\px-\ph|^{2\alpha} + |x-\px|^{2\alpha} \right] |h|^{-\alpha}
  \\
  & \le 3 \|f\|_\CHN{2\alpha} \max \left\{|x-\px|^{2\alpha}, |h-\ph|^{2\alpha} \right\} |h|^{-\alpha}
  \\
  & \le 3 \|f\|_\CHN{2\alpha} \max \left\{|x-\px|, |h-\ph| \right\}^\alpha.
\end{align*}
For the second term, since $\alpha \le 1$, we have
\begin{align*}
  II
  & \le \|f\|_\CHN{2\alpha} |\ph|^{2\alpha} \left| |h|^{-\alpha} - |\ph|^{-\alpha} \right|
  \\
  & \le \|f\|_\CHN{2\alpha} |h|^\alpha |\ph|^\alpha \left| |h|^{-\alpha} - |\ph|^{-\alpha} \right|
  \\
  & \le \|f\|_\CHN{2\alpha} \left| |\ph|^\alpha - |h|^\alpha \right|
  \\
  & \le \|f\|_\CHN{2\alpha} |\ph - h|^\alpha.
\end{align*}
Combining all inequalities shows the result.
  
\end{proof}

\begin{lemma} \label{lemma:concentration:continuity:holder}

  Assume for $k=0, \dots, \ell-2$ the weights $W_k$ are fixed and bounded $\|W^k\| n_k^{-1/2} \lesssim 1$. Assume that $W^{\ell-1}$ is i.i.d. sub-gaussian with $\|W_{ij}^{\ell-1}\|_{\psi_2} \lesssim 1$. Assume that $\activation$ satisfies the growth condition \eqref{eq:assumption:activation-growth}, has bounded derivative \eqref{eq:assumption:dactivation-bounded} and may be different in each layer. Let $r \in [n_\ell]$. Then for $\alpha, \beta \le 1/2$
  \begin{equation*}
    \left\|\Delta_x^\alpha \Delta_y^\beta \egpr_r^\ell \right\|_\CHHND{\alpha}{\beta}{\Delta\dom \times \Delta\dom;\psi_1}
    \lesssim \frac{n_0}{n_{\ell-1}},
  \end{equation*}
  with $\Delta\dom$ defined in \eqref{eq:supplements:delta-dom}.

\end{lemma}

\begin{proof}

\newcommand{\fx}{f}
\newcommand{\fy}{\tilde{f}}

Throughout the proof, we abbreviate
\begin{align*}
  \fx^\ell & = f^\ell(x), &
  \CHND{\alpha}{\psi_i} & = \CHND{\alpha}{\Delta\dom, \psi_i}, &
  i & = 1,2,
  \\
  \fy^\ell & = f^\ell(y), &
  \CHHND{\alpha}{\beta}{\psi_i} & = \CHHND{\alpha}{\beta}{\Delta\dom \times \Delta\dom, \psi_i}. &
\end{align*}
Since by Lemma \ref{lemma:supplements:psi-norm-product} we have $\|XY\|_{\psi_1} \le \|X\|_{\psi_2} \|Y\|_{\psi_2}$ by the product inequality Lemma \ref{lemma:supplements:holder-properties} Item \ref{item:dot-product:lemma:supplements:holder-properties} for Hölder norms we obtain
\begin{align*}
  \left\|\Delta_x^\alpha \Delta_y^\beta \egpr_r^\ell \right\|_\CHHND{\alpha}{\beta}{\psi_1}
  & = \left\|\Delta_x^\alpha \activationp{\fx_r^\ell} \Delta_y^\beta \activationp{\fy_r^\ell} \right\|_\CHHND{\alpha}{\beta}{\psi_1}
  \\
  & \lesssim \left\|\Delta_x^\alpha \activationp{\fx_r^\ell} \right\|_\CHND{\alpha}{\psi_2} 
      \left\| \Delta_y^\beta \activationp{\fy_r^\ell} \right\|_\CHND{\beta}{\psi_2}.
\end{align*}
Next, we use Lemma \ref{lemma:concentration:fd-to-holder} to eliminate the finite difference in favour of a higher Hölder norm
\begin{equation*}
  \left\|\Delta_x^\alpha \Delta_y^\beta \egpr_r^\ell \right\|_\CHHND{\alpha}{\beta}{\psi_1}
  \lesssim \left\|\activationp{\fx_r^\ell} \right\|_\CHND{2\alpha}{\psi_2} 
      \left\| \activationp{\fy_r^\ell} \right\|_\CHND{2\beta}{\psi_2}.
\end{equation*}
Finally, Lemma \ref{lemma:concentration:continuity:bounded-lipschitz} implies that $\left\|\activationp{\fx_r^\ell} \right\|_\CHND{2\alpha}{D;\psi_2} \le n_0^{1/2} n_{\ell-1}^{-1/2}$ and likewise for $\fy_r^\ell$ and thus
\begin{equation*}
  \left\|\Delta_x^\alpha \Delta_y^\beta \egpr_r^\ell \right\|_\CHHND{\alpha}{\beta}{\psi_1}
  \lesssim \frac{n_0}{n_{\ell-1}}.
\end{equation*}

\end{proof}

\begin{lemma} \label{lemma:concentration:concentration-last-layer}

  Assume for $k=0, \dots, \ell-2$ the weights $W_k$ are fixed and bounded $\|W^k\| n_k^{-1/2} \lesssim 1$. Assume that $W^{\ell-1}$ is i.i.d. sub-gaussian with $\|W_{ij}^{\ell-1}\|_{\psi_2} \lesssim 1$. Assume that the domain $D$ is bounded, that $\activation$ satisfies the growth condition \eqref{eq:assumption:activation-growth}, has bounded derivative \eqref{eq:assumption:dactivation-bounded} and may be different in each layer. Then for $\alpha = \beta = 1/2$

  \begin{equation*}
    \pr{
      \left\|\egp^\ell - \E{\egp^\ell}\right\|_\CHHND{\alpha}{\beta}{\dom} 
      \ge C \frac{n_0}{n_{\ell-1}}\left[ \frac{\sqrt{d}+\sqrt{u}}{\sqrt{n_{\ell-1}}} + \frac{d+u}{n_{\ell-1}} \right]
    }
    \le e^{-u}.
  \end{equation*}

\end{lemma}

\begin{proof}

Since $\Delta_x^\alpha \Delta_y^\beta \egpr_r^\ell$ for $r \in [n_\ell]$ only depends on the random vector $W_{r\cdot}^{\ell-1}$, all stochastic processes $\left( \Delta_{x,h_x}^\alpha \Delta_{y,h_y}^\beta \egpr_r^\ell(x,y) \right)_{(x, h_x, y, h_y) \in \Delta\dom \times \Delta\dom}$ are independent and satisfy
\begin{equation*}
  \left\|\Delta_x^\alpha \Delta_y^\beta \egpr_r^\ell \right\|_\CHHND{\alpha}{\beta}{\Delta\dom \times \Delta\dom;\psi_1}
  \lesssim \frac{n_0}{n_{\ell-1}}
\end{equation*}
by Lemma \ref{lemma:concentration:continuity:holder}. Thus, we can estimate the processes' supremum by the chaining Corollary \ref{corollary:supplements:chaining}
\begin{equation*}
  \pr{
    \sup_{\substack{(x,h_x) \in \Delta\dom \\ (y,h_y) \in \Delta\dom}} \left\|\frac{1}{n_{\ell-1}} \sum_{r=1}^{n_{\ell-1}} \Delta_x^\alpha \Delta_y^\beta \egpr_r^\ell - \E{\Delta_x^\alpha \Delta_y^\beta \egpr_r^\ell}\right\| 
    \ge C \tau
  }
  \le e^{-u},
\end{equation*}
with
\begin{equation*}
  \tau = \frac{n_0}{n_{\ell-1}}\left[ \left(\frac{d}{n_{\ell-1}}\right)^{1/2} + \frac{d}{n_{\ell-1}} + \left(\frac{u}{n_{\ell-1}}\right)^{1/2} + \frac{u}{n_{\ell-1}} \right].
\end{equation*}
Noting that
\begin{equation*}
  \sup_{\substack{(x,h_x) \in \Delta\dom \\ (y,h_y) \in \Delta\dom}} \left|\Delta_x^\alpha \Delta_y^\beta \,\, \cdot \,\, \right|
  = \|\cdot\|_\CHHND{\alpha}{\beta}{\dom}
\end{equation*}
and
\begin{equation*}
  \frac{1}{n_{\ell-1}} \sum_{r=1}^{n_{\ell-1}} \Delta_x^\alpha \Delta_y^\beta \egpr_r^\ell
  = \Delta_x^\alpha \Delta_y^\beta \frac{1}{n_{\ell-1}} \sum_{r=1}^{n_{\ell-1}} \egpr_r^\ell
  = \Delta_x^\alpha \Delta_y^\beta \egp^\ell
\end{equation*}
completes the proof.

\end{proof}

\subsubsection{Perturbation of Covariances}
\label{sec:concentration:covariance}

This section contains the tools to estimate
\begin{equation*}
  \left\| \EE{\ell}{\egp^{\ell+1}} - \gp^{\ell+1} \right\|_\CHHN{\alpha}{\beta},
\end{equation*}
with an argument analogous to \cite{DuLeeLiEtAl2019}, except that we measure differences in Hölder norms.
As we will see in the next section, both $\EE{\ell}{\egp^{\ell+1}}$ and $\gp^{\ell+1}$ are of the form
\begin{equation*}
  \EE{(u,v) \sim \gaussian{0,A}}{\activation(u) \activation(v)},
\end{equation*}
with two different matrices $A$ and $\eA$ and thus it suffices to show that the above expectation is Hölder continuous in $A$. By a variable transform
\begin{equation*}
  A =
  \begin{bmatrix}
    a_{11} & a_{12} \\
    a_{21} & a_{22}
  \end{bmatrix}
  =
  \begin{bmatrix}
    a^2 & \rho a b \\
    \rho a b & b^2
  \end{bmatrix}
\end{equation*}
and rescaling, we reduce the problem to matrices of the form
\begin{align*}
  A & =
  \begin{bmatrix}
    1 & \rho \\
    \rho & 1
  \end{bmatrix}. 
\end{align*}
For these matrices, by Mehler's theorem we decompose the expectation as
\begin{equation*}
  \EE{(u,v) \sim \gaussian{0,A}}{\activation(u) \activation(v)} 
  = \sum_{k=0}^\infty \dualp{\activation, H_k}_N \dualp{\activation, H_k}_N \frac{\rho^k}{k!},
\end{equation*}
where $H_k$ are Hermite polynomials. The rescaling introduces rescaled activation functions, which we denote by
\begin{equation} \label{eq:concentration:scaled-activation}
  \activation_a(x) := \activation(ax).
\end{equation}
Finally, we show Hölder continuity by bounding derivatives. To this end, we use the multi-index $\gamma$ to denote derivatives $\partial^\gamma = \partial_a^{\gamma_a} \partial_b^{\gamma_b} \partial_\rho^{\gamma_\rho}$ with respect to the transformed variables. Details are as follows.

\begin{lemma} \label{lemma:concentration:mehler}
  Let
  \begin{equation*}
    A =
    \begin{bmatrix}
      a^2 & \rho a b \\
      \rho a b & b^2
    \end{bmatrix}
    =
    \begin{bmatrix}
      a & \\ & b
    \end{bmatrix}
    \begin{bmatrix}
      1 & \rho \\
      \rho & 1 
    \end{bmatrix}
    \begin{bmatrix}
      a & \\ & b
    \end{bmatrix}.
  \end{equation*}
  Then
  \begin{equation*}
    \EE{(u,v) \sim \gaussian{0,A}}{\activation(u) \activation(v)} 
    = \sum_{k=0}^\infty \dualp{\activation_a, H_k}_N \dualp{\activation_b, H_k}_N \frac{\rho^k}{k!}.
  \end{equation*}
  
\end{lemma}

\begin{proof}

By rescaling, or more generally, linear transformation of Gaussian random variables, we have
\begin{align*}
  \EE{(u,v) \sim \gaussian{0,A}}{\activation(u) \activation(v)} 
  & =
  \int \activation(u) \activation(v) dN\left(0,
    \begin{bmatrix}
      a & \\ & b
    \end{bmatrix}
    \begin{bmatrix}
      1 & \rho \\
      \rho & 1 
    \end{bmatrix}
    \begin{bmatrix}
      a & \\ & b
    \end{bmatrix}
  \right)(u,v)
  \\
  & =
  \int \activation(au) \activation(bv) dN\left(0,
    \begin{bmatrix}
      1 & \rho \\
      \rho & 1 
    \end{bmatrix}
  \right)(u,v).
\end{align*}
Thus, by Mehler's theorem (Theorem \ref{th:supplements:hermite:mehler} in the appendix) we conclude that 
\begin{align*}
  \EE{(u,v) \sim \gaussian{0,A}}{\activation(u) \activation(v)} 
  & =
  \iint \activation(au) \activation(bv) \sum_{k=0}^\infty H_k(u) H_k(v) \frac{\rho^k}{k!} \, d\gaussian{0,1}(u) \, d\gaussian{0,1}(v)
  \\
  & = \sum_{k=0}^\infty \dualp{\activation_a, H_k}_N \dualp{\activation_b, H_k}_N \frac{\rho^k}{k!}.
\end{align*}

\end{proof}

\begin{lemma} \label{lemma:concentration:covariance-derivative}
Assume $A = \begin{bmatrix} a^2 & \rho a b \\ \rho a b & b^2 \end{bmatrix}$ is positive semi-definite and all derivatives up to $\activation^{(\gamma_a + \gamma_\rho)}$ and $\activation_b^{(\gamma_b + \gamma_\rho)}$ are continuous and have at most polynomial growth for $x \to \pm \infty$. Then
\begin{equation*}
  \partial^\gamma \EE{(u,v) \sim \gaussian{0,A}}{\activation(u) \activation(v)} 
  \le \left\|\partial^{\gamma_a+\gamma_\rho} (\activation_a)\right\|_N \left\|\partial^{\gamma_b+\gamma_\rho} (\activation_b)\right\|_N.
\end{equation*}
\end{lemma}

\begin{proof}

By Lemma \ref{lemma:concentration:mehler}, we have
\begin{equation} \label{eq:proof:1:lemma:concentration:covariance-derivative}
  \begin{aligned}
    \partial^\gamma \EE{(u,v) \sim \gaussian{0,A}}{\activation(u) \activation(v)} 
    & = \partial^\gamma \sum_{k=0}^\infty \dualp{\activation_a, H_k}_N \dualp{\activation_b, H_k}_N \frac{\rho^k}{k!}
    \\
    & = \sum_{k=0}^\infty \partial^{\gamma_a} \dualp{\activation_a, H_k}_N \partial^{\gamma_b} \dualp{\activation_b, H_k}_N \partial^{\gamma_\rho} \frac{\rho^k}{k!}.
  \end{aligned}
\end{equation}
We first estimate the $\rho$ derivative. Since $0 \preceq A$ and $a,b > 0$, we must have $0 \preceq \begin{bmatrix} 1 & \rho \\ \rho & 1 \end{bmatrix}$ and thus $\det \begin{bmatrix} 1 & \rho \\ \rho & 1 \end{bmatrix} = 1 - \rho^2 \ge 0$. It follows that $|\rho| \le 1$. Therefore
\begin{equation} \label{eq:proof:2:lemma:concentration:covariance-derivative}
 \left|\partial^{\gamma_\rho} \frac{\rho^k}{k!}\right|
 = \left|\frac{1}{k!} \frac{k!}{(k-\gamma_\rho)!} \rho^{k-\gamma_\rho}\right|
 \le \frac{1}{(k-\gamma_\rho)!}.
\end{equation}
We eliminate the denominator $(k-\gamma_\rho)!$ by introducing extra derivatives into $\partial^{\gamma_a} \dualp{\activation_a, H_k}_N$. To this end, by Lemma \ref{lemma:supplements:hermite:properties}, we decrease the degree of the Hermite polynomial for a higher derivative on $\activation_a$:
\begin{equation*}
  \partial^{\gamma_a} \dualp{\activation_a, H_k}_N
  = \dualp{\partial^{\gamma_a} (\activation_a), H_k}_N
  = \dualp{\partial^{\gamma_a+\gamma_\rho} (\activation_a), H_{k-\gamma_\rho}}_N.
\end{equation*}
By Lemma \ref{lemma:supplements:hermite:properties}, $\|\cdot\|_N$ normalized Hermite polynomials are given by
\begin{equation*}
  \bar{H}_k := \frac{1}{\sqrt{k!}} H_k
\end{equation*}
and thus
\begin{equation*}
  \partial^{\gamma_a} \dualp{\activation_a, H_k}_N
  = \dualp{\partial^{\gamma_a+\gamma_\rho} (\activation_a), \bar{H}_{k-\gamma_\rho}}_N \sqrt{(k-\gamma_\rho)!}.
\end{equation*}
Plugging the last equation and \eqref{eq:proof:2:lemma:concentration:covariance-derivative} into \eqref{eq:proof:1:lemma:concentration:covariance-derivative}, we obtain
\begin{multline*}
  \partial^\gamma \EE{(u,v) \sim \gaussian{0,A}}{\activation(u) \activation(v)} 
  \\
  \begin{aligned}
    & \le \sum_{k=0}^\infty \left|\dualp{\partial^{\gamma_a+\gamma_\rho} (\activation_a), \bar{H}_k}_N \right| \left|\dualp{\partial^{\gamma_b+\gamma_\rho} (\activation_b), \bar{H}_k}_N\right|
    \\
    & \le \left(\sum_{k=0}^\infty \dualp{\partial^{\gamma_a+\gamma_\rho} (\activation_a), \bar{H}_k}_N^2 \right)^{1/2} \left(\sum_{k=0}^\infty \dualp{\partial^{\gamma_b+\gamma_\rho} (\activation_b), \bar{H}_k}_N^2 \right)^{1/2},
    \\
    & = \left\|\partial^{\gamma_a+\gamma_\rho} (\activation_a)\right\|_N \left\|\partial^{\gamma_b+\gamma_\rho} (\activation_b)\right\|_N,
  \end{aligned}
\end{multline*}
where in the second step we have used Cauchy-Schwarz and in the last that $\bar{H}_k$ are an orthonormal basis.

\end{proof}

\begin{lemma} \label{lemma:concentration:matrix-change-of-variable}

  Let $f(a_{11}, a_{22}, a_{12})$ be implicitly defined by solving the identity 
  \begin{equation*}
    \begin{bmatrix}
      a_{11} & a_{12} \\
      a_{12} & a_{22}
    \end{bmatrix}
    =
    \begin{bmatrix}
      a & \rho a b \\
      \rho a b & b
    \end{bmatrix}
  \end{equation*}
  for $a$, $b$ and $\rho$. Let $\dom_f$ be a domain with $a_{11}, a_{22} \ge c > 0$ and $|a_{12}| \lesssim 1$. Then
  \begin{equation*}
    \|f'''\|_\CND{1}{\dom_f} \lesssim 1.
  \end{equation*}
  
\end{lemma}

\begin{proof}
  
Comparing coefficients, $f$ is explicitly given by
\begin{equation*}
  f(a_{11}, a_{22}, a_{12})
  = \begin{bmatrix}
    a_{11} & a_{22} & \frac{a_{12}}{a_{11} a_{22}}
  \end{bmatrix}^T.
\end{equation*}
Since the denominator is bounded away from zero, all third partial derivatives exist and are bounded.

\end{proof}

\begin{lemma} \label{lemma:concentration:covariance-perturbation}
  For $\dom \subset \real^d$ and $x,y \in \dom$, let
  \begin{align*}
    A(x,y) & = 
    \begin{bmatrix}
      a_{11}(x,y) & a_{12}(x,y) \\
      a_{12}(x,y) & a_{22}(x,y)
    \end{bmatrix}
    & B(x,y) & = 
    \begin{bmatrix}
      b_{11}(x,y) & b_{12}(x,y) \\
      b_{12}(x,y) & b_{22}(x,y)
    \end{bmatrix},
  \end{align*}
  with
  \begin{align*}
    a_{11}(x,y) & \ge c > 0, & a_{22}(x,y) & \ge c > 0, & |a_{12}(x,y)| \lesssim 1, \\
    b_{11}(x,y) & \ge c > 0, & b_{22}(x,y) & \ge c > 0, & |b_{12}(x,y)| \lesssim 1.
  \end{align*}
  Assume the derivatives $\activation^{(i)}$, $i=0, \dots, 3$ are continuous and have at most polynomial growth for $x \to \pm \infty$ and for all $a \in \{a(x,y) : \, x,y \in \dom, \, a \in \{a_{11}, a_{22}, b_{11}, b_{22} \} \}$ the scaled activation satisfies
  \begin{align*}
    \left\|\partial^i (\activation_a) \right\|_N & \lesssim 1, & 
    i & = 1, \dots, 3,
  \end{align*}
  with $\activation_a$ defined in \eqref{eq:concentration:scaled-activation}. Then, for $\alpha, \beta \le 1$ the functions
  \begin{align*}
    x & \to \EE{(u,v) \sim \gaussian{0,A(x,y)}}{\activation(u) \activation(v)},
    \\
    x & \to \EE{(u,v) \sim \gaussian{0,B(x,y)}}{\activation(u) \activation(v)}
  \end{align*}
  satisfy
  \begin{multline*}
    \left\| \EE{(u,v) \sim \gaussian{0,A}}{\activation(u) \activation(v)} - \EE{(u,v) \sim \gaussian{0,B}}{\activation(u) \activation(v)} \right\|_\CHHND{\alpha}{\beta}{\dom}
    \\
    \lesssim \|A\|_\CHHND{\alpha}{\beta}{\dom} \|B\|_\CHHND{\alpha}{\beta}{\dom} \|A-B\|_\CHHND{\alpha}{\beta}{\dom}.
  \end{multline*}

\end{lemma}

\begin{proof}

Define
\begin{align*}
  F(a, b, \rho) & = \EE{(u,v) \sim \gaussian{0,\bar{A}}}{\activation(u) \activation(v)}. &
  \bar{A} & = 
  \begin{bmatrix}
    a & \rho a b \\ \rho a b & b
  \end{bmatrix}
\end{align*}
and $f(a_{11}, a_{22}, a_{12})$ by solving the identity 
\begin{equation*}
  \begin{bmatrix}
    a_{11} & a_{12} \\
    a_{12} & a_{22}
  \end{bmatrix}
  =
  \begin{bmatrix}
    a & \rho a b \\
    \rho a b & b
  \end{bmatrix}
\end{equation*}
for $a$, $b$ and $\rho$. Then
\begin{align*}
  F \circ f \circ A
  & = x,y \to \EE{(u,v) \sim \gaussian{0,A(x,y)}}{\activation(u) \activation(v)},
  \\
  F \circ f \circ B
  & = x,y \to \EE{(u,v) \sim \gaussian{0,B(x,y)}}{\activation(u) \activation(v)}
\end{align*}
and
\begin{multline*}
  \left\| \EE{(u,v) \sim \gaussian{0,A}}{\activation(u) \activation(v)} - \EE{(u,v) \sim \gaussian{0,B}}{\activation(u) \activation(v)} \right\|_\CHHND{\alpha}{\beta}{\dom}
  \\
  = \left\| F \circ f \circ A - F \circ f \circ B \right\|_\CHHND{\alpha}{\beta}{\dom}.
\end{multline*}
By Lemmas \ref{lemma:supplements:fd-composition} (for $\Delta^\alpha$ and $\Delta^\beta$) and \ref{lemma:supplements:fd-composition-2} (for $\Delta^\alpha \Delta^\beta$), we have
\begin{multline*}
  \left\| F \circ f \circ A - F \circ f \circ B \right\|_\CHHND{\alpha}{\beta}{\dom}
  \\
  \lesssim \|F \circ f\|_\CND{3}{\dom_f} \|A-B\|_\CHHND{\alpha}{\beta}{\dom} 
  \\
    \max\{1,\, \|A\|_\CHHND{\alpha}{\beta}{\dom} \}
    \max\{1,\, \|B\|_\CHHND{\alpha}{\beta}{\dom} \},
\end{multline*}
with $\dom_f = A(\dom) \cup B(\dom)$, so that it suffices to bound $\|F \circ f\|_\CND{3}{\dom_f} \lesssim 1$. This follows directly from the assumptions, chain rule, product rule and Lemmas \ref{lemma:concentration:covariance-derivative} and \ref{lemma:concentration:matrix-change-of-variable}. Finally, we simplify
\begin{equation*}
  \max\{1,\, \|A\|_\CHHND{\alpha}{\beta}{\dom} \} \le \frac{1}{c}\|A\|_\CHHND{\alpha}{\beta}{\dom}
\end{equation*}
because 
\begin{equation*}
  \frac{1}{c}\|A\|_\CHHND{\alpha}{\beta}{\dom}
  \ge \frac{1}{c} a_{11}(\cdot) \ge 1
\end{equation*}
and likewise for $B$.
  
\end{proof}

\subsubsection{Concentration of the NTK}
\label{sec:concentration:ntk}

We combine the results from the last two sections to show concentration inequalities, first for the forward kernels $\gp^\ell$ and $\dgp^\ell$ and then for the NTK $\ntk$.

\begin{lemma} \label{lemma:concentration:concentration-gp}

  Let $\alpha = \beta = 1/2$ and $k = 0, \dots, \ell$.
  \begin{enumerate}
    \item Assume that all $W^k$ are are i.i.d. standard normal. 
    \item Assume that $\activation$ satisfies the growth condition \eqref{eq:assumption:activation-growth}, has uniformly bounded derivative \eqref{eq:assumption:dactivation-bounded}, derivatives $\activation^{(i)}$, $i=0, \dots, 3$ are continuous and have at most polynomial growth for $x \to \pm \infty$ and the scaled activations satisfy
    \begin{align*}
      \left\|\partial^i (\activation_a) \right\|_N & \lesssim 1, & 
      a & \in \{\gp^k(x,x): x \in \dom\}, &
      i & = 1, \dots, 3,
    \end{align*}
    with $\activation_a$ defined in \eqref{eq:concentration:scaled-activation}. The activation function may be different in each layer.
    \item For all $x \in \dom$ assume
    \begin{equation*}
      \gp^k(x,x) \ge c_\gp > 0.
    \end{equation*}
    \item The widths satisfy $n_\ell \gtrsim n_0$ for all $\ell=0, \dots, L$.
  \end{enumerate}
  Then, with probability at least
  \begin{equation*}
    1 - c \sum_{k=1}^{\ell-1} e^{-n_k} + e^{-u_k}
  \end{equation*}
  we have
  \begin{align*}
    \left\| \gp^\ell \right\|_\CHHN{\alpha}{\beta}  & \lesssim 1
    \\
    \left\| \egp^\ell \right\|_\CHHN{\alpha}{\beta}  & \lesssim 1
    \\
    \left\| \egp^\ell - \gp^\ell \right\|_\CHHN{\alpha}{\beta} 
    & \lesssim \sum_{k=0}^{\ell-1} \frac{n_0}{n_k}\left[ \frac{\sqrt{d}+\sqrt{u_k}}{\sqrt{n_k}} + \frac{d+u_k}{n_k} \right]
    \le \frac{1}{2} c_\gp
  \end{align*}
  for all $u_1, \dots, u_{\ell-1} \ge 0$ sufficiently small so that the last inequality holds.
  
\end{lemma}

\begin{proof}

We prove the statement by induction. Let us first consider $\ell \ge 1$. We split off the expectation over the last layer
\begin{align*}
  \left\| \egp^{\ell+1} - \gp^{\ell+1} \right\|_\CHHN{\alpha}{\beta}
  & \le \left\| \egp^{\ell+1} - \EE{\ell}{\egp^{\ell+1}} \right\|_\CHHN{\alpha}{\beta}
    + \left\| \EE{\ell}{\egp^{\ell+1}} - \gp^{\ell+1} \right\|_\CHHN{\alpha}{\beta}
  \\
  & = I + II,
\end{align*}
where $\EE{\ell}{\cdot}$ denotes the expectation with respect to $W^\ell$. We Estimate $I$, given that the lower layers satisfy
\begin{align} \label{eq:proof:1:lemma:concentration:concentration-gp}
  \|W^k\| n_k^{-1/2} & \lesssim 1, & k & = 0, \dots, \ell-1,
\end{align}
which is true with probability at least $1 - 2 e^{-n_{k}}$, see e.g. \cite[Theorem 4.4.5]{Vershynin2018}. Then, by Lemma \ref{lemma:concentration:concentration-last-layer} for $u_\ell \ge 0$
\begin{equation} \label{eq:proof:2:lemma:concentration:concentration-gp}
  \pr{
    \left\|\egp^{\ell+1} - \E{\egp^{\ell+1}}\right\|_\CHHND{\alpha}{\beta}{\dom} 
    \ge C \frac{n_0}{n_\ell}\left[ \frac{\sqrt{d}+\sqrt{u_\ell}}{\sqrt{n_\ell}} + \frac{d+u_\ell}{n_\ell} \right]
  }
  \le e^{-u_\ell}.
\end{equation}
Next we estimate II. To this end, recall that $\egp^{\ell+1}(x,y)$ is defined by
\begin{equation*}
  \egp^{\ell+1}(x,y) 
  = \frac{1}{n_\ell}\sum_{r=1}^{n_{\ell+1}} \activationp{f_r^{\ell+1}(x)} \activationp{f_r^{\ell+1}(y)}.
\end{equation*}
For fixed lower layers $W^0, \dots, W^{\ell-1}$, the inner arguments
\begin{align*}
  f_r^{\ell+1}(x) & = W_{r\cdot}^\ell n_\ell^{-1/2} \activationp{f^\ell(x)} &
  f_r^{\ell+1}(x) & = W_{r\cdot}^\ell n_\ell^{-1/2} \activationp{f^\ell(y)}
\end{align*}
are Gaussian random variables in $W_{r\cdot}^\ell$ with covariance
\begin{multline} \label{eq:proof:3:lemma:concentration:concentration-gp}
  \EE{l}{
    W_{r\cdot}^\ell n_\ell^{-1/2} \activationp{f^\ell(x)}^T
    W_{r\cdot}^\ell n_\ell^{-1/2} \activationp{f^\ell(y)}
  }
  \\
  = \frac{1}{n_\ell} \sum_{r=1}^{n_\ell} n_\ell^{-1/2} \activationp{f^\ell(x)} n_\ell^{-1/2} \activationp{f^\ell(y)}
  = \egp^\ell(x,y).
\end{multline}
It follows that
\begin{align*}
  \EE{\ell}{\egp^{\ell+1}(x,y)}
  & = \EE{(u,v) \sim \gaussian{0,\eA}}{\activation(u) \activation(v)}, & 
  \eA
  & = \begin{bmatrix}
    \egp^\ell(x,x) & \egp^\ell(x,y) \\
    \egp^\ell(y,x) & \egp^\ell(y,y) \\
  \end{bmatrix}.
\end{align*}
This matches the definition 
\begin{align*}
  \gp^{\ell+1}(x,y) & = \EE{u,v \sim \gaussian{0, A}}{\activationp{u}, \activationp{v}} &
  A & = \begin{bmatrix}
    \gp^\ell(x,x) & \gp^\ell(x,y) \\
    \gp^\ell(y,x) & \gp^\ell(y,y)
  \end{bmatrix} & 
\end{align*}
of the process $\gp^{\ell+1}$ up to the covariance matrix $\eA$ versus $A$. Thus, we can estimate the difference $\left\|\EE{\ell}{\egp^{\ell+1}(x,y)} - \gp^{\ell+1}\right\|_\CHHN{\alpha}{\beta}$ by Lemma \ref{lemma:concentration:covariance-perturbation} if the entries of $A$ and $\eA$ satisfy the required bounds. To this end, we first bound the diagonal entries away from zero. For $A$, this is true by assumption. For $\eA$, by induction, with probability at least $1 - c \sum_{k=1}^{\ell-1} e^{-n_k} + e^{-u_k}$ we have
\begin{equation} \label{eq:proof:4:lemma:concentration:concentration-gp}
  \left\| \egp^\ell - \gp^\ell \right\|_\CHHN{\alpha}{\beta} 
  \lesssim \sum_{k=0}^{\ell-1} \frac{n_0}{n_k}\left[ \frac{\sqrt{d}+\sqrt{u_k}}{\sqrt{n_k}} + \frac{d+u_k}{n_k} \right]
  \le \frac{1}{2} c_\gp.
\end{equation}
In the event that this is true, we have 
\begin{equation*}
  \egp^\ell(x,x) \ge \frac{1}{2} c_\gp > 0.
\end{equation*}
Next, we bound the off diagonal terms. Since the weights are bounded \eqref{eq:proof:1:lemma:concentration:concentration-gp}, Lemma \ref{lemma:continuity:egp-holder} implies
\begin{align*}
  \left\|\egp^\ell\right\|_\CHHN{\alpha}{\beta} & \lesssim \frac{n_0}{n_l} \lesssim 1, &
  \left\|\gp^\ell\right\|_\CHHN{\alpha}{\beta} & \lesssim 1,
\end{align*}
where the last inequality follows from \eqref{eq:proof:4:lemma:concentration:concentration-gp}. In particular,
\begin{align*}
  \egp^\ell(x,y) & \lesssim 1, &
  \gp^\ell(x,y) & \lesssim 1
\end{align*}
for all $x,y \in \dom$. Hence, we can apply Lemma \ref{lemma:concentration:covariance-perturbation} and obtain
\begin{multline*}
  \left\|\EE{\ell}{\egp^{\ell+1}} - \gp^{\ell+1}\right\|_\CHHN{\alpha}{\beta}
  \lesssim \left\|\gp^{\ell}\right\|_\CHHN{\alpha}{\beta} \left\|\egp^{\ell}\right\|_\CHHN{\alpha}{\beta} \left\|\egp^{\ell} - \gp^{\ell}\right\|_\CHHN{\alpha}{\beta}
  \\
  \lesssim \left\|\egp^{\ell} - \gp^{\ell}\right\|_\CHHN{\alpha}{\beta}.
  \lesssim \sum_{k=0}^{\ell-1} \frac{n_0}{n_k}\left[ \frac{\sqrt{d}+\sqrt{u_k}}{\sqrt{n_k}} + \frac{d+u_k}{n_k} \right],
\end{multline*}
where the last line follows by induction. Together with \eqref{eq:proof:1:lemma:concentration:concentration-gp}, \eqref{eq:proof:2:lemma:concentration:concentration-gp} and a union bound, this shows the result for $\ell \ge 1$.

Finally, we consider the induction start for $\ell=0$. The proof is the same, except that in \eqref{eq:proof:3:lemma:concentration:concentration-gp} the covariance simplifies to
\begin{multline*}
  \EE{l}{
    f^1(x)
    f^1(y)
  }
  = \EE{l}{
    (W_{r\cdot}^0 V x)
    (W_{r\cdot}^0 y)
  }
  = (Vx)^T (Vy)
  = x^T y
  = \gp^0(x,y).
\end{multline*}
Hence, for $\ell=1$ the two covariances $A$ and $\eA$ are identical and therefore $\|\EE{0}{\egp^1(x,y)} - \gp^1\|_\CHHN{\alpha}{\beta} = 0$.

\end{proof}

\begin{lemma}[Lemma \ref{lemma:concentration:concentration-ntk}, restated from the overview]
  \LemmaConcentration{\label{eq:1:lemma:concentration:concentration-ntk}}  
\end{lemma}

\begin{proof}
By definition \eqref{eq:def-ntk} of $\ntk$ and Lemma \ref{lemma:ntk:entk-as-gp} for $\entk$, we have
\begin{align*}
  \ntk(x,y) & = \dgp^L(x,y) \gp^{L-1}(x,y), &
  \entk(x,y) & = \edgp^L(x,y) \egp^{L-1}(x,y)
\end{align*}
and therefore
\begin{align*}
  \left\|\ntk - \entk\right\|_\CHHN{\alpha}{\beta}
  & = \left\|\dgp^L \gp^{L-1} - \edgp^L \egp^{L-1}\right\|_\CHHN{\alpha}{\beta}
  \\
  & = \left\| \left[\dgp^L - \edgp^L \right] \gp^{L-1}\right\|_\CHHN{\alpha}{\beta}
    + \left\|\edgp^L \left[\gp^{L-1} - \egp^{L-1}\right]\right\|_\CHHN{\alpha}{\beta}
  \\
  & = \left\| \dgp^L - \edgp^L \right\|_\CHHN{\alpha}{\beta}  \left\|\gp^{L-1}\right\|_\CHHN{\alpha}{\beta}
    + \left\|\edgp^L\right\|_\CHHN{\alpha}{\beta} \left\|\gp^{L-1} - \egp^{L-1}\right\|_\CHHN{\alpha}{\beta},
\end{align*}
where in the last step we have used Lemma \ref{lemma:supplements:holder-properties} Item \ref{item:product:lemma:supplements:holder-properties}. Thus, the result follows from
\begin{align*}
  \left\| \gp^{L-1} \right\|_\CHHN{\alpha}{\beta}  & \lesssim 1, &
  \left\| \egp^{L-1} \right\|_\CHHN{\alpha}{\beta}  & \lesssim 1, &
  \left\| \dgp^L \right\|_\CHHN{\alpha}{\beta}  & \lesssim 1, &
  \left\| \edgp^L \right\|_\CHHN{\alpha}{\beta}  & \lesssim 1
\end{align*}
and
\begin{multline*}
  \max \left\{\left\|\gp^{L-1} - \egp^{L-1}\right\|_\CHHN{\alpha}{\beta}, \, \left\|\dgp^L - \edgp^L\right\|_\CHHN{\alpha}{\beta}\right\}
  \\
  \lesssim \sum_{k=0}^{L-1} \frac{n_0}{n_k}\left[ \frac{\sqrt{d}+\sqrt{u_k}}{\sqrt{n_k}} + \frac{d+u_k}{n_k} \right]
  \le \frac{1}{2} c_\gp,
\end{multline*}
with probability \eqref{eq:1:lemma:concentration:concentration-ntk} by Lemma \ref{lemma:concentration:concentration-gp}. For $\dgp^L$, we do not require the lower bound $\dgp^k(x,x) \ge c_\gp > 0$ because in the recursive definition $\dactivation$ is only used in the last layer and therefore not necessary in the induction step in the proof of Lemma \ref{lemma:concentration:concentration-gp}.
  
\end{proof}

\subsection{Proof of Lemma \ref{lemma:close-to-initial:close-to-initial}: Weights stay Close to Initial}

The derivative $\partial_{W^k} f^\ell(x) \in \real^{n_{\ell-1} \times (n_{k+1} \times n_k)}$ is a tensor with three axes for which we define the norm
\begin{equation*}
  \left\|\partial_{W^k} f^\ell(x)\right\|_*
  := \sup_{\substack{\|u\|, \|v\|, \|w\| \le 1}} \sum_{r,i,j} u_r v_i w_j \partial_{W_{ij}^k} f_r^\ell(x)
\end{equation*}
and the corresponding maximum norm $\|\cdot\|_\CND{0}{\dom;*}$ for functions mapping $x$ to a tensor measured in the $\|\cdot\|_*$ norm. We use this norm for an inductive argument in a proof, but later only apply it for the last layer $\ell=L+1$. In this case $n_{L+1} = 1$ and the norm reduces to a regular matrix norm.

\begin{lemma} \label{lemma:close-to-initial:w-derivative}
  Assume that $\activation$ satisfies the growth and derivative bounds \eqref{eq:assumption:activation-growth}, \eqref{eq:assumption:dactivation-bounded} and may be different in each layer. Assume the weights are bounded $\|W^k\| n_k^{-1/2} \lesssim 1$, $k=1, \dots, \ell-1$. Then for $0 \le \alpha \le 1$
  \begin{equation*}
    \left\| \partial_{W^k} f^\ell \right\|_\CND{0}{\dom;*} 
    \lesssim \left(\frac{n_0}{n_k}\right)^{1/2}.
  \end{equation*}
  
\end{lemma}

\begin{proof}

First note that for any tensor $T$
\begin{equation*}
  \left\| \sum_{r,i,j} u_r v_i w_j T_{rij} \right\|_\CN{0} \le C \|u\| \|v\| \|w\|
\end{equation*}
implies that $\|T\|_\CND{0}{\dom;*} \le C$, which we use throughout the proof. We proceed by induction over $\ell$. For $k \ge \ell$, the pre-activation $f^\ell$ does not depend on $W^k$ and thus $\partial_{W^k} f^\ell(x) = 0$. For $k = \ell-1$, we have
\begin{equation*}
  \partial_{W_{ij}^k} f_r^{k+1}(x)
  = \partial_{W_{ij}^k} W_{r\cdot}^k n_k^{-1/2} \activationp{f^k(x)}
  = \delta_{ir} n_k^{-1/2} \activationp{f_j^k(x)}
\end{equation*}
and therefore for any vectors $u$, $v$, $w$
\begin{multline*}
  \left\| \sum_{r,i,j} u_r v_i w_j \partial_{W_{ij}^k} f_r^k(x) \right\|_\CN{0}
  = \left\| n_k^{-1/2} (u^T v)\left(w^T \activationp{f^k}\right) \right\|_\CN{0}
  \\
  \le n_k^{-1/2} \|u\| \|v\| \|w\| \left\| \activationp{f^k} \right\|_\CN{0}
  \lesssim \|u\| \|v\| \|w\| \left(\frac{n_0}{n_k}\right)^{1/2},
\end{multline*}
where in the last step we have used Lemma \ref{lemma:continuity:nn-holder}. Thus, we conclude that
\begin{equation*}
  \left\| \partial_{W^k} f^{k+1}(x) \right\|_\CND{0}{\dom;*} 
  \lesssim \left(\frac{n_0}{n_k}\right)^{1/2}.
\end{equation*}
For $k < \ell-1$, we have
\begin{equation*}
  \partial_{W_{ij}^k} f^\ell(x)
  = \partial_{W_{ij}^k} W^{\ell-1} n_{\ell-1}^{-1/2} \activationp{f^{\ell-1}}
  = W^{\ell-1} n_{\ell-1}^{-1/2} \left[\dactivationp{f^{\ell-1}} \odot \partial_{W_{ij}^k}f^{\ell-1} \right]
\end{equation*}
and therefore
\begin{align*}
  \left\| \sum_{r,i,j} u_r v_i w_j \partial_{W_{ij}^\ell} f_r^k \right\|_\CN{0}
  & \le \|u^T W^{\ell-1} n_{\ell-1}^{-1/2} \| \|v\| \|w\| \left\| \dactivationp{f^{\ell-1}} \odot \partial_{W_{ij}^k} f^{\ell-1}\right\|_\CND{0}{\dom;*}
  \\
  & \le \|u\| \|v\| \|w\| \left\| \dactivationp{f^{\ell-1}} \right\|_\CND{0}{\dom;\ell_\infty} \left\| \partial_{W_{ij}^k} f^{\ell-1}\right\|_\CND{0}{\dom;*}
  \\
  & \lesssim \|u\| \|v\| \|w\| \left(\frac{n_0}{n_k}\right)^{1/2},
\end{align*}
where in the second step we have used that $\|W^{\ell-1}\| n_{\ell-1}^{-1/2} \lesssim 1$ and in the last step we have used that $\left\|\dactivationp{f^{\ell-1}}\right\|_{\ell_\infty} \lesssim 1$ because $|\dactivation(\cdot)| \lesssim 1$ and the induction hypothesis. It follows that
\begin{equation*}
  \left\| \partial_{W^k} f^\ell(x) \right\|_\CND{0}{\dom;*} 
  \lesssim \left(\frac{n_0}{n_k}\right)^{1/2}.
\end{equation*}

\end{proof}

\begin{lemma}[Lemma \ref{lemma:close-to-initial:close-to-initial}, restated from the overview]
  \LemmaCloseToInitial
\end{lemma}

\begin{proof}

By assumption, we have
\begin{align*}
  \|W^\ell(\tau)\| n_\ell^{-1/2} & \lesssim 1, &
  0 & \le \tau < t, &
  \ell & =1, \dots, L.
\end{align*}
With loss $\loss$ and residual $\res=f_\theta-f$, because
\begin{equation*}
  \frac{d}{d\tau} W^\ell 
  = -\nabla_{W^\ell} \loss
  = \int_\dom \res(x) D_{W_\ell} f^{L+1}(x) \, dx
\end{equation*}
we have
\begin{align*}
  \left\|W^\ell(t) - W^\ell(0)\right\|
  & = \left\|\int_0^t \frac{d}{d\tau} W^\ell(\tau) \, d\tau \right\|
  \\
  & = \left\|\int_0^t \int_\dom \res(x) D_{W_\ell} f^{L+1}(x) \, dx \, d\tau \right\|
  \\
  & \le \int_0^t \int_\dom |\res(x)| \left\|D_{W_\ell} f^{L+1}(x) \right\| \, dx \, d\tau
  \\
  & \lesssim \left(\frac{n_0}{n_\ell}\right)^{1/2} \int_0^t \|\res\|_{\CND{0}{\dom}'} \, dx \, d\tau,
\end{align*}
where in the last step we have used Lemma \ref{lemma:close-to-initial:w-derivative}. Multiplying with $n_\ell^{-1/2}$ shows the result.

\end{proof}

\subsection{Proof of Theorem \ref{th:convergence}: Main Result}
\label{sec:proof:th:convergence}

\begin{proof}[Proof of Theorem \ref{th:convergence}]

The result follows directly from Lemma \ref{lemma:convergence:perturbed-gradient-flow} with the smoothness spaces $\hs^\sm = H^\sm(\Sd)$. While the lemma bounds the residual $\res$ in the $\hs^{-\sm}$ and $\hs^\sm$ norms, we aim for an $\hs^0 = L_2(\Sd)$ bound. This follows directly from the interpolation inequality
\begin{equation*}
  \|\cdot\|_{L_2(\Sd)} 
  = \|\cdot\|_{H^0(\Sd)}
  \le \|\cdot\|_{H^{-\sm}(\Sd)}^{1/2}
    \|\cdot\|_{H^\sm(\Sd)}^{1/2}.
\end{equation*}
It remains to verify all assumptions. To this end, first note that the initial weights satisfy
\begin{align} \label{eq:1:proof:th:convergence}
  \|W(0)^\ell\| n_\ell^{-1/2} & \lesssim 1, & \ell & = 0, \dots, L,
\end{align}
with probability at least $1 - 2 e^{-c m}$ since $n_\ell \sim m$ by assumption, see e.g. \cite[Theorem 4.4.5]{Vershynin2018}. Then, the assumptions are shown as follows.
\begin{enumerate}

  \item \emph{The weights stay close to the initial \eqref{eq:perturbed-gradient-flow:weight-diff}:} We use the scaled matrix norm
  \begin{equation*}
    \wnorm{\theta} := \max_{L \in [L]} \|W^\ell\| n_\ell^{-1/2}
  \end{equation*}
  to measure the weight distance. Then, by \eqref{eq:1:proof:th:convergence} with $p_0(m) := 2L e^{-m}$ given that $\wnorm{\theta(\tau) - \theta(0)} \le 1$, Lemma \ref{lemma:close-to-initial:close-to-initial} implies that
  \begin{multline*}
    \wnorm{\theta(t) - \theta(0)}
    = \max_{\ell \in [L]} \left\|W^\ell(t) - W^\ell(0)\right\| n_\ell^{-1/2}
    \\
    \lesssim \frac{n_0^{1/2}}{n_\ell} \int_0^t \|\res\|_{\CND{0}{\Sd}'} \, dx \, d\tau,
    \lesssim m^{-1/2} \int_0^t \|\res\|_{H^{0}(\Sd)} \, dx \, d\tau,
  \end{multline*}
  where the last step follows from the assumption $n_0 \sim \dots \sim n_{L-1} =: m$  and the embedding $\|\cdot\|_{\CND{0}{\Sd}'} \lesssim \|\cdot\|_{H^0(\Sd)'} =  \|\cdot\|_{H^{0}(\Sd)}$, which follows directly from the inverted embedding $\|\cdot\|_{H^0(\Sd)} \lesssim \|\cdot\|_\CND{0}{\Sd}$.

  \item \emph{Norms and Scalar Product \eqref{eq:interpolation-inequality}:} Both are well known for Sobolev spaces, and follow directly from norm definition \eqref{eq:supplements:sobolev-spherical-harmonics} with Cauchy-Schwarz.

  \item \emph{Concentration of the Initial NTK \eqref{eq:perturbed-gradient-flow:initial}:} Since by \eqref{eq:setup:dactivation-growth} the first four derivatives of the activation function have at most polynomial growth, we have
  \begin{equation*}
    \| \partial^i(\activation_a)\|_
    = \int_\real \activation^{(i)}(ax) a^i \, d\gaussian{0,1})(x)
    \lesssim 1
  \end{equation*}
  for all $a \in \{\gp^k(x,x): x \in \dom\}$ contained in the set $\{c_\gp, C_\gp\}$ for some $C_\gp\ge0$, by assumption. Together with $\sm+\epsilon < 1/2$ for sufficiently small $\epsilon$, hidden dimensions $d \lesssim n_0 \sim \dots, \sim n_L =: m$ and the concentration result Lemma \ref{lemma:concentration:concentration-ntk} we obtain, with probability at least
  \begin{equation*}
    1 - p_\infty(m, \tau) := 1 - c L (e^{-m} + e^{-\tau})
  \end{equation*}
  the bound
  \begin{equation*}
    \left\| \entk - \ntk \right\|_\CHHN{\sm+\epsilon}{\sm+\epsilon} 
    \lesssim L \left[ \sqrt{\frac{d}{m}} + \sqrt{\frac{\tau}{m}} + \frac{\tau}{m} \right]
  \end{equation*}
  for the neural tangent kernel for all $0 \le \tau = u_0 = \cdots = u_{L-1} \lesssim 1$. By Lemma \ref{lemma:supplements:kernel-bound}, the kernel bound directly implies the operator norm bound
  \begin{equation*}
    \left\| \ntkop - \ntkop_{\theta(0)} \right\|_{-\sm,\sm}
    \lesssim L \left[ \sqrt{\frac{d}{m}} + \sqrt{\frac{\tau}{m}} + \frac{\tau}{m} \right]
  \end{equation*}
  for the corresponding integral operators $\ntkop$ and $\ntkop_{\theta(0)}$, with kernels $\ntk$ and $\entk$, respectively. If $\tau/m \lesssim 1$, we can drop the last term and thus satisfy assumption \eqref{eq:perturbed-gradient-flow:initial}.

  \item \emph{Hölder continuity of the NTK \eqref{eq:perturbed-gradient-flow:perturbation}:}
  By \eqref{eq:1:proof:th:convergence} with probability at least
  \begin{equation*}
    1 - p_L(m) := 1 - L e^{-m}
  \end{equation*}
  we have $\wnorm{\theta(0)} \lesssim 1$ and thus for all perturbations $\ptheta$ with $\wnorm{\ptheta - \theta(0)} \le h \le 1$ by Lemma \ref{lemma:continuity:entk-holder} that
  \begin{equation*}
    \left\|\entk - \pentk\right\|_\CHHN{\sm+\epsilon}{\sm+\epsilon}
    \lesssim L h^{1-\sm-\epsilon}
  \end{equation*}
  for any sufficiently small $\epsilon > 0$.
  By Lemma \ref{lemma:supplements:kernel-bound}, the kernel bound implies the operator norm bound
  \begin{equation*}
    \left\|H_{\theta(0)} - H_\ptheta\right\|_{\sm \leftarrow -\sm}
    \lesssim L h^\holder
  \end{equation*}
  for any $\holder < 1-\sm$ and integral operators $\ntkop_{\theta(0)}$ and $\ntkop_\ptheta$ corresponding to kernels $\ntk_\theta(0)$ and $\entk_\ptheta$, respectively.

  \item \emph{Coercivity \eqref{item:perturbed-gradient-flow:coercive}:} Is given by assumption.

\end{enumerate}
Thus, all assumptions of Lemma \ref{lemma:convergence:perturbed-gradient-flow} are satisfied, which directly implies the theorem as argued above.

\end{proof}

\section{Technical Supplements}
\label{sec:supplements}

\subsection{Hölder Spaces}
\label{sec:supplements:holder-spaces}

\begin{definition}
  \begin{enumerate}

    Let $U$ and $V$ be two normed spaces.

    \item For $0 < \alpha \le 1$, we define the Hölder spaces on the domain $\dom \subset U$ as all functions $f \colon \dom \to V$ for which the norm
    \begin{align*}
      \|f\|_\CHND{\alpha}{\dom;V}
      := \max \{ \|f\|_\CND{0}{\dom;V}, |f|_\CHND{\alpha}{\dom;V} \}
      < \infty
    \end{align*}
    is finite, with
    \begin{align*}
      |f|_\CND{0}{\dom;V}
      & := \sup_{x \in D} \|f(x)\|_V, &
      |f|_\CHND{\alpha}{\dom;V}
      & := \sup_{x \ne \px \in D} \frac{\|f(x) - f(\px)\|_V}{\|x - \px\|_U^\alpha}.
    \end{align*}

    \item For $0 < \alpha, \beta \le 1$, we define the mixed Hölder spaces on the domain $\dom \times \dom \subset U \times U$ as all functions $g \colon \dom \times \dom \to V$ for which the norm
    \begin{align*}
      \|f\|_\CHHND{\alpha}{\beta}{\dom;V}
      & := \max_{\substack{a \in \{0, \alpha\} \\ b \in \{0, \beta\}}} |f|_\CHHND{a}{b}{\dom;V}
      < \infty,
    \intertext{with}
      |f|_\CHHND{0}{0}{\dom;V} & := \sup_{x,y \in D} \|f(x,y)\|_V,
      \\
      |f|_\CHHND{\alpha}{0}{\dom;V} & := \sup_{x\ne\px, y \in D} \frac{\|f(x,y) - f(\px,y)\|_V}{\|x - \px\|_U^\alpha},
      \\
      |f|_\CHHND{0}{\beta}{\dom;V} & := \sup_{x,y\ne\py \in D} \frac{\|f(x,y) - f(x,\py)\|_V}{\|y - \py\|_U^\beta},
      \\
      |f|_\CHHND{\alpha}{\beta}{\dom;V} & := \sup_{x\ne\px, y\ne\py \in D} \frac{\|f(x,y) - f(\px,y) - f(x, \py) + f(\px,\py)\|_V}{\|x - \px\|_U^\alpha \|y - \py\|_U^\beta}.
    \end{align*}

    \item We use the following abbreviations:
    \begin{enumerate}
      \item If $D$ is understood from context and $V = \real^n$, both equipped with the Euclidean norm, we write
      \begin{equation*}
        \CHN{\alpha} = \CHND{\alpha}{\dom} = \CHND{\alpha}{\dom; \ell_2(\real^n)}.
      \end{equation*}
      \item If $V = L_{\psi_i}$, $i=1,2$ is an Orlicz space, we write
      \begin{equation*}
        \CHND{\alpha}{\dom; \psi_i} = \CHND{\alpha}{\dom; L_{\psi_i}}.
      \end{equation*}
    \end{enumerate}
    We use analogous abbreviations for all other spaces.

  \end{enumerate}

\end{definition}

It is convenient to express Hölder spaces in terms of finite difference operators, 
\begin{align*}
  \Delta^0_h f(x) & = f(x), &
  \Delta_h^\alpha f(x) & = \|h\|_U^{-\alpha} [f(x+h) - f(x)], &
  \alpha & > 0,
\end{align*}
which satisfy product and chain rules similar to derivatives. We may also consider these as functions in both $x$ and $h$ 
\begin{align*}
  \Delta^\alpha f \colon (x,h) \in \Delta\dom & \to V, &
  \Delta^\alpha f (x,h) & = \Delta_h^\alpha f(x)
\end{align*}
on the domain
\begin{equation} \label{eq:supplements:delta-dom}
  \Delta\dom := \left\{(x,h) : \, x \in \dom, \, x+h \in \dom \right\} \subset U \times U.
\end{equation}
Then, the Hölder norms can be equivalently expressed as
\begin{equation*}
  |f|_\CHND{\alpha}{\dom;V} 
  = \sup_{x \ne x+h \in D} \left\|\Delta_h^\alpha f \right\|_V
  = \left\|\Delta^\alpha f \right\|_\CND{0}{\Delta\dom;V}.
\end{equation*}
If $f = f(x,y)$ depends on multiple variables, we denote the partial finite difference operators by $\Delta_{x,h_x}^\alpha$ and $\Delta_{y,h_y}^\alpha$ defined by
\begin{align*}
  \Delta_{x,h_x}^0 f(x,y) & := f(x,y), &
  \Delta_{x,h_x}^\alpha f(x,y) & := \|h_x\|_U^{-\alpha} [f(x+h_x, y) - f(x,y)],
  \\
  \Delta_{y,h_y}^0 f(x,y) & := f(x,y), &
  \Delta_{y,h_y}^\beta f(x,y) & := \|h_y\|_U^{-\beta} [f(x, y+h_y) - f(x,y)],
\end{align*}
for $\alpha > 0$, and likewise
\begin{align*}
  \Delta_x^\alpha f(x,y,h_x) & = \Delta_{x,h_x}^\alpha f(x,y), &
  \Delta_y^\alpha f(x,y,h_y) & = \Delta_{y,h_y}^\alpha f(x,y).
\end{align*}
Then, the mixed Hölder norms is
\begin{align*}
  |f|_\CHHND{\alpha}{\beta}{\dom;V} 
  = \sup_{\substack{x\ne x+h_x \in D \\ y\ne y+h_y\in D}} \left\| \Delta_{x,h_x}^\alpha \Delta_{y,h_y}^\beta f(x,y) \right\|_V
  = \left\| \Delta_x^\alpha \Delta_y^\beta f \right\|_\CND{0}{\Delta\dom \times \Delta\dom; V}
\end{align*}
for all $\alpha, \, \beta \ge 0$ and likewise for all other Hölder semi-norms. 

In the following lemma, we summarize several useful properties of finite differences.

\begin{lemma} \label{lemma:supplements:fd-properties}

  Let $U, V$ and $W$ be three normed spaces, $\dom \subset U$ and $0 < \alpha, \, \beta \le 1$.

  \begin{enumerate}

    \item \emph{Product rule:} Let $f,g \colon \dom \to \real$. Then
    \begin{equation*}
      \Delta^\alpha_h [fg](x)
      = \left[\Delta^\alpha_h f(x)\right] g(x) + f(x+h) \left[\Delta^\alpha_h g(x)\right].
    \end{equation*}

    \item \emph{Chain rule:} Let $f:\dom \to V$ and $g: f(\dom) \to W$. Define
    \begin{equation*}
      \idelta_h (f,g)(x)
      := \int_0^1 f'(t g(x+h) + (1-t) g(x)) \, dt.
    \end{equation*}
    Then
    \begin{equation*}
      \Delta_h^\alpha (f \circ g)(x)
      = \idelta_h(f,g)(x) \Delta_h^\alpha g(x).
    \end{equation*}

  \end{enumerate}

\end{lemma}

\begin{proof}
\begin{enumerate}

  \item Plugging in the definitions, we have
  \begin{align*}
      \Delta^\alpha_h [fg](x)
      & = \|h\|_U^{-\alpha} \left[ f(x+h)g(x+h) - f(x)g(x) \right]
      \\
      & = \|h\|_U^{-\alpha} \left[ [f(x+h)-f(x)]g(x) + f(x+h)[g(x+h)-g(x)] \right]
      \\
      & = \left[\Delta^\alpha_h f(x)\right] g(x) + f(x+h) \left[\Delta^\alpha_h g(x)\right].
  \end{align*}

  \item Follows directly from the integral form of the Taylor remainder:
  \begin{align*}
      \Delta_h^\alpha (f \circ g)(x)
      & = \|h\|_U^{-\alpha} \left[ f(g(x+h)) - f(g(x)) \right]
      \\
      & = \|h\|_U^{-\alpha} \int_0^1 f'(tg(x+h) + (1-t)g(x)) \, dt [g(x+h) - g(x)]
      \\
      & = \idelta_h(f,g)(x) \Delta_h^\alpha g(x).
  \end{align*}

\end{enumerate}
\end{proof}

In the following lemma, we summarize several useful properties of Hölder spaces.

\begin{lemma} \label{lemma:supplements:holder-properties}
  Let $U$ and $V$ be two normed spaces, $\dom \subset U$ and $0 < \alpha, \, \beta \le 1$.
  \begin{enumerate}

    \item {Interpolation Inequality:} For any $f \in \CND{1}{\dom;V}$, we have
    \[
      \|f\|_\CHND{\alpha}{\dom;V}
      \le 2 \|f\|_\CND{0}{\dom;V}^{1-\alpha} \|f\|_\CHND{1}{\dom;V}^\alpha.
    \]

    \item Assume $\activation$ satisfies the growth and Lipschitz conditions $\|\activationp{x}\|_V \lesssim \|x\|_V$ and $\|\activationp{x} - \activationp{\px}\|_V \lesssim \|x-\px\|_V$. Then
    \[
      \|\activation \circ f\|_\CHND{\alpha}{\dom;V}
      \lesssim \|f\|_\CHND{\alpha}{\dom;V}.
    \]

    \item \label{item:dot-product:lemma:supplements:holder-properties} Let $V_1$ and $V_2$ be two normed spaces and $f, g: \dom \to V_1$. Let $\cdot: V_1 \times V_1 \to V_2$ be a distributive product that satisfies $\|u \cdot v\|_{V_2} \lesssim \|u\|_{V_1} \|v\|_{V_1}$. Then
    \[
      \|f \cdot g\|_\CHHND{\alpha}{\beta}{\dom;V_2} \lesssim \|f\|_\CHND{\alpha}{\dom;V_1} \|g\|_\CHND{\beta}{\dom;V_1}.
    \]

    \item \label{item:product:lemma:supplements:holder-properties} Let $V=\real$ and $f, g: \dom \times \dom \to \real$. Then
    \[
      \|f g\|_\CHHND{\alpha}{\beta}{\dom} \lesssim \|f\|_\CHHND{\alpha}{\beta}{\dom} \|g\|_\CHHND{\alpha}{\beta}{\dom}.
    \]
    
  \end{enumerate}
\end{lemma}

\begin{proof}

\begin{enumerate}

  \item The inequality follows directly from
  \begin{align*}
    \left|f\right|_\CHND{\alpha}{\dom;V}
    & = \sup_{x,\px \in \dom} \frac{\|f(x) - f(\px)\|_V}{\|x-\px\|_U^\alpha}
    \\
    & \le \sup_{x\ne\px \in \dom} \|f(x) - f(\px)\|_V^{1-\alpha} \sup_{x\ne\px \in \dom} \frac{\|f(x) - f(\px)\|_V^\alpha}{\|x-\px\|_U^\alpha}
    \\
    & \le 2 \|f\|_\CND{0}{\dom;V}^{1-\alpha} \|f\|_\CHND{1}{\dom;V}^\alpha.
  \end{align*}

  \item Follows from
  \begin{align*}
    \left|\activation \circ f \right|_\CHND{\alpha}{\dom;V}
    & = \sup_{x,\px \in \dom} \frac{\|\activation(f(x)) - \activation(f(\px))\|_V}{\|x-\px\|_U^\alpha}
    \\
    & \lesssim \sup_{x,\px \in \dom} \frac{\|f(x) - f(\px)\|_V^\alpha}{\|x-\px\|_U^\alpha}
    = \|f\|_\CHND{\alpha}{\dom;V}.
  \end{align*}
  and likewise for the $|\cdot|_\CND{0}{\dom;V}$ norm.

  \item Follows from 
  \begin{align*}
    |f \cdot g|_\CHHND{\alpha}{\beta}{\dom;V_2} 
    & = \sup_{x,\px, y, \py \in D} \frac{\left\|f(x) \cdot g(y) - f(\px) \cdot g(y) - f(x) \cdot g(\py) + f(\px) \cdot g(\py) \right\|_{V_2}}{\|x - \px\|_U^\alpha \|y - \py\|_U^\beta}
    \\
    & = \sup_{x,\px, y, \py \in D} \frac{\left\|[f(x) - f(\px)] \cdot [g(y) - g(\py)] \right\|_{V_2}}{\|x - \px\|_U^\alpha \|y - \py\|_U^\beta}
    \\
    & \lesssim \sup_{x,\px, y, \py \in D} \frac{\|f(x) - f(\px)\|_{V_1} \|g(y) - g(\py)\|_{V_1}}{\|x - \px\|_U^\alpha \|y - \py\|_U^\beta}
    \\
    & = |f|_\CHND{\alpha}{\dom;V_1} |g|_\CHND{\beta}{\dom;V_1}
  \end{align*}
  and analogous identities for the remaining semi norms $|fg|_\CHHND{0}{0}{\dom;V_2}$, $|fg|_\CHHND{\alpha}{0}{\dom;V_2}$, $|fg|_\CHHND{0}{\beta}{\dom;V_2}$.

  \item We only show the bound for $|\cdot|_\CHHND{\alpha}{\beta}{\dom}$. The other semi-norms follow analogously. Applying the product rule (Lemma \ref{lemma:supplements:fd-properties})
  \begin{equation*}
    \Delta_{x,h_x}^\alpha \left[f(x,y) g(x,y)\right]
    = \left[\Delta_{x,h_x}^\alpha f(x,y)\right] g(x,y) 
      + \ f(x+h_x,y) \left[\Delta_{x,h_x}^\alpha f(x,y)\right] 
  \end{equation*}
  and then analogously for $\Delta_{y,h_y}^\beta$
  \begin{multline*}
  \Delta_{y,h_y}^\beta \Delta_{x,h_x}^\alpha \left[f(x,y) g(x,y)\right]
  \\
  \begin{aligned}
    & = \Delta_{y,h_y}^\beta \left\{
        \left[\Delta_{x,h_x}^\alpha f(x,y)\right] g(x,y) 
        + \ f(x+h_x,y) \left[\Delta_{x,h_x}^\alpha f(x,y)\right] 
      \right\}
    \\
    & = \left[\Delta_{y,h_y}^\beta \Delta_{x,h_x}^\alpha f(x,y)\right] g(x,y) 
        + \left[\Delta_{x,h_x}^\alpha f(x,y+h_y)\right] \left[\Delta_{y,h_y}^\beta g(x,y) \right]
    \\
    & \quad
        + \left[\Delta_{y,h_y}^\beta f(x+h_x,y) \right] \left[\Delta_{x,h_x}^\alpha f(x,y)\right] 
        + f(x+h_x,y+h_y) \left[\Delta_{y,h_y}^\beta  \Delta_{x,h_x}^\alpha f(x,y)\right].
  \end{aligned}
  \end{multline*}
  Taking the supremum directly shows the result.

\end{enumerate}

\end{proof}

The following two lemmas contain chain rules for Hölder and mixed Hölder spaces.

\begin{lemma} \label{lemma:supplements:fd-composition}

  Let $\dom \subset U$ and $\dom_f \subset V$ be domains in normed spaces $U$, $V$ and $W$. Let $g \colon \dom \to \dom_f$ and $f \colon \dom_f \to W$. Let $0 < \alpha, \, \beta \le 1$. Then
  \begin{equation*}
    \left\|\Delta^\alpha (f \circ g)\right\|_\CND{0}{\Delta\dom; W}
    \le \|f'\|_\CHND{0}{\dom_f; L(V,W)} \|g\|_\CHND{\alpha}{\dom;V}
  \end{equation*}
  and
  \begin{multline*}
    \left\|\Delta^\alpha (f \circ g) - \Delta^\alpha (f \circ \pg) \right\|_\CND{0}{\Delta\dom; W}
    \\
    \begin{aligned}
      & \le \|f'\|_\CHND{1}{\dom_f; L(V,W)} \|g-\pg\|_\CND{0}{\dom;V} \|\pg\|_\CHND{\alpha}{\dom;V}
      \\
      & \quad+ \|f'\|_\CHND{0}{\dom_f; L(V,W)} \|g-\pg\|_\CHND{\alpha}{\dom;V},
      \\
      & \le 2 \|f'\|_\CHND{1}{\dom_f; L(V,W)} \|g-\pg\|_\CHND{\alpha}{\dom;V} \max\{1, \|\pg\|_\CHND{\alpha}{\dom;V}\},
    \end{aligned}
  \end{multline*}
  where $L(V,W)$ is the space of all linear maps $V \to W$ with induced operator norm.

\end{lemma}

\begin{proof}

Note that
\begin{equation*}
  \idelta_h (f,g)(x)
  := \int_0^1 f'(t g(x+h) + (1-t) g(x)) \, dt
\end{equation*}
takes values in the linear maps $L(V,W)$ and thus $\|\idelta_h(f,g)(x) v\|_W \le \|\idelta_h(f,g)(x)\|_{L(V,W)} \|v\|_V$, for all $v \in V$. Using the chain rule Lemma \ref{lemma:supplements:fd-properties}, it follows that
\begin{align*}
  \left\|\Delta_h^\alpha (f \circ g)(x)\right\|_W
  & = \left\|\idelta_h(f,g)(x) \Delta_h^\alpha g(x)\right\|_W
  \\
  & \le \left\|\idelta_h(f,g)(x) \right\|_{L(V,W)} \left\|\Delta_h^\alpha g(x)\right\|_V
\end{align*}
and
\begin{align*}
  \left\|\Delta_h^\alpha (f \circ g)(x) - \Delta_h^\alpha (f \circ \pg)(x)\right\|_W
  & = \left\|\idelta_h(f,g)(x) \Delta_h^\alpha g(x) - \idelta_h(f,\pg)(x) \Delta_h^\alpha \pg(x)\right\|_W
  \\
  & \le \left\|\idelta_h(f,g)(x) - \idelta_h(f,\pg)(x) \right\|_{L(V,W)} \left\|\Delta_h^\alpha g(x)\right\|_V
  \\
  & \quad + \left\|\idelta_h(f,\pg)(x)\right\|_{L(V,W)} \left\|\Delta_h^\alpha g(x) - \Delta_h^\alpha \pg(x) \right\|_V.
\end{align*}
Hence, the result follows from
\begin{equation} \label{eq:proof:lemma:supplements:fd-composition:1}
  \left\|\idelta_h(f,\pg)(x)\right\|_{L(V,W)}
  \le \|f'\|_\CND{0}{\dom_f; L(V,W)}
\end{equation}
and
\begin{multline} \label{eq:proof:lemma:supplements:fd-composition:2}
  \left\|\idelta_h(f,g)(x) - \idelta_h(f,\pg)(x)\right\|_{L(V,W)}
  \\
  \le \|f'\|_\CHND{1}{\dom_f; L(V,W)} \int_0^1 \left\|t (g-\pg)(x+h) + (1-t)(g-\pg)(x)\right\| \, dt
  \\
  \le \|f'\|_\CHND{1}{\dom_f; L(V,W)} \|g-\pg\|_\CND{0}{\dom; V},
\end{multline}
where we have used that unlike $\Delta_h^\alpha$, the integral $\idelta_h$ does not have an inverse $\|h\|_U^{-\alpha}$ factor.
  
\end{proof}

\begin{lemma} \label{lemma:supplements:fd-composition-2}

  Let $\dom \subset U$ and $\dom_f \subset V$ be domains in normed spaces $U$, $V$ and $W$. Let $g \colon \dom \to \dom_f$ and $f \colon \dom_f \to W$. Let $0 < \alpha, \, \beta \le 1$. Then
  \begin{multline*}
    \left\|\Delta^\alpha \Delta^\beta \left[f \circ g - f \circ \pg\right] \right\|_\CND{0}{\Delta\dom \times \Delta\dom; W}
    \\
    \le \|f\|_\CND{3}{\dom_f,W} \|g-\pg\|_\CHHND{\alpha}{\beta}{\dom;V} 
    \\
    \max\{1, \|g\|_\CHHND{\alpha}{\beta}{\dom;V}\}
    \max\{1, \|\pg\|_\CHHND{\alpha}{\beta}{\dom;V}\}.
  \end{multline*}

\end{lemma}

\begin{proof}

In the following, we fix $x$ and $y$, but only include it in the formulas if necessary, e.g. $f = f(x,y)$. By the chain rule Lemma \ref{lemma:supplements:fd-properties}, we have
\begin{align*}
  \Delta_{y,h_y}^\beta [f \circ g - f \circ \pg]
  & = \idelta_{y,h_y}(f,g) \Delta_{y,h_y}^\beta g - \idelta_{y,h_y}(f,\pg) \Delta_{y,h_y}^\beta \pg
  \\
  & = \left[\idelta_{y,h_y}(f,g) - \idelta_{y,h_y}(f,\pg) \right] \Delta_{y,h_y}^\beta g
  \\
  & \quad + \idelta_{y,h_y}(f,\pg) \left[\Delta_{y,h_y}^\beta g - \Delta_{y,h_y}^\beta \pg \right]
  \\
  & =: I + II.
\end{align*}
Applying the product rule Lemma \ref{lemma:supplements:fd-properties} to the first term yields
\begin{align*}
  \left\|\Delta_{x,h_x}^\alpha  I\right\|_W
  & = \left\|\left[\Delta_{x,h_x}^\alpha [\idelta_{y,h_y}(f,g)] - \Delta_{x,h_x}^\alpha [\idelta_{y,h_y}(f,\pg)] \right] \Delta_{y,h_y}^\beta g(x+h_x,y)\right.
  \\
  & \quad + \left.\left[\idelta_{y,h_y}(f,g) - \idelta_{y,h_y}(f,\pg) \right] \Delta_{x,h_x}^\alpha \Delta_{y,h_y}^\beta g\right\|_W
  \\
  & \le \left\|\left[\Delta_{x,h_x}^\alpha [\idelta_{y,h_y}(f,g)] - \Delta_{x,h_x}^\alpha [\idelta_{y,h_y}(f,\pg)] \right] \right\|_{L(V,W)} \left\|\Delta_{y,h_y}^\beta g(x+h_x,y)\right\|_W
  \\
  & \quad + \left\|\left[\idelta_{y,h_y}(f,g) - \idelta_{y,h_y}(f,\pg) \right] \right\|_{L(V,W)} \left\| \Delta_{x,h_x}^\alpha \Delta_{y,h_y}^\beta g\right\|_W.
\end{align*}
Likewise, applying the product Lemma rule \ref{lemma:supplements:fd-properties} to the second term yields
\begin{align*}
  \left\|\Delta_{x,h_x}^\alpha  II\right\|_W
  & = \left\|\Delta_{x,h_x}^\alpha \idelta_{y,h_y}(f,\pg) \left[\Delta_{y,h_y}^\beta g - \Delta_{y,h_y}^\beta \pg \right]\right\|_W
  \\
  & \quad + \left\|\idelta_{y,h_y}(f,\pg)(x+h_x,y) \left[\Delta_{x,h_x}^\alpha \Delta_{y,h_y}^\beta g - \Delta_{x,h_x}^\alpha \Delta_{y,h_y}^\beta \pg \right]\right\|_W
  \\
  & \le \left\|\Delta_{x,h_x}^\alpha \idelta_{y,h_y}(f,\pg) \right\|_{L(V,W)} \left\| \Delta_{y,h_y}^\beta g - \Delta_{y,h_y}^\beta \pg \right\|_W
  \\
  & \quad + \left\|\idelta_{y,h_y}(f,\pg)(x+h_x,y) \right\|_{L(V,W)} \left\| \Delta_{x,h_x}^\alpha \Delta_{y,h_y}^\beta g - \Delta_{x,h_x}^\alpha \Delta_{y,h_y}^\beta \pg \right\|_W.
\end{align*}
All terms involving only $g$ and $\pg$ can easily be upper bounded by $\|g\|_\CHHND{\alpha}{\beta}{\dom;V}$, $\|\pg\|_\CHHND{\alpha}{\beta}{\dom;V}$ or $\|g-\pg\|_\CHHND{\alpha}{\beta}{\dom;V}$. The terms 
\begin{align*}
  \left\|\idelta_{y,h_y}(f,\pg)(x+h_x,y) \right\|_{L(V,W)} 
  & \le \|f'\|_\CND{0}{\dom_f;L(V,W)}& 
  \\
  \left\|\left[\idelta_{y,h_y}(f,g) - \idelta_{y,h_y}(f,\pg) \right] \right\|_{L(V,W)}
  & \le \|f'\|_\CHND{1}{\dom_f;L(V,W)} \|g-\pg\|_\CND{0}{D;V} 
\end{align*}
are bounded by \eqref{eq:proof:lemma:supplements:fd-composition:1} and \eqref{eq:proof:lemma:supplements:fd-composition:2} in the proof of Lemma \ref{lemma:supplements:fd-composition}. For the remaining terms, define 
\begin{equation*}
  G(x) := tg(x, y+h_y) + (1-t)g(x,y)
\end{equation*}
\newcommand{\pG}{\bar{G}}
and likewise $\pG$. Then 
\begin{align*}
  \|G\|_\CHND{\alpha}{D,V} & \lesssim \|g\|_\CHHND{\alpha}{\beta}{D,V}, &
  \|G-\pG\|_\CHND{\alpha}{D,V} & \lesssim \|g-\pg\|_\CHHND{\alpha}{\beta}{D,V}.
\end{align*}
Thus, by Lemma \ref{lemma:supplements:fd-composition}, we have
\begin{align*}
  \left\|\Delta_{x,h_x}^\alpha \left[\idelta_{y,h_y}(f,g) \right] \right\|_{L(V,W)}
  & = \left\|\int_0^1 \Delta_{x,h_x}^\alpha (f' \circ G) \, dt \right\|_{L(V,W)}
\\
  & \le \|f''\|_\CHND{0}{\dom_f; L(V,L(V,W))} \|g\|_\CHHND{\alpha}{\beta}{\dom;V}
\end{align*}
and
\begin{multline*}
  \left\|\Delta_{x,h_x}^\alpha \left[\idelta_{y,h_y}(f,g) - \idelta_{y,h_y}(f,\pg) \right] \right\|_{L(V,W)}
  \\
  \begin{aligned}
    & = \left\|\int_0^1 \Delta_{x,h_x}^\alpha \left[f' \circ G - f' \circ \pG \right] \, dt \right\|_{L(V,W)}
    \\
    & \le 2 \|f''\|_\CHND{1}{\dom_f; L(V,L(V,W))} \|g-\pg\|_\CHHND{\alpha}{\beta}{\dom;V} \max \{1, \|\pg\|_\CHHND{\alpha}{\beta}{\dom;V}\}.
  \end{aligned}
\end{multline*}
Combining all inequalities yields the proof.
  
\end{proof}

\subsection{Concentration}
\label{sec:supplements:concentration}

In this section, we recall the definition of Orlicz norms, some basic properties and the chaining concentration inequalities we use to show that the empirical NTK is close to the NTK. 

\begin{definition}
  For random variable $X$, we define the \emph{sub-gaussian} and \emph{sub-exponential} \emph{norms} by
  \begin{align*}
    \|X\|_{\psi_2} & = \inf \left\{ t>0 : \, \E{\exp(X^2/t^2)} \le 2 \right\},
    \\
    \|X\|_{\psi_1} & = \inf \left\{ t>0 : \, \E{\exp(|X|/t)} \le 2 \right\}.
  \end{align*}
  
\end{definition}

\begin{lemma} \label{lemma:supplements:psi-norm-lipschitz}
  Assume that $\activation$ satisfies the growth and Lipschitz conditions 
  \begin{align*}
    |\activation(x)| & \le G |x|, &
    |\activation(x) - \activation(y)| & \le L |x-y|
  \end{align*}
  for all $x, y \in \real$ and let $X$, $Y$ be two sub-gaussian random variables. Then
  \begin{align*}
    \|\activationp{X}\|_{\psi_2} & \lesssim G \|X\|_{\psi_2}, &
    \|\activationp{X} - \activationp{Y}\|_{\psi_2} & \lesssim L \|X-Y\|_{\psi_2}.
  \end{align*}
\end{lemma}

\begin{proof}

For two random variables $X$ and $Y$ with $X^2 \le Y^2$ almost surely, we have
\begin{multline*}
  \|X\|_{\psi_2} 
  = \inf \left\{ t>0 : \, \E{\exp(X^2/t^2)} \le 2 \right\}
  \\
  \le \inf \left\{ t>0 : \, \E{\exp(Y^2/t^2)} \le 2 \right\}
  = \|Y\|_{\psi_2}.
\end{multline*}
Thus, the result follows directly form
  \begin{align*}
    \activation(X)^2 & \le G^2 X^2, &
    [\activation(x) - \activation(y)]^2 & \le L^2 [x-y]^2.
  \end{align*}
  
\end{proof}

\begin{lemma} \label{lemma:supplements:psi-norm-product}
  Let $X$ and $Y$ be two sub-gaussian random variables. Then
  \begin{equation*}
    \|XY\|_{\psi_1} \le \|X\|_{\psi_2} \|Y\|_{\psi_2}.
  \end{equation*}
\end{lemma}

\begin{proof}

Let
\begin{equation*}
  t 
  = \|X\|_{\psi_2}^{1/2} \|Y\|_{\psi_2}^{1/2}
  = \left\|\left(\frac{\|Y\|_{\psi_2}}{\|X\|_{\psi_2}}\right)^{1/2} X\right\|_{\psi_2}
  = \left\|\left(\frac{\|X\|_{\psi_2}}{\|Y\|_{\psi_2}}\right)^{1/2} Y\right\|_{\psi_2}.
\end{equation*}
Ignoring a simple $\epsilon$ perturbation, we assume that the infima in the definition of the $\|X\|_{\psi_2}$ and $\|Y\|_{\psi_2}$ norms are attained. Then
\begin{align*}
  \exp\left(\frac{\|Y\|_{\psi_2}}{\|X\|_{\psi_2}} \frac{X^2}{t^2} \right) & \le 2, &
  \exp\left(\frac{\|X\|_{\psi_2}}{\|Y\|_{\psi_2}} \frac{Y^2}{t^2} \right) & \le 2.
\end{align*}
Thus, Young's inequality implies
\begin{multline*}
  \exp\left(\frac{|XY|}{t}\right)
  \le \exp\left(\frac{1}{2}\frac{\|Y\|_{\psi_2}}{\|X\|_{\psi_2}} \frac{X^2}{t^2} + \frac{1}{2}\frac{\|X\|_{\psi_2}}{\|Y\|_{\psi_2}} \frac{Y^2}{t^2}\right)
  \\
  \le \exp\left(\frac{\|Y\|_{\psi_2}}{\|X\|_{\psi_2}} \frac{X^2}{t^2} + \frac{\|X\|_{\psi_2}}{\|Y\|_{\psi_2}} \frac{Y^2}{t^2}\right)^{1/2}
  \le \sqrt{2} \sqrt{2}
  \le 2.
\end{multline*}
Hence
\begin{equation*}
  \|XY\|_{\psi_1} \le t \le \|X\|_{\psi_2} \|Y\|_{\psi_2}.
\end{equation*}
  
\end{proof}

\begin{theorem}[{\cite[Theorem 3.5]{Dirksen2015}}] \label{th:supplements:chaining}
  
  Let $\mathcal{X}$ be a normed linear space. Assume the $\mathcal{X}$ valued separable random process $(X_t)_{t \in T}$,  has a mixed tail, with respect to some semi-metrics $d_1$ and $d_2$ on $T$, i.e.
  \begin{equation*}
    \pr{ \|X_t - X_s\| \ge \sqrt{u} d_2(t,s) + u d_1(t,s) }
    \le 2 e^{-u}
  \end{equation*}
  for all $s, \, t \in T$ and $u \ge 0$. Set
  \begin{align*}
    \gamma_\alpha(T, d_i) & := \inf_{\mathcal{T}} \sup_{t \in T} \sum_{n=0}^\infty 2^{n/a} d(t, T_n), &
    \alpha & \in \{0, 1\},
    \\
    \Delta_d(T) & := \sup_{s,t \in T} d(s,t),
  \end{align*}
  where the infimum is taken over all admissible sequences $T_n \subset T$ with $|T_0| = 1$ and $|T_n| \le 2^{2^n}$. Then for any $t_0 \in T$
  \begin{equation*}
    \pr{
      \sup_{t \in T} \|X_t - X_{t_0}\| 
      \ge C \left[ \gamma_2(T, d_2) + \gamma_1(T, d_1) + \sqrt{u} \Delta_{d_2}(T) + u \Delta_{d_1}(T)\right]
    }
    \le e^{-u}.
  \end{equation*}

\end{theorem}

\begin{remark}

  \cite[Theorem 3.5]{Dirksen2015} assumes that $T$ is finite. Using separability and monotone convergence, this can be extended to infinite $T$ by standard arguments.
  
\end{remark}

\begin{lemma} \label{lemma:supplements:gamma-functional}
  Let $0 \le \alpha \le 1$ and $\dom \subset \real^d$ be as set of Euclidean norm $|\cdot|$-diameter smaller than $R \ge 1$. Then
  \begin{align*}
    \gamma_1(\dom, |\cdot|^\alpha)
    & \lesssim \frac{3 \alpha + 1}{\alpha} R^{1+\alpha} d, &
    \gamma_2(\dom, |\cdot|^\alpha)
    & \lesssim \left(\frac{3^\alpha}{4 \alpha}\right)^{1/2}  R^{\alpha/2} d^{1/2}.
  \end{align*}
\end{lemma}

\begin{proof}

Let $N(\dom, |\cdot|^\alpha, u)$ be the covering number of $D$, i.e. the smallest number of $u$-balls in the metric $|\cdot|^\alpha$ necessary to cover $\dom$. It is well known (e.g. \cite[(2.3)]{Dirksen2015}) that
\begin{equation*}
  \gamma_i(\dom, |\cdot|^\alpha) 
  \lesssim \int_0^\infty \left[\log N(\dom, |\cdot|^\alpha, u) \right]^{1/i} \, du
  \lesssim \int_0^{R^\alpha} \left[\log N(\dom, |\cdot|^\alpha, u) \right]^{1/i} \, du,
\end{equation*}
where in the last step we have used that $N(\dom, |\cdot|^\alpha, u) = 1$ for $u \ge R^\alpha$ and thus its logarithm is zero. Since every $u$-cover in the $|\cdot|$ norm is a $u^\alpha$ cover in the $|\cdot|^\alpha$ metric, the covering numbers can be estimated by
\begin{equation*}
  N(\dom, |\cdot|^\alpha, u)
  = N(\dom, |\cdot|, u^{1/\alpha}) 
  \le \left(\frac{3R}{u^{1/\alpha}}\right)^d
  = \left(\frac{(3R)^\alpha}{u}\right)^{d/\alpha},
\end{equation*}
see e.g. \cite{Vershynin2018}. Hence
\begin{multline*}
  \gamma_1(\dom, |\cdot|^\alpha)
  \lesssim \int_0^{R^\alpha} \log \left(\frac{(3R)^\alpha}{u}\right)^{d/\alpha} \, du
  = \frac{d}{\alpha} \int_0^{R^\alpha} \alpha \log (3R) - \log u \, du
  \\
  \le \frac{d}{\alpha} \left[3 \alpha R^{1+\alpha} - R^\alpha \log R^\alpha + R^\alpha \right]
  \le \frac{d}{\alpha} (3 \alpha + 1) R^{1+\alpha}
\end{multline*}
and using $\log x \le x-1 \le x$
\begin{multline*}
  \gamma_2(\dom, |\cdot|^\alpha)
  \lesssim \left(\frac{d}{\alpha}\right)^{1/2} \int_0^{R^\alpha} \left[ \log \frac{(3R)^\alpha}{u} \right]^{1/2} \, du
  \lesssim \left(\frac{d}{\alpha}\right)^{1/2} \int_0^{R^\alpha} \left[ \frac{(3R)^\alpha}{u} \right]^{1/2} \, du
  \\
  \lesssim \left(\frac{3^\alpha d R^\alpha}{\alpha}\right)^{1/2} \int_0^{R^\alpha} \left[ \frac{1}{u} \right]^{1/2} \, du
  \lesssim \left(\frac{3^\alpha d}{4 \alpha}\right)^{1/2}  R^{\alpha/2}.
\end{multline*}

\end{proof}

The following is a rewrite of the chaining inequality \cite[Theorem 3.5]{Dirksen2015} or Theorem \ref{th:supplements:chaining}, that is compatible with the terminology used in the NTK concentration proof.

\begin{corollary} \label{corollary:supplements:chaining}
  
  For $j \in [N]$, let $(X_{j,t})_{t \in \dom}$ be real valued independent stochastic processes on some domain $\dom$ with radius $\lesssim 1$. Assume that the map $t \to X_{j,t}$ with values in the Orlicz space $L_{\psi_1}$ is Hölder continuous
  \begin{equation*}
    \|X_{j,\cdot}\|_\CHND{\alpha}{\dom;\psi_1} \le L.
  \end{equation*}
  Then
  \begin{equation*}
    \pr{
      \sup_{t \in T} \left\|\frac{1}{N} \sum_{j=1}^N X_{j,t}- \E{X_{j,t}}\right\| 
      \ge C L \left[ \left(\frac{d}{N}\right)^{1/2} + \frac{d}{N} + \left(\frac{u}{N}\right)^{1/2} + \frac{u}{N} \right]
    }
    \le e^{-u}.
  \end{equation*}

\end{corollary}

\begin{proof}

We show the result with Theorem \ref{th:supplements:chaining} for the process
\begin{equation*}
  Y_t := \frac{1}{N} \sum_{j=1}^N X_{j,t}- \E{X_{j,t}}.
\end{equation*}
We first show that it has mixed tail. For all $s, t \in \dom$, we have
\begin{equation*}
  \|X_{j,t} - X_{j,s}\|_{\psi_1} \le L |s-t|^\alpha.
\end{equation*}
Hence, Bernstein's inequality implies
\begin{multline*}
  \pr{|Y_t - Y_s| \ge \tau}
  = \pr{
    \left| \frac{1}{N} \sum_{j=1}^N [X_{j,t} - X_{j,s}] - \E{X_{j,t} - X_{j,s}} \right| \ge \tau
  }
  \\
  \le 2 \exp\left(
    -c N \min\left\{\frac{\tau^2}{L^2 |t-s|^{2 \alpha}}, \frac{\tau}{L |t-s|^\alpha}\right\}
  \right).
\end{multline*}
An elementary computation shows that
\begin{align*}
  u & := cN \min\left\{\frac{\tau^2}{L^2|t-s|^2}, \, \frac{\tau}{L|t-s|}\right\} & 
  & \Rightarrow & 
  \tau & = L|t-s|^\alpha \max\left\{\sqrt{\frac{u}{cN}}, \frac{u}{cN}\right\}
\end{align*}
and thus
\begin{equation} \label{eq:proof:1:corollary:supplements:chaining}
  \pr{
    | Y_t - Y_s| \ge L|t-s|^\alpha \max\left\{\sqrt{\frac{u}{cN}}, \frac{u}{cN}\right\}
  }
  \le 2 \exp(-u).
\end{equation}
I.e. the centered process $Y_t$ has mixed tail with
\begin{align*}
  d_i(t,s) := (cN)^{-1/i} L |t-s|^\alpha,
\end{align*}
for $i =1,2$, which are metrics because $\alpha \le 1$. Moreover the $\gamma_i$-functional are linear in scaling
\begin{equation*}
  \gamma_i(\dom, d_i) 
  = (cN)^{-1/i} L \gamma_i(\dom, |\cdot|^\alpha) 
\end{equation*}
and thus by Lemma \eqref{lemma:supplements:gamma-functional}
\begin{align*}
  \gamma_1(\dom, |\cdot|^\alpha)
  & \lesssim L \frac{d}{N}, &
  \gamma_2(\dom, |\cdot|^\alpha)
  & \lesssim  L \left(\frac{d}{N}\right)^{1/2}.
\end{align*}
Thus, by chaining Theorem \ref{th:supplements:chaining} we have
\begin{equation*}
  \pr{
    \sup_{t \in T} \|Y_t - Y_{t_0}\| 
    \ge C L \left[ \left(\frac{d}{N}\right)^{1/2} + \frac{d}{N} + \left(\frac{u}{N}\right)^{1/2} + \frac{u}{N} \right]
  }
  \le e^{-u},
\end{equation*}
which directly yields the corollary with $\sup_{t \in \dom} \|Y_t\| \le \sup_{t \in \dom} \|Y_t - Y_{t_0}\| + \|Y_{t_0}\|$ and \eqref{eq:proof:1:corollary:supplements:chaining}.

\end{proof}

\subsection{Hermite Polynomials}

Hermite polynomials are defined by
\begin{equation*}
  H_n(x) := (-1)^n e^{x^2/2} \frac{d^n}{dx^n} e^{-x^2/2}
\end{equation*}
and orthogonal with respect to the Gaussian weighted scalar product
\begin{equation*}
  \dualp{f, g}_N 
  := \EE{u\sim \gaussian{0,1}}{f(u) g(u)} 
  = \frac{1}{\sqrt{2\pi}} \int_\real f(u) g(u) e^{-x^2/2} \, du.
\end{equation*}

\begin{lemma} \label{lemma:supplements:hermite:properties}
  \begin{enumerate}

    \item \emph{Normalization:}
    \begin{equation*}
      \dualp{H_n, H_m}_N = n! \, \delta_{nm}.
    \end{equation*}

    \item {Derivatives:} Let $f: \real \to \real$ be $k$ times continuously differentiable so that all derivatives smaller or equal to $k$ have at most polynomial growth for $x \to \pm \infty$. Then
    \begin{equation*}
      \dualp{f, H_n}
      = \dualp{f^{(k)}, H_{n-k}}_N.
    \end{equation*}

  \end{enumerate}
\end{lemma}

\begin{proof}
  
The normalization is well known, we only show the formula for the derivative. By the growth condition, we have $\left| f^{(k)}(x) \frac{d^{n-k-1}}{d x^{n-k-1}} e^{-x^2/2} \right| \to 0$ for $x \to \pm \infty$. Thus, in the integration by parts formula below all boundary terms vanish and we have
\begin{align*}
  \dualp{f, H_n}
  & = \frac{1}{\sqrt{2\pi}} \int_\real f(u) H_n(u) e^{-x^2/2} \, du.
  \\
  & = \frac{1}{\sqrt{2\pi}} \int_\real f(u) \left[(-1)^n e^{x^2/2} \frac{d^n}{dx^n} e^{-x^2/2} \right] e^{-x^2/2} \, du.
  \\
  & = \frac{1}{\sqrt{2\pi}} (-1)^n \int_\real f(u) \frac{d^n}{dx^n} e^{-x^2/2} \, du.
  \\
  & = \frac{1}{\sqrt{2\pi}} (-1)^{n-k} \int_\real f^{(k)}(u) \frac{d^{n-k}}{dx^{n-k}} e^{-x^2/2} \, du.
  \\
  & = \frac{1}{\sqrt{2\pi}} \int_\real f^{(k)}(u) \left[(-1)^{n-k} e^{x^2/2} \frac{d^{n-k}}{dx^{n-k}} e^{-x^2/2} \right] e^{-x^2/2} \, du.
  \\
  & = \dualp{f^{(k)}, H_{n-k}}_N.
\end{align*}

\end{proof}

\begin{theorem}[Mehler's theorem] \label{th:supplements:hermite:mehler}
  Let 
  \begin{equation*}
    A = \begin{bmatrix}
      1 & \rho \\
      \rho & 1
    \end{bmatrix}.
  \end{equation*}
  Then the multi- and uni-variate normal density functions satisfy
  \begin{equation*}
    \pdf_{\gaussian{0,A}}
    = \sum_{k=0}^\infty H_k(u) H_k(v) \frac{\rho^k}{k!} \pdf_{\gaussian{0,1}}(u) \pdf_{\gaussian{0,1}}(v).
  \end{equation*}
  
\end{theorem}

\begin{proof}
See \cite{WithersNadarajah2010} for Mehler's theorem in the form stated here.
\end{proof}

\subsection{Sobolev Spaces on the Sphere}

\subsubsection{Definition and Properties}
\label{sec:sobolev:norms}

We use two alternative characterizations of Sobolev spaces on the sphere. The first is based on spherical harmonics, which are also eigenfunctions of the NTK and thus establishes connections to the available NTK literature. Second, we consider Sobolev Slobodeckij type norms, which are structurally similar to Hölder norms and allow connections to the perturbation analysis in this paper.

The spherical harmonics
\begin{align*}
  & Y_\ell^j, &
  \ell & = 0, 1, 2, \dots, &
  1 & \le j \le \nu(\ell)
\end{align*}
of degree $\ell$ and order $j$ are an orthonormal basis on the sphere $L_2(\Sd)$, comparable to Fourier bases for periodic functions. For any $f \in L_2(\Sd)$, we denote by $\hat{f}_{\ell j} = \dualp{f, Y_\ell^j}$ the corresponding basis coefficient. The Sobolev space $H^\alpha(\Sd)$ consists of all function for which the norm
\begin{equation*}
  \|f\|_{H^\alpha(\Sd)}^2 
  = \sum_{\ell=0}^\infty \sum_{j=1}^{\nu(\ell)}
    \left(1 + \ell^{1/2} (\ell+d-2)^{1/2}\right)^{2\alpha} \left|\hat{f}_{\ell j}\right|^2 
\end{equation*}
is finite. We write $H^\alpha = H^\alpha(\Sd)$ if the domain is understood from context. Since the constants in this paper are dimension dependent, we simplify this to the equivalent norm
\begin{equation} \label{eq:supplements:sobolev-spherical-harmonics}
  \|f\|_{H^\alpha(\Sd)}^2 
  = \sum_{\ell=0}^\infty \sum_{j=1}^{\nu(\ell)}
    \left( 1 + \ell \right)^{2\alpha} \left|\hat{f}_{\ell j}\right|^2.
\end{equation}
Another equivalent norm, similar to Sobolev-Slobodeckij norms, is given in \cite[Proposition 1.4]{BarceloLuquePerez2020} and defined as follows for the case $0 < \alpha < 2$. For the spherical cap centered at $x \in \Sd$ and angle $t \in (0, \pi)$ given by
\begin{equation*}
  C(x, t) := \left\{y \in \Sd : \, x \cdot y \ge \cos t\right\}
\end{equation*}
set
\newcommand{\aint}{\mathop{\ooalign{$\int$\cr$-$}}}
\begin{equation*}
  A_t(f)(x) := \dashint_{C(x,t)} f(\tau) \, d\tau.
\end{equation*}
With
\begin{equation*}
  S_\alpha(f)^2(x) := \int_0^\pi \left| A_t f(x) - f(x) \right|^2 t^{-2\alpha-1} \, dt
\end{equation*}
the Sobolev norm on the sphere is equivalent to
\begin{equation} \label{eq:supplements:sobolev:slobodeckij}
  \|f\|_{H^{\alpha}(\Sd)} \sim \left\|S_\alpha(f)\right\|_{L_2(\Sd)}.
\end{equation}
Using the definition \eqref{eq:supplements:sobolev-spherical-harmonics} for $a < b < c$, the interpolation inequality
\begin{align} \label{eq:supplements:sobolev-interpolation-inequality}
  \|\cdot\|_{H^b(\Sd)} & \lesssim \|\cdot\|_{H^a(\Sd)}^{\frac{c-b}{c-a}} \|\cdot\|_{H^c(\Sd)}^{\frac{b-a}{c-a}}, &
  \dualp{\cdot, \cdot}_{-\sm} & \lesssim \|\cdot\|_{-3\sm} \|\cdot\|_\sm,
\end{align}
follows directly from Cauchy-Schwarz. Moreover, we have the following embedding.

\begin{lemma} \label{lemma:supplements:sobolev-holder}

  Let $0 < \alpha < 1$. Then for any $\epsilon > 0$ with $\alpha+\epsilon \le 1$, we have
  \begin{equation*}
    \|\cdot\|_{H^\alpha(\Sd)} \lesssim \|\cdot\|_\CHND{\alpha+\epsilon}{\Sd}.
  \end{equation*}
  
\end{lemma}

\begin{proof}
  The proof is standard and similar to Lemma \ref{lemma:supplements:kernel-bound}.
\end{proof}

\subsubsection{Kernel Bounds}

In this section, we provide bounds for the kernel integral
\begin{equation*}
  \dualp{f,g}_k := \iint_{\dom \times \dom} f(x) k(x,y) g(y) \, dx \, dy
\end{equation*}
on the sphere $\dom = \Sd$ in Sobolev norms on the sphere. Clearly, for $0 \le \alpha, \beta < 2$, we have
\begin{equation*}
  \dualp{f,g}_k 
  \le \|f\|_{H^{-\alpha}} \left\|\int_\dom k(\cdot, y) g(y) \, dy \right\|_{H^\alpha}
  \le \|f\|_{H^{-\alpha}} \|k\|_{H^\alpha \leftarrow H^{-\beta}}\|g\|_{H^{-\beta}},
\end{equation*}
where the norm of $k$ is the induced operator norm. While the norms for $f$ and $g$ are the ones used in the convergence analysis, concentration and perturbation results for $k$ are computed in mixed Hölder norms. We show in this section, that these bound the operator norm.

Indeed, $\dualp{f,g}_k$ is a bilinear form on $f$ and $g$ and thus is bounded by the tensor product norms
\begin{equation*}
  \dualp{f,g}_k
  \le \|f \otimes g\|_{(H^\alpha \otimes H^\beta)'} \|k\|_{H^\alpha \otimes H^\beta}
  \le \|f\|_{H^{-\alpha}} \|g\|_{H^{-\beta}} \|k\|_{H^\alpha \otimes H^\beta},
\end{equation*}
where $\cdot '$ denotes the dual norm. The $H^\alpha \otimes H^\beta$ norm contains mixed smoothness and with Sobolev-Slobodeckij type definition \eqref{eq:supplements:sobolev:slobodeckij} is easily bounded by corresponding mixed Hölder regularity. In order to avoid rigorous characterization of tensor product norms on the sphere, the following lemma shows the required bounds directly.

\begin{lemma} \label{lemma:supplements:kernel-bound}

  Let $0 < \alpha, \beta < 1$. Then for any $\epsilon > 0$ with $\alpha+\epsilon \le 1$ and $\beta+\epsilon<1$, we have
  \begin{equation*}
    \iint_{\dom \times \dom} f(x) k(x,y) g(y) \, dx \, dy
    \le \|f\|_{H^{-\alpha}(\Sd)} \|g\|_{H^{-\beta}(\Sd)} \|k\|_\CHHND{\alpha+\epsilon}{\beta+\epsilon}{\Sd}.
  \end{equation*}
  
\end{lemma}

\begin{proof}

Since for any $u$, $v$
\begin{multline*}
  \int u(x) v(x) \, dx
  = \int u(x) \frac{v(x)}{\|v\|_{H^\alpha}} \, dx \, \|v\|_{H^\alpha}
  \\
  \le \sup_{\|w\|_{H^\alpha} \le 1} \int u(x) w \, dx \, \|v\|_{H^\alpha}
  \le \|u\|_{H^{-\alpha}} \|v\|_{H^\alpha},
\end{multline*}
with $\dom = \Sd$ we have
\begin{equation*}
  \dualp{f,g}_k = \iint_{\dom \times \dom} f(x) k(x,y) g(y) \, dx \, dy
  \le \|f\|_{H^{-\alpha}} \left\|\int_\dom k(\cdot, y) g(y) \right\|_{H^\alpha}
\end{equation*}
so that it remains to estimate the last term. Plugging in definition \eqref{eq:supplements:sobolev:slobodeckij} of the Sobolev norm, we obtain
\begin{equation*}
  \left\|\int_\dom k(\cdot, y) g(y) \right\|_{H^\alpha}^2
  = \int_D \int_0^\pi \left|(A_t^x-I) \left(\int_D k(\cdot, y) g(y) \, dy\right)(x) \right|^2 t^{-2\alpha-1} \, dt \, dx,
\end{equation*}
where $A_t^x$ is the average in \eqref{eq:supplements:sobolev:slobodeckij} applied to the $x$ variable only and $I$ the identity. Swapping the inner integral with the one inside the definition of $A_t^x$, we estimate
\begin{align*}
  \left\|\int_\dom k(\cdot, y) g(y) \right\|_{H^\alpha}^2
  & = \int_D \int_0^\pi \left|\int_D \left[(A_t^x-I) (k(\cdot, y))(x)\right] g(y) \, dy \right|^2 t^{-2\alpha-1} \, dt \, dx,
  \\
  & \le \int_D \int_0^\pi \left\|(A_t^x-I) (k(\cdot, y))(x)\right\|_{H^\beta}^2 \|g\|_{H^{-\beta}}^2 t^{-2\alpha-1} \, dt \, dx,
  \\
  & = \iint_{\dom\times\dom} \iint_0^\pi \left|(A_s^y-I)(A_t^x-I) (k)(x,y) \right|^2  t^{-2\alpha-1} s^{-2\beta-1} \, dst \, dxy \, \|g\|_{H^{-\beta}}^2.
\end{align*}
Plugging in the definition of the averages $A_s^y$ and $A_t^x$, the integrand is estimated by the mixed Hölder norm
\begin{align*}
  \left|(A_s^y-I)(A_t^x-I) (k)(x,y) \right|
  & = \dashint_{C(y,s)} \dashint_{C(x,t)} \left| k(\tau, \sigma) - k(x, \sigma) - k(\tau, y) + k(x,y) \right| \, d\tau\sigma
  \\
  & \le \dashint_{C(y,s)} \dashint_{C(x,t)} |x-\tau|^{\alpha + \epsilon} |y-\sigma|^{\beta + \epsilon} \|k\|_\CHHN{\alpha+\epsilon}{\beta+\epsilon} \, d\tau\sigma.
\end{align*}
The difference $|x-\tau|$, and likewise $|y-\sigma|$, is bounded by the angle of the cap $C(x,t)$. Indeed
\begin{equation*}
  |x-\tau|^2 
  = |x|^2 + |\tau|^2 - 2 \dualp{x, \tau}
  = 2(1 - \dualp{x, \tau})
  \le 2(1 - \cos t)
  \lesssim t^2
\end{equation*}
for $t \le T$ for some $T \ge 0$. Since for all other $t$ the difference $|x-\tau| \le 2$ is bounded, we obtain
\begin{align*}
  |x-\tau| & \lesssim \min\{t, T\}, &
  |y-\sigma| & \lesssim \min\{s, T\}.
\end{align*}
It follows that
\begin{align*}
  \left| (A_s^y-I)(A_t^x-I) (k)(x,y) \right|
  & \lesssim  \min\{t, T\}^{\alpha + \epsilon} \min\{s, T\}^{\beta + \epsilon} \|k\|_\CHHN{\alpha+\epsilon}{\beta+\epsilon}.
\end{align*}
Putting all estimates together, we find that
\begin{multline*}
  \dualp{f,g}_k
  \lesssim \|f\|_{H^{-\alpha}} \|g\|_{H^{-\beta}} \|k\|_\CHHN{\alpha+\epsilon}{\beta+\epsilon} 
  \\
  \cdot \left[ \iint_{\dom\times\dom} \iint_0^\pi \left[ \min\{t, T\}^{\alpha + \epsilon} \min\{s, T\}^{\beta + \epsilon} \right]^2 t^{-2\alpha-1} s^{-2\beta-1} \, dst \, dxy \right]^{\frac{1}{2}}.
\end{multline*}
Since the integral is bounded, we conclude that
\begin{equation*}
  \dualp{f,g}_k
  \lesssim \|f\|_{H^{-\alpha}} \|g\|_{H^{-\beta}} \|k\|_\CHHN{\alpha+\epsilon}{\beta+\epsilon}.
\end{equation*}
  
\end{proof}

\subsubsection{NTK on the Sphere}
\label{sec:ntk-sphere}

This section fills in the proofs for Section \ref{sec:coercivity}. Recall that we denote the normal NTK used in \cite{BiettiMairal2019,GeifmanYadavKastenGalunJacobsRonen2020,ChenXu2021} by
\begin{equation*}
  \NTK(x,y) 
  = \lim_{\text{width}\to\infty} \sum_\w \partial_\w f^{L+1}(x) \partial_\w f^{L+1}(y),
\end{equation*}
whereas the NTK $\ntk(x,y)$ used in this paper confines the sum to $|\w| = L-1$, i.e. the second but last layer, see Section \ref{sec:coercivity}. We first show that the reproducing kernel Hilbert space (RKHS) of the NTK is a Sobolev space.

\begin{lemma} \label{lemma:supplements:sobolev:ntk-sobolev}
  Let $\NTK(x,y)$ be the neural tangent kernel for a fully connected neural network on the sphere $\Sd$ with bias and $\relu$ activation. Then the corresponding RKHS $H_\NTK$ is the Sobolev space $H^{d/2}(\Sd)$ with equivalent norms
  \begin{equation*}
    \|\cdot\|_{H_\NTK} \sim \|\cdot\|_{H^{d/2}}.
  \end{equation*}
  
\end{lemma}

\begin{proof}

By \cite[Theorem 1]{ChenXu2021} the RKHS $H_{\NTK}$ is the same as the RKHS $H_{Lap}$ of the Laplacian kernel 
\begin{equation*}
  k(x,y) = e^{- \|x-y\|}.
\end{equation*}
An inspection of their proof reveals that these spaces have equivalent norms. By \cite[Theorem 2]{GeifmanYadavKastenGalunJacobsRonen2020}, the Laplace kernel has the same eigenfunctions as the NTK (both are spherical harmonics) and eigenvalues
\begin{align*}
  \ell^{-d} & \lesssim \lambda_{\ell,j} \lesssim \ell^{-d}, & 
  \ell & \ge \ell_0, &
  j = 1, \dots, \nu(\ell),
\end{align*}
for some $\ell_0 \ge 0$, whereas the remaining eigenvalues are strictly positive. By rearranging the constants, this implies 
\begin{align*}
  (\ell+1)^{-d} & \lesssim \lambda_{\ell,j} \lesssim (\ell+1)^{-d}, & 
  \ell & \ge 0, &
  j = 1, \dots, \nu(\ell),
\end{align*}
for all eigenvalues. With Mercer's theorem and the definition \eqref{eq:supplements:sobolev-spherical-harmonics} of Sobolev norms, we conclude that
\begin{equation*}
  \|f\|_{H_\NTK}^2
  \sim \|f\|_{Lap}^2
  = \sum_{\ell=0}^\infty \sum_{j=1}^{\nu(\ell)} \lambda_{\ell,j}^{-1} |\hat{f}_{\ell,j}|^2
  \sim \sum_{\ell=0}^\infty \sum_{j=1}^{\nu(\ell)} (\ell+1)^d |\hat{f}_{\ell,j}|^2
  = \|f\|_{H^{d/2}(\Sd)}^2.
\end{equation*}
 
\end{proof}

\begin{lemma} \label{lemma:supplements:sobolev:ntk-eigenvalues}

  Let $\NTK(x,y)$ be the neural tangent kernel for a fully connected neural network on the sphere $\Sd$ with bias and $\relu$ activation. It's eigenfunctions are spherical harmonics with eigenvalues
  \begin{align*}
    (\ell+1)^{-d} & \lesssim \lambda_{\ell,j} \lesssim (\ell+1)^{-d}, & 
    \ell & \ge 0, &
    j = 1, \dots, \nu(\ell),
  \end{align*}
  
\end{lemma}

\begin{proof}
  This follows directly form the norm equivalence $\|\cdot\|_{H_\NTK} \sim \|\cdot\|_{H^{d/2}}$ in Lemma \ref{lemma:supplements:sobolev:ntk-sobolev} and in Mercer's theorem representation of the RKHS
  \begin{equation*}
    \sum_{\ell=0}^\infty \sum_{j=1}^{\nu(\ell)} \lambda_{\ell,j}^{-1} |\hat{f}_{\ell,j}|^2
    = \|f\|_{H_\NTK}^2
    \sim \|f\|_{H^{d/2}(\Sd)}^2.
    = \sum_{\ell=0}^\infty \sum_{j=1}^{\nu(\ell)} (\ell+1)^d |\hat{f}_{\ell,j}|^2,
  \end{equation*}
  choosing $f=Y_\ell^j$ as a spherical harmonic.
\end{proof}

With the knowledge of the full spectrum of the NTK, it is now straight forward to show coercivity.

\begin{lemma}[Lemma \ref{lemma:supplements:sobolev:ntk-coercive}, restated]
  \lemmaCoercivity
\end{lemma}

\begin{proof}

Plugging in $f = \sum_{\ell=0}^\infty \sum_{j=1}^{\nu(\ell)} \hat{f}_{\ell j} Y_{\ell}^j$ in eigenbasis, and using the estimate $\lambda_{\ell j} \sim (\ell+1)^{-d}$ of the eigenvalues in Lemma \ref{lemma:supplements:sobolev:ntk-eigenvalues}, we have
\begin{align*}
    \dualp{f, L_\NTK f}_{H^\alpha(\Sd)}
    & = \sum_{\ell=0}^\infty \sum_{j=1}^{\nu(\ell)} (\ell+1)^{2\alpha} \hat{f}_{\ell j} \widehat{L_\theta f}_{\ell j}
    \\
    & = \sum_{\ell=0}^\infty \sum_{j=1}^{\nu(\ell)} (\ell+1)^{2\alpha} \lambda_{\ell j} |\hat{f}_{\ell j}|^2
    \\
    & = \|f\|_{H^{\alpha - d/2}(\Sd)}^2.
\end{align*}

\end{proof}

\bibliographystyle{abbrv}
\bibliography{approxoptdeep}

\end{document}